\documentclass[11pt]{article}
\usepackage{fullpage}

\usepackage[utf8]{inputenc} % allow utf-8 input
\usepackage[T1]{fontenc}    % use 8-bit T1 fonts
\usepackage{url}            % simple URL typesetting

\usepackage{algorithm}
\usepackage{algorithmic}
\usepackage[export]{adjustbox}
\usepackage{subcaption}
\usepackage[percent]{overpic}
\usepackage{wrapfig}
\usepackage{graphicx}
\usepackage{svg}

\usepackage{caption}
\usepackage{pifont}% http://ctan.org/pkg/pifont
\usepackage{float}

\usepackage{amsmath,amsfonts,amsthm,amssymb, nccmath}
\usepackage{nicefrac}       % compact symbols for 1/2, etc.
\allowdisplaybreaks % allow page breaks of equations
\usepackage{bbm}
\usepackage{framed}
\usepackage{enumerate}
\usepackage{mathtools}
\usepackage{scalefnt}
\usepackage{changepage}

\usepackage{comment}
\usepackage{setspace}
\usepackage{verbatim}
\usepackage{xcolor}
\usepackage{scalefnt}
\usepackage{xr}

\usepackage[shortlabels]{enumitem}
\setlist{nolistsep, noitemsep, leftmargin=20pt, labelwidth=3pt}

\usepackage{placeins}
\usepackage{color}
\usepackage{footmisc}
\usepackage{xspace}

\usepackage{thmtools}
\usepackage{thm-restate}

\usepackage{multirow}
\usepackage{booktabs}
\usepackage{array}
\usepackage{tabularx}

\usepackage{microtype}

\usepackage{hyperref}

\newcommand{\sipp}{\textsc{SiPP}\xspace}
\newcommand{\sippdet}{\textsc{SiPPDet}\xspace}
\newcommand{\sipprand}{\textsc{SiPPRand}\xspace}
\newcommand{\sipphybrid}{\textsc{SiPPHybrid}\xspace}
\newcommand{\sippsimple}{\textsc{SiPPSimple}\xspace}

\newcommand{\size}[1]{\nnz{#1}}
\newcommand{\nnz}[1]{\|{#1}\|_0}

\DeclarePairedDelimiterX{\dotp}[2]{\langle}{\rangle}{#1, #2}

\newcommand{\kmax}{C}

\newcommand{\logTerm}[1][\eta^*]{\log \left({#1} \, \rho / \delta \right)}

\newcommand{\SizeOfS}[1][]{K \logTerm[#1] }

\newcommand\BigO{\mathcal{O}}

\newcommand{\be}{\begin{equation}}
\newcommand{\ee}{\end{equation}}

\renewcommand{\AA}{{\mathcal{A}}}
\newcommand{\NN}{{\mathbb N}}

\newcommand{\norm}[1]{\left\| #1\right\|}                               %

\newcommand{\abs}[1]        {\left| #1 \right|}

\newcommand{\eps}{\ensuremath{\varepsilon}}                       % Epsilon
\renewcommand{\epsilon}{\varepsilon}

\newcommand{\REAL}{\ensuremath{\mathbb{R}}}                       % Norm
\newcommand{\RR}{{\REAL}}

\DeclareMathOperator*{\E}{\mathbb{E} \,}
\let\Pr\relax
\DeclareMathOperator*{\Pr}{\mathbb{P}}
\DeclareMathOperator*{\Var}{\text{Var}}

\DeclarePairedDelimiter{\ceil}{\lceil}{\rceil}

\declaretheorem[name=Theorem]{theorem}
\declaretheorem[name=Definition]{definition}
\declaretheorem[name=Assumption]{assumption}
\declaretheorem[name=Problem]{problem}

\declaretheorem[name=Corollary, numberlike=theorem]{corollary}
\declaretheorem[name=Lemma, numberlike=theorem]{lemma}

\newcommand{\neuron}{r}
\newcommand{\edge}{j}
\newcommand{\idx}{k}

\newcommand{\compressiblelayers}{[L]}
\newcommand{\etaDef}{\sum_{\ell = 1}^L \eta^\ell}

\newcommand{\epsilonLayer}[1][\ell]{\epsilon_{#1}}

\newcommand{\SampleComplexityDeltaLayers}[1][\epsilonLayer]{\ceil*{32  \, L^2 (\Delta_r^{\ell \rightarrow})^2 \, \kmax \log (8 \eta / \delta) \,  \epsilon^{-2} \, \sum_{\edge \in \Wpm} s_\edge }}

\newcommand{\epsilonNeuron}[1][\mNeuron]{\sqrt{\frac{\tilde S}{{#1}} \left(\frac{\tilde S}{{#1}} + 6\right)} + \frac{\tilde S}{{#1}}}

\newcommand{\param}{\theta}
\newcommand{\paramDef}{(\WW^1, \ldots, \WW^L)}
\newcommand{\paramHat}{\hat{\theta}}
\newcommand{\paramHatDef}{(\WWHat^1, \ldots, \WWHat^L)}

\newcommand{\f}{f_\param}
\newcommand{\fHat}{f_{\paramHat}}

\newcommand{\Input}{x}

\newcommand{\Point}{\Input}

\newcommand{\Wpm}{\II}

\newcommand{\II}{\mathcal{I}}

\newcommand{\gDefNoMax}[2][a]{ \frac{\WWRowCon_\edge \, {#1}_{\edge}(#2)}{\sum_{\idx \in \II_i^\ell } \WWRowCon_\idx \, {#1}_{\idx}(#2) }}

\newcommand{\gDef}[2][a]{\max_{{#1}(\cdot) \in \AA_i^\ell} \gDefNoMax[{#1}]{#2}}

\newcommand{\g}[1]{g_\edge(#1)}

\newcommand{\cdf}[1]{F_\edge \left(#1 \right)}

\newcommand{\sPM}[1][\edge]{s_{#1}}

\newcommand{\mNeuron}{N}

\newcommand{\STildeDef}[2][S]{\frac{{#1} C}{3} \log ({{#2} / \delta)}}

\newcommand{\T}[1][]{T_{#1}}

\newcommand{\WWRow}[1][\neuron]{w_{#1}}

\newcommand{\WWRowCon}[1][\neuron]{w}
\newcommand{\WWHatRowCon}[1][\neuron]{{\hat{w}}}
\newcommand{\WWHat}{\hat{\WW}}

\newcommand{\CC}{\mathcal C}

\newcommand{\DD}{{\mathcal D}}

\newcommand{\XX}{\mathcal{X}}

\newcommand{\YY}{\mathcal{Y}}
\newcommand{\WW}{W}

\renewcommand\SS{\mathcal{S}}

\newcommand\BB{\mathcal{B}}

\DefineFNsymbols{mySymbols}{{\ensuremath\dagger}  {\ensuremath\ddagger} *\P
   *{**}{\ensuremath{\dagger\dagger}}{\ensuremath{\ddagger\ddagger}}}
\setfnsymbol{mySymbols}

\newcommand{\smallsubsub}[1]{\subsubsection*{#1}}

\author{Cenk Baykal%
\thanks{Computer Science and Artificial Intelligence Laboratory, Massachusetts Institute of Technology,
\newline \indent \indent
emails: \texttt{\{baykal, lucasl, igilitschenski, rus\}@mit.edu}} $^*$,
Lucas Liebenwein$^{\dagger*}$,
Igor Gilitschenski$^{\dagger}$,
Dan Feldman%
\thanks{Robotics and Big Data Laboratory, University of Haifa, email: \texttt{dannyf.post@gmail.com}} , 
Daniela Rus$^{\dagger}$%
{\thanks{These authors contributed equally to this work}}
}

\date{}
\begin{document}

\title{SiPPing Neural Networks: Sensitivity-informed Provable Pruning of Neural Networks}
\maketitle

% REQUIRED
\begin{abstract}
We introduce a family of pruning algorithms that provably sparsifies the parameters of a trained model in a way that approximately preserves the model's predictive accuracy. Our algorithms use a small batch of input points to construct a data-informed importance sampling distribution over the network's parameters, and either use a sampling-based or deterministic pruning procedure, or an adaptive mixture thereof, to discard redundant weights. Our pruning methods are simultaneously computationally efficient, provably accurate, and broadly applicable to various network architectures and data distributions. The presented approaches are simple to implement and can be easily integrated into standard prune-retrain pipelines. We present empirical comparisons showing that our algorithms reliably generate highly compressed networks that incur minimal loss in performance relative to that of the original network.
\end{abstract}

\section{Introduction}
\label{sec:introduction}

The deployment of large state-of-the-art neural networks to resource-constrained platforms, such as mobile phones and embedded devices, is often prohibitive in terms of both time and space. \emph{Network pruning} algorithms have the potential to reduce the memory footprint and inference time complexity of large neural network models in low-resource settings. The goal of network pruning is to discard redundant weights of an overparameterized network and generate a compressed model whose performance is competitive with that of the original network. 
Network pruning can also be used to reduce the burden of manually designing a small network by automatically inferring efficient architectures from larger networks. Moreover, pruning algorithm can enable novel insights into the theoretical and practical properties of neural networks, including overparameterization and generalization~\cite{arora2018stronger, liebenwein2020provable}.

Existing network pruning algorithms are predominantly based on data-oblivious~\cite{renda2020comparing, Han15} or data-informed~\cite{gamboa2020campfire, li2019learning, molchanov2019importance, yu2018nisp, lee2018snip} heuristics that work well in practice in combination with an appropriate pruning pipeline that incorporates retraining. However, existing approaches generally lack provable guarantees (including data-informed approaches with the exception of~\cite{blg2018} which is only applicable to multi-layer perceptrons) and thus provide little insight into the mechanics of the pruning algorithms and consequently into the pruned network.

\begin{figure*}[htb!]
  \centering
   \includegraphics[width=0.95\textwidth]{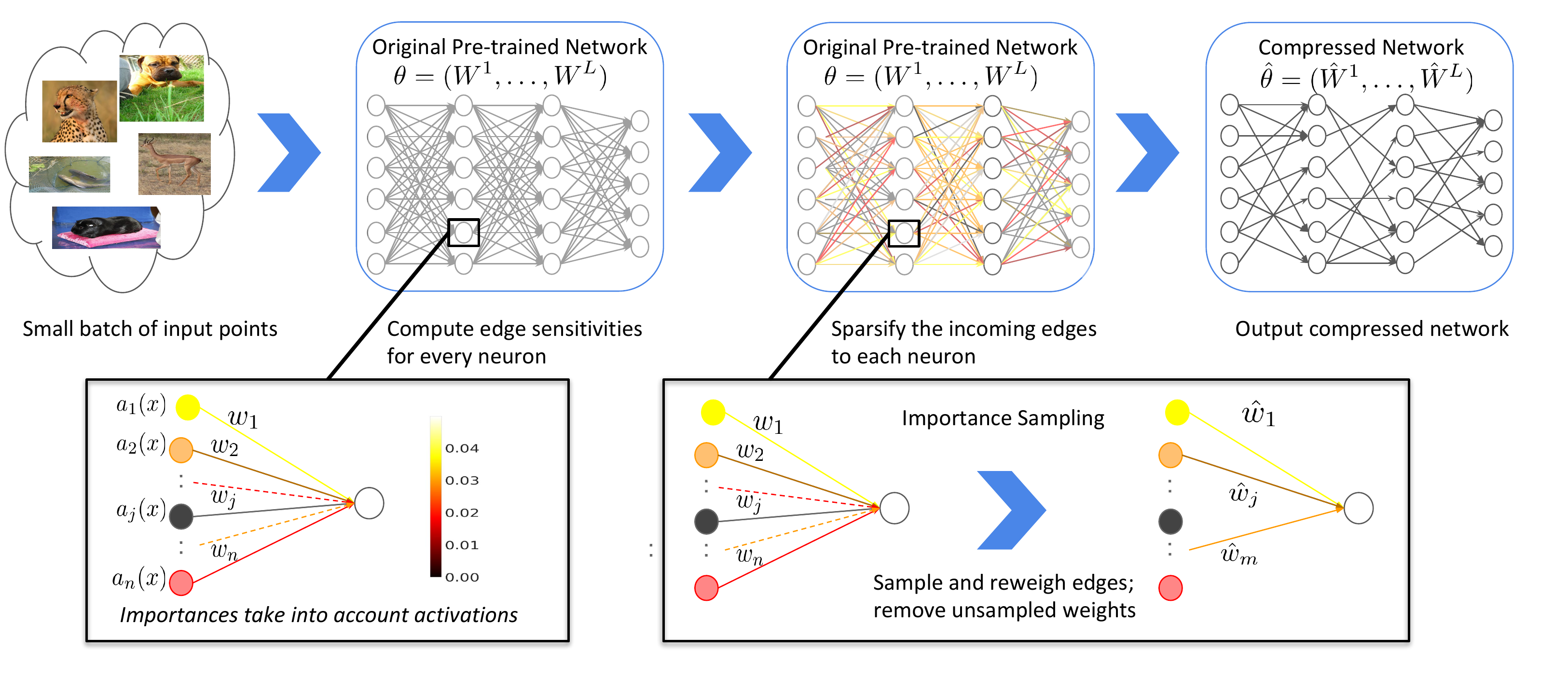}
  %\vspace{-0.5ex}
  \caption{
  The overview of our randomized method consisting of 4 parts. We use a small batch of input points to quantify the relative contribution (importance) of each edge to the output of each neuron. We then construct an importance sampling distribution over the incoming edges and sample a small set of weights for each neuron. The unsampled parameters are then discarded to obtain the resulting compressed network with fewer edges.
  }
  \label{fig:method-overview}
  \vspace{-3ex}
\end{figure*}

We close this research gap by introducing \sipp, see Figure~\ref{fig:method-overview} for an overview, a family of network pruning algorithms that provably compresses the network's parameters in a data-informed manner. Building and improving on state-of-the-art pruning methods, our algorithm is simultaneously provably accurate, data-informed, and applicable to various architectures including fully-connected (FNNs), convolutional (CNNs), and recurrent neural networks (RNNs). 
Relative to existing approaches, \sipp exhibits provable guarantees that hold regardless of the specific state of the network, i.e., it is simultaneously applicable to untrained, trained, or partially trained networks, and hence tends to perform consistently well across diverse pruning pipelines that incorporate various amounts of retraining. 
In addition, the theoretical analysis of \sipp provides novel analytical compression bound for deep neural networks that, e.g., can be utilized in the context of generalization bounds~\cite{blg2018, arora2018stronger,allen2019learning,arora2019fine,neyshabur2018the,zhou2018non}.

%Our contributions are:
This paper contributes the following:
\begin{enumerate}
    \item 
    A provable and versatile family of pruning algorithms, \sipp, that combines novel sample size allocation and adaptive sparsification procedures to prune network parameters.
    
    \item 
    An analysis of the resulting size and accuracy of the compressed network generated by \sipp that establishes novel compression bounds for a large class of neural networks.
    
    \item 
    Empirical evaluations for state-of-the-art iterative prune + retrain, random-init + prune + train scenarios on fully-connected and convolutional with comparisons to baseline pruning approaches highlighting the ability of $\sipp$ to span both theory and practice.
\end{enumerate}

\section{Related Work}
\label{sec:related-work}

\paragraph{Traditional approaches.}
Techniques such as Singular Value Decomposition (SVD) and regularized training~\cite{Denil2013, Denton14, jaderberg2014speeding, kim2015compression, tai2015convolutional, ioannou2015training, alvarez2017compression, yu2017compressing} were traditionally applied to compress networks. Other approaches in this realm exploit the structure of weight tensors to induce sparsity~\cite{Zhao17, sindhwani2015structured, cheng2015exploration, choromanska2016binary, wen2016learning}. 
Our work, in contrast, is a data-informed approach with guarantees on the size, relative error incurred at each output, and accuracy of the compressed network.

\paragraph{Network pruning.}
Weight pruning~\cite{lecun1990optimal} hinges on the idea that only a few dominant weights within a layer are required to approximately preserve the output. Approaches of this flavor were investigated by~\cite{lebedev2016fast,dong2017learning}, e.g., by embedding sparsity as a constraint~\cite{iandola2016squeezenet, aghasi2017net, lin2017runtime}. A popular weight-based pruning method is that of~\cite{Han15, renda2020comparing}, where weights with absolute values below a threshold are removed. A recent approach of~\cite{lee2018snip} prunes the parameters of the network by using a mini-batch of data points to approximate the influence of each parameter on the loss function of a randomly initialized network. Other data-informed techniques include~\cite{gamboa2020campfire, Lin2020Dynamic, li2019learning, molchanov2016pruning, molchanov2019importance, yu2018nisp, liebenwein2020provable}.
For an extensive overview see~\cite{gale2019state, blalock2020state, ye2018rethinking}.
Despite their favorable empirical performance, these approaches generally lack rigorous theoretical analysis of the effect that the discarded weights can have on the model's performance.

\paragraph{Theoretical foundations.} 
Recently, \cite{arora2018stronger} introduced a compression method based on random projections and proved norm-based bounds on the compressed network for points in the \emph{training set only}. In contrast, our work provides approximation guarantees on the network's output that hold even for points outside the training set. A coresets-based~\cite{feldman2011unified, braverman2016new} approach for compressing fully-connected networks was introduced by~\cite{blg2018} but is limited to FNNs and ReLUs. Our approach builds on this coresets-based framework to be applicable to various architectures and activation functions. Our algorithm also exhibits stronger error guarantees by mixing deterministic and sampling-based pruning strategies, by optimally allocating the sample sizes across the network to minimize the approximation error, and by establishing stronger network compression bounds using novel error propagation techniques.
\section{Background}
\label{sec:background}
We consider a neural network $f_\theta: \XX \to \YY$ consisting of $L$ layers with parameters $\theta$ and distribution $\DD$ over the input space $\XX$ from which we can sample i.i.d.\ input/label pairs $(x, y)$.

\paragraph{Network notation.}
For a given input $x \sim \DD$, we denote the pre-activation and activation of layer $\ell$ by $Z^\ell(x)$ and $A^\ell(x)$, respectively. Note that $f_\theta(x) = A^L(x)$,  $Z^0(x) = x$, and $A^\ell(x) = \phi^\ell(Z^\ell(x))$, where $\phi^\ell(\cdot)$ denotes the activation function.
We consider any multi-dimensional layer that can be described by a linear map with \emph{parameter sharing}, e.g. fully-connected layers, convolutional layers, or LSTM cells. Specifically, for a layer $\ell$ the pre-activation $Z^\ell(x)$ of layer $\ell$ is described by the linear mapping of the activation $A^{\ell - 1}(x)$ with $\WW^\ell$, i.e., 
$Z^\ell(x) = \WW^\ell * A^{\ell - 1}(x)$, where $*$ denotes the operator of the linear map, e.g., the convolutional operator. Moreover, we denote by $c^\ell$ the number of parameter groups within a layer that do not interact with each other, e.g., individual filters in convolutional layers. Then, let $Z_i^\ell(x) = \WW^\ell_i * A^{\ell - 1}(x)$, $i \in [c^\ell]$, denote the $i^\text{th}$ pre-activation channel of layer $\ell$ produced by parameter group $\WW^\ell_i$.

\paragraph{Problem definition.}
For given $\epsilon, \delta \in (0,1)$, our overarching goal is to use a pruning algorithm to generate a sparse reparameterization $\paramHat$ of $\theta$ such that $\size{\paramHat} \ll \size{\param}$ and for $\Point \sim \DD$ the $\ell_2$-norm of reference network output $\f(\Input)$ can be approximated by $\fHat(\Input)$ up to $1 \pm \eps$
multiplicative error with probability greater than $1- \delta$, i.e., $\Pr (\norm{\fHat (\Input) - \f(\Input)} \leq \eps \norm{\f(\Input)}) \ge 1 - \delta$. 

\section{Method}
\label{sec:method}

In this section, we present an overview for our family of pruning algorithms, \sipp: Sensitivity-informed Provable Pruning (see Figure~\ref{fig:method-overview} and Algorithm~\ref{alg:sipp}). In its core, \sipp proceeds as follows: (1) optimally allocate a given budget across layers and parameter groups to minimize the theoretical error bounds resulting from our analysis ($\textsc{OptAlloc}$, Line~\ref{lin:sipp-optalloc}); (2) compute the relative importance of individual weights within parameter groups ($\textsc{EmpiricalSensitivity}$, Line~\ref{lin:sipp-sensitivity}); (3) prune weights within each parameter group using the desired variant of \sipp according to their relative importance ($\textsc{Sparsify}$, Line~\ref{lin:sipp-sparsify}); (4) repeat (2) and (3) for each parameter group and each layer.
% Below, we provide additional details for each step.

\paragraph{$\textsc{OptAlloc}(\theta, \BB, \SS)$}
In the course of our analysis (see Section~\ref{sec:analysis}) we establish relative error bounds for the approximation $\hat Z^\ell_i(x) = \WWHat_i^\ell * A^{\ell-1}(x)$ of the form 
$
\textstyle
\hat Z^\ell_i(x) \in \left( 1 \pm \epsilon_i^\ell(m_i^\ell) \right) Z^\ell_i(x)
$
for individual parameter groups.
Roughly speaking, the associated relative error $\epsilon_i^\ell(m_i^\ell)$ is a (convex) function of the parameter group, the input, and the allocated budget $m_i^\ell$.
Thus in order to optimally utilize a desired budget $\BB$ we aim to minimize the following objective during the allocation procedure:
$$
\begin{aligned}
    \textstyle
    &\min_{m_i^\ell \in \NN \; \forall i \in [c^\ell], \; \forall \ell \in [L]} 
    && \textstyle
    \sum_{\ell \in [L], \; i \in [c^\ell]} \epsilon_i^\ell(m_i^\ell)
    & \text{s. t.} 
    && \textstyle 
    \sum_{\ell \in [L], \; i \in [c^\ell]} m_i^\ell \leq \BB.
\end{aligned}
$$
We note that the integral constraint $m_i^\ell \in \NN$ prevents us from efficiently finding a solution, we relax it to $m_i^\ell \in \RR$ to find the optimal fractional solution. We then use a technique like randomized rounding~\cite{srinivasan1999improved} to find an approximately optimal integral solution. Depending on the variant of \sipp, however, this step is not necessary.

\begin{algorithm}[t!]
\small
\caption{\sipp$(\param, \BB, \SS)$}
\label{alg:sipp}
\textbf{Input:}
$\param = \paramDef$: weights of the uncompressed neural network; $\BB \in \NN$: sampling budget; $\SS$: a set of $n$ i.i.d. validation points drawn from $\DD$
\\
\textbf{Output:} sparse weights $\paramHat = \paramHatDef$
\begin{spacing}{1.2}
\begin{algorithmic}[1]
\small

\STATE $m_i^\ell \gets \textsc{OptAlloc}(\theta, \BB, \SS) \; \forall i \in [c^\ell],\; \forall \ell \in [L]$ \COMMENT{optimally allocate budget $\BB$ across parameter groups and layers} \label{lin:sipp-optalloc}

\FOR{$\ell \in \compressiblelayers$}

    \STATE $\WWHat^\ell \gets \textbf{0}$; \COMMENT{Initialize a null tensor}
    
    \FOR{$i \in [c^\ell]$}
    
        \STATE $s_j \gets \textsc{EmpiricalSensitivity}(\theta, \SS, i, \ell) \; \forall w_j \in \WW_i^\ell$ \COMMENT{Compute parameter importance for each weight $w_j$ in the parameter group} \label{lin:sipp-sensitivity}
    
        \STATE $\WWHat_i^\ell \gets \textsc{Sparsify}(\WW_i^\ell, m_i^\ell, \{s_j\}_j)$ \COMMENT{prune weights according to \sippdet, \sipprand, or \sipphybrid such that only $m_i^\ell$ weights remain in the parameter group} \label{lin:sipp-sparsify}
	
	\ENDFOR

\ENDFOR

\STATE \textbf{return} $\paramHat = \paramHatDef$;

\end{algorithmic}
\end{spacing}
\end{algorithm}

\paragraph{$\textsc{EmpiricalSensitivity}(\theta, \SS, i, \ell)$}
To estimate the relative importance of a weight $w_j$ within a parameter group $\WW_i^\ell$, we use and extend the notion of \emph{empirical sensitivity} (ES) as first introduced in~\cite{blg2018} for fully-connected layers only. In its essence, ES quantifies the maximum relative contribution of a weight parameter $w_j$ to the output (pre-activation) of the layer compared to other weights in the parameter group. More formally, let the ES $s_j$ of $w_j$ in parameter group $W_i^\ell$ be defined as 
\begin{equation}
\label{eq:sens-definition}
\textstyle
\sPM := \max_{\Point \in \SS}  \, \max_{a(\cdot) \in \AA} \nicefrac{\WWRow[\edge]  a_{\edge}(\Input)}{\sum_{\idx} \WWRow[\idx]  a_{\idx}(\Input)},
\end{equation}
where we assume $W_i^\ell \geq 0$, $A^{\ell-1}(x) \geq 0$ for ease of exposition (see Section~\ref{sec:analysis} for the generalization to all weights and activations).
We note that the definition of $s_j$ entails two maxima.
The maximum over data points $\max_{\Point \in \SS}$ ensures that ES approximates the relative importance of $w_j$ sufficiently well for any i.i.d.\ data point $x \sim \DD$.
The maximum over patches $\AA$, which are generated from $A^{\ell-1}(x)$, ensures that ES approximates the relative importance of $w_j$ sufficiently well for all scalars in the output $Z_i^\ell(x)$ that require $w_j$ (c.f. parameter sharing). 
To further contextualize the purpose of patches $\AA$, consider a single parameter group within a convolutional layer, i.e., a filter. The filter gets slid across the input image to generate the output image repeatedly applying the same weights. Thus in order to quantify the importance of some weight $w_j$ we need to consider its relative importance across all sliding windows, henceforth requiring $\max_{a(\cdot) \in \AA}$.

\paragraph{$\textsc{Sparsify}(\WW_i^\ell, m_i^\ell, \{s_j\}_j)$}
Equipped with a budget, c.f. $\textsc{OptAlloc}$, and a notion of parameter importance, c.f. $\textsc{EmpiricalSensitivity}$, we introduce the three variants of \sipp, all of which exhibit provable guarantees as outlined in Section~\ref{sec:analysis}, to prune weights from a parameter group:
\begin{enumerate}
    \item 
    \sippdet: we deterministically pick the $m_i^\ell$ weights with largest sensitivity and zero out the rest of the weights to construct $\WWHat_i^\ell$.
    \item
    \sipprand: we construct an importance sampling distribution over weights $w_j$ using their associated sensitivities $s_j$, then sample with replacement until we obtain a set of $m_i^\ell$ unique weights to construct $\WWHat_i^\ell$.
    \item 
    \sipphybrid: we evaluate the theoretical error guarantees (see Section~\ref{sec:analysis}) associated with the two other methods, and prune using the method that incurs the lower relative error. 
\end{enumerate}
We note that while \sippdet is particularly simple to implement, \sipphybrid provides the biggest amount of flexibility and consistently good prune results since it can adaptively choose for each parameter group whether to prune using \sippdet or \sipprand.

\section{Analysis}
\label{sec:analysis}
In this section, we outline the theoretical guarantees for \sipp.
The full proofs can be found in the supplementary material.
We start out by establishing the core lemmas that constitute the relative error guarantees for both \sippdet and \sipprand for the case where $W_i^\ell \geq 0$,  $A^{\ell-1}(x) \geq 0$ for ease of exposition. Specifically, we establish relative error guarantees for each individual output patch that is associated with a parameter group.
We then outline the steps that are required to generalize the analysis to all weights and activations.
Finally, we show -- by means of composing together the error guarantees from individual output patches, parameters groups, and layers -- how to derive the analytical compression bounds for the entire network.

% \paragraph{Preliminaries.}

\smallsubsub{Empirical sensitivity}
In the previous section we introduce the notion of ES, see equation~\ref{eq:sens-definition}, as a means to quantify the importance of weight $w_j$ relative to the other weights within a parameter group $\WW_i^\ell$. Using ES we establish a key inequality that upper bounds the contribution of $w_j a_j(x)$ to its associated output patch $z(x) = \sum_k w_k a_k(x)$ for any $x \sim \DD$ with high probability (w.h.p.) under mild regularity assumption on the input distribution to the layer.

\begin{lemma}[Informal ES inequality]
\label{lem:informal-es-inequality}
For weights $w_j$ from parameter group $\WW_i^\ell$ and an arbitrary input patch $a(\cdot)$ we have w.h.p.\ for any $x \sim \DD$ that 
$
\textstyle
w_j a_j(x) \leq C s_j z(x),
$
where $z(x)$ denotes the associated output patch and $C = \BigO(1)$.
\end{lemma}

The ES inequality is a key ingredient in bounding the error of \sippdet and \sipprand in terms of sensitivity. Specifically, Lemma~\ref{lem:informal-es-inequality} puts the individual contribution of a weight to the output patch in terms of its sensitivity and the output patch itself. The inequality hereby holds \emph{w.h.p.} for any data point $x \sim \DD$ which enables us to bound the quality of the approximation even for previously unseen data points.
We leverage Lemma~\ref{lem:informal-es-inequality} in the subsequent analysis
to quantify the approximation error of an output patch when the output patch was only approximately computed using a subset of weights, i.e., with the weights that remain after pruning.

\smallsubsub{Error guarantees for \sippdet}
Recall that \sippdet prunes weights by keeping the $m_i^\ell$ weights of parameter group $W_i^\ell$ with largest ES. Now let $\II$ denote the index set of all weights in $W_i^\ell$ and $\II_{det}$ the index set of weights with largest sensitivity that are kept after pruning such that $\abs{\II_{det}} = m_i^\ell$. We bound the incurred error of the approximation by considering the difference between the output patch and the approximated output patch, i.e., the difference between
$
\textstyle
z(x) = \sum_{j \in \II} w_j a_j(x) \text{ and } \hat z_{det} (x) = \sum_{j \in \II_{det}}  w_j a_j(x).
$

\begin{lemma}[Informal \sippdet error bound]
\label{lem:informal-sippdet}
For weights $w_j$ from parameter group $\WW_i^\ell$, an arbitrary associated input patch $a(\cdot) \in \AA$, and corresponding output patch $z(\cdot)$ \sippdet generates an index set $\II_{det}$ of pruned weights such that for any $x \sim \DD$ w.h.p.
$
\abs{\hat z_{det}(x) - z(x)} \leq \epsilon_{det} z(x),
$
where $\epsilon_{det} = C \sum_{j \in \II \setminus \II_{det}} s_j$.
\end{lemma}
The proof of Lemma~\ref{lem:informal-sippdet} follows from the fact that the difference between the approximate output patch $\hat z_{det}(x)$ and the unpruned output patch $z(x)$ is exactly the sum over the contributions from weights that are \emph{not} in the pruned subset of weights $\II_{det}$. Using Lemma~\ref{lem:informal-es-inequality} we then bound the error in terms of the sensitivity of the \emph{pruned} weights.
Intuitively, ES of an individual weight precisely quantifies the relative error incurred when that weight is pruned. The resulting relative error can thus be described by the cumulative ES of pruned weights. 

\smallsubsub{Error guarantees for \sipprand}
Here we prune weights from a parameter group by constructing an importance sampling distribution from the associated ESs. Specifically, some weight $w_j$ is sampled with probability $q_j = \nicefrac{s_j}{\sum_{k \in \II} s_k}$ and we repeatedly sample with replacement until the corresponding set of sampled weights contains $m_i^\ell$ unique weights. Each sampled weight is then reweighed by the number of times it was sampled divided by the total number of samples and its sample probability to construct the approximate output patch, i.e.,
$$
\textstyle
\hat z_{rand} (x) = \sum_{j \in \II_{rand}} \hat w_j a_j(x) = \sum_{j \in \II_{rand}} \frac{n_j}{N q_j} w_j a_j(x),
$$
where $\II_{rand}$ denotes the index set of weights that were sampled at least once, $n_j$ denotes the number of times weight $w_j$ was sampled, and $N = \sum_{j \in \II_{rand}} n_j$ denotes the total number of samples. 
We then bound the incurred error by analyzing the random difference between the approximated output patch and the original output patch, i.e., $\abs{\hat z_{rand}(x) - z(x)}$,  establishing the following error guarantee.

\begin{lemma}[Informal \sipprand error bound]
\label{lem:informal-sipprand}
For weights $w_j$ from parameter group $\WW_i^\ell$, an arbitrary associated input patch $a(\cdot) \in \AA$, and corresponding output patch $z(\cdot)$ \sipprand generates a set of pruned weights such that for any $x \sim \DD$ w.h.p.
$
\abs{\hat z_{rand}(x) - z(x)} \leq \epsilon_{rand} z(x),
$
where $\epsilon_{rand} = \BigO(\sqrt{\nicefrac{S}{N}})$ and $S = \sum_{k \in \II} s_k$ denote the relative error and sum of ESs, respectively.
\end{lemma}
The proof proceeds in two steps. 
First, we show that the (random) approximation is an unbiased estimator of the original parameter group, i.e., $\E[\hat z_{rand}(x)] = z(x)$, which follows from the reweighing term of $\hat w_j$.
Second, we show that using Bernstein's concentration inequality~\cite{vershynin2016high} the sampling distribution exhibits strong subGaussian~\cite{vershynin2016high} concentration around the mean, i.e., the approximate output patch is $\epsilon$-close to the original, unpruned output patch w.h.p. Specifically, we leverage Lemma~\ref{lem:informal-es-inequality} to bound the variance of the approximate output patch using the cumulative ES $S = \sum_{k \in \II} s_k$ of the parameter group.

\smallsubsub{Discussion of error bounds and \sipphybrid}
Most notably, \sipprand is an unbiased estimator regardless of the budget, while \sippdet is always an underapproximation becoming increasingly worse in expectation with lower budget. 
On the other hand, if the parameter group is dominated by a few weights, \sippdet can directly captures these weights whereas \sipprand inherent randomness from the sampling procedure may introduce additional sources of failure.
Combining the strengths of both, we introduce \sipphybrid, which evaluates both theoretical error guarantees before pruning a parameter group to adaptively choose the better prune strategy.

\smallsubsub{Generalization to all weights}
Previously, we have assumed that both the parameter group and input activations are strictly non-negative, i.e., $\WW_i^\ell \geq 0$ and $A^{\ell-1}(x) \geq 0$. To handle the general case, we split the parameter group and input activations each into a positive and negative part representing the four quadrants such that each quadrant is now strictly non-negative. We can then incorporate each quadrant into our pruning procedure to ensure that the error guarantees hold simultaneously for all quadrants. To obtain error bounds for the actual pre-activation we introduce $\Delta^\ell$, which quantifies the ``sign complexity'' of the overall approximation for a particular layer to quantify the additional complexity from considering the alternating signs of each quadrant, see supplementary material for more details.

\smallsubsub{Network compression bounds}
In the previous section we have outlined how to obtain error guarantees for individual output patches. Naturally, since the guarantees hold for all patches within a parameter group and individual parameter groups within a layer are independent from each other, we can simultaneously establish norm-based error guarantees for the entire pre-activation of a layer, i.e., $\norm{\hat Z^\ell(x) - Z^\ell(x)} \leq \eps \norm{Z^\ell(x)}$ w.h.p.
Moreover, assuming the activation function is entry-wise and $1$-Lipschitz continuous, the same relative error guarantees hold
for the activation of layer. Note that any common activation function satisfies
the above assumption, including and all others listed in PyTorch’s documentation~\cite{pytorch2020nonlinear}.
Finally, we have to consider the effect of pruning \emph{multiple layers} simultaneously and the implications on the final output $f_{\theta}(x) = A^L(x)$ of the network. Informally speaking, we incur two sources of error from each layer. (1) the error associated from pruning within layers and (2) the error associated with propagated the incurred error throughout the network to the output layer.
We quantify the error within layers using our patch-wise guarantees and the sign complexity $\Delta^\ell$ of the layers. 
We quantify the propagated error across layers by upper bounding the layer condition number, $\kappa^\ell$ which quantifies the relative error incurred in the output for some relative error incurred within the layer. Intuitively, the concept of the layer condition number is closely related to the Lipschitz constant between some layer and the output of the network.
Below, we informally state the compression bound when pruning the entire network with \sipprand.

\begin{theorem}[Informal compression bound]
For given $\delta \in (0, 1)$ and budget $\BB$ \sipp (Algorithm~\ref{alg:sipp}) generates a set of pruned parameters $\paramHat$ such that $\size{\hat \theta} \leq \BB$,
$
\Pr_{\paramHat, \Point} \left(\norm{\fHat (\Input) - \f(\Input)} \leq \eps \norm{\f(\Input)}\right) \ge 1 - \delta
$
and
$
\eps = \BigO ( \sum_{\ell = 1}^L \kappa^\ell \Delta^\ell \max_{i \in [c^\ell]} (S_i^\ell - S_i^\ell(N_i^\ell) ) ),
$
where $S_i^\ell$ and $S_i^\ell (N_i^\ell)$ is the sum over all and the largest $N_i^\ell$ ESs, respectively, and $N_i^\ell$ is the budget allocated for parameter group $W_i^\ell$.
\end{theorem}

We note that the compression bound is proportional to the sum of cumulative ESs for each parameter group, a term which arises in numerous applications of coresets~\cite{feldman2011unified}.
Moreover, we see the layer condition number $\kappa^\ell$ and sign complexity $\Delta^\ell$ of each layer appear in the final bound. Both terms are related to how injecting error simultaneously in each layer (by pruning the network) affects the overall output of the network and are related to concepts such as the Lipschitz constant of the network and/or interlayer cushion as introduced in related work that establishes generalization bounds for neural networks~\cite{arora2018stronger,neyshabur2018a}. Like other recent work in the field~\cite{arora2018stronger, Suzuki2020Compression} our work highlights the intrinsic connection between the compression ability and generalization ability of neural networks.

\section{Experiments}
\label{sec:results}

In this section, we evaluate and compare the performance of our algorithm, \sipp, on pruning fully-connected, convolutional, and residual networks. We embed our pruning algorithm into pruning pipelines including retraining to empirically test its performance and test it for scenarios involving significant amounts of (re)-training as well as a prune  pipeline that utilizes no more training epochs than regular training. To be able to compare our pruning approach \sipp to competing pruning approaches, we consider standard retraining pipelines that are network-agnostic and yield state-of-the-art pruning results~\cite{lee2018snip, renda2020comparing}.
Specifically, we consider two scenarios -- iterative prune + retrain and random-init + prune + train -- as described below.

\subsection{Experimental Setup}

\paragraph{Architectures and data sets.} 
We train and prune networks on CIFAR10~\cite{torralba200880} and ImageNet~\cite{ILSVRC15}. 
We consider ResNets20/56/110~\cite{he2016deep}, WideResnet16-8~\cite{zagoruyko2016wide}, Densenet22~\cite{huang2017densely}, VGG16~\cite{Simonyan14}, CNN5~\cite{Nakkiran2020Deep} and ResNet18, ResNet101~\cite{he2016deep} for CIFAR10 and ImageNet, respectively.

\paragraph{Training.}
For both training and retraining we deploy the standard sets of hyperparameters as described in the respective papers. All hyperparameters are listed in the supplementary material.

\paragraph{Pruning algorithms.}
We consider the following pruning algorithms to be incorporated into the pruning pipelines discussed above:

\begin{itemize}

\item 
\sippdet.
We prune the entire network deterministically. Note that in this case (due to the sample size allocation procedure) \sippdet corresponds to global thresholding of sensitivity (reminiscent of weight thresholding).

\item
\sipprand.
We prune the entire network using importance sampling.

\item
\sipphybrid.
We use our combined pruning approach as outlined in Algorithm~\ref{alg:sipp}.

\item
\textsc{WT}.
We globally prune weights according to their magnitude~\cite{Han15, renda2020comparing}.

\item
\textsc{Snip}.
We globally prune weights according to the (data-informed) magnitude of the product between weight and gradient~\cite{lee2018snip}.

\end{itemize}

We note that \textsc{WT} (``learning rate rewinding'') is the current state-of-the-art for iterative prune+retrain pipelines~\cite{renda2020comparing} while \textsc{Snip} is the current state-of-the-art for random-init + prune + train~\cite{lee2018snip}. We also report comparisons against a broader set of pruning pipelines in the supplementary material.

\subsection{Iterative prune + retrain}

\begin{wrapfigure}{r}{0.5\textwidth}
  %\vspace{-8ex}
  \centering
  \includegraphics[width=0.5\textwidth]{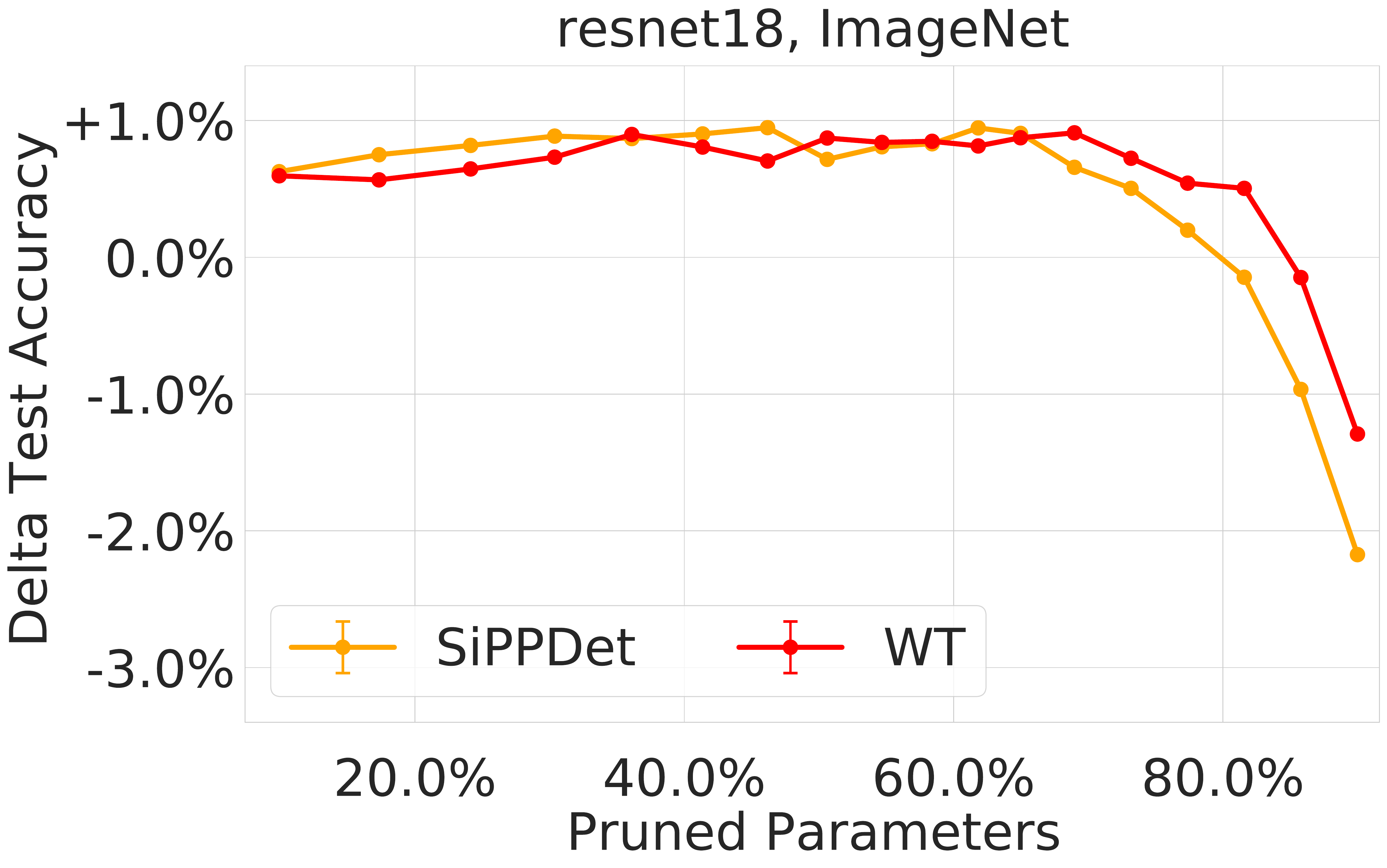}
  \caption{Results on the pruning performance of a ResNet18 trained on ImageNet using \textbf{iterative prune+retrain}.}
  \label{fig:resnet18_imagenet}
  %\vspace{-3ex}
\end{wrapfigure}

\paragraph{Methodology}
We deploy an iterative prune + retrain scheme that proceeds as follows:
\begin{enumerate}
    \item 
    train network to completion;
    \item
    prune a fixed ratio of parameters from the network;
    \item
    retrain using the same hyperparameters as during training;
    \item
    iteratively repeat steps 2., 3. to obtain smaller prune ratios.
\end{enumerate}
This procedure as used in~\cite{renda2020comparing, liebenwein2020provable} is shown to produce state-of-the-art prune results although it requires significant amount of retraining resources. We choose it for its simplicity and network-agnostic hyperparameters. Due to the expensive nature of iterative prune+retrain we choose to only evaluate it for \sipp but not the other variants of our algorithm as it is the simplest and we observed little difference in performance between the three variations. In the supplementary material, we provide additional (experimental) justification that supports our claim.

\paragraph{Results.}
Figure~\ref{fig:cifar_prune_results} summarizes the results of the iterative prune + retrain procedure for various CIFAR10 networks. The results were averaged across 3 trained networks. Our empirical evaluation shows that our algorithm consistently performs comparably to state-of-the-art \textsc{WT} with learning rate rewinding~\cite{renda2020comparing}. We note that \textsc{Snip}'s performance is much lower is these scenarios. We suspect this is due to the gradients being close to zero for a fully-trained network (the pruning step is performed after training in this scenario). 
In Figure~\ref{fig:resnet18_imagenet} we show results for a ResNet18 trained, pruned, and retrained on ImageNet. As in the case of CIFAR10 networks we observe that \sipp performs en par with \textsc{WT}.

\begin{figure}[b!]
%\vspace{-3ex}
\centering
\begin{minipage}[t]{0.33\textwidth}
    \includegraphics[width=\textwidth]{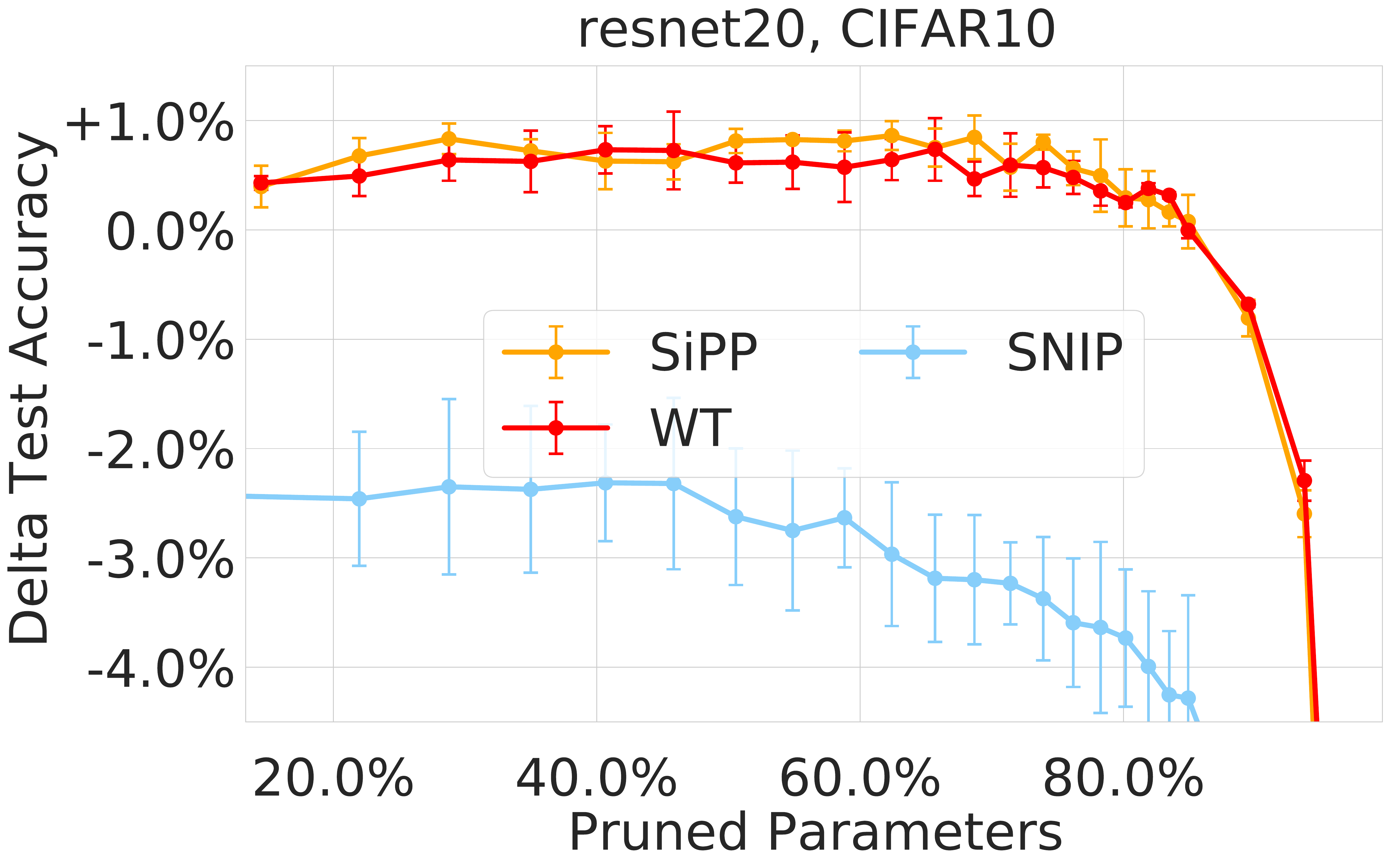}
    % \subcaption{Resnet20}
\end{minipage}%
\begin{minipage}[t]{0.33\textwidth}
    \includegraphics[width=\textwidth]{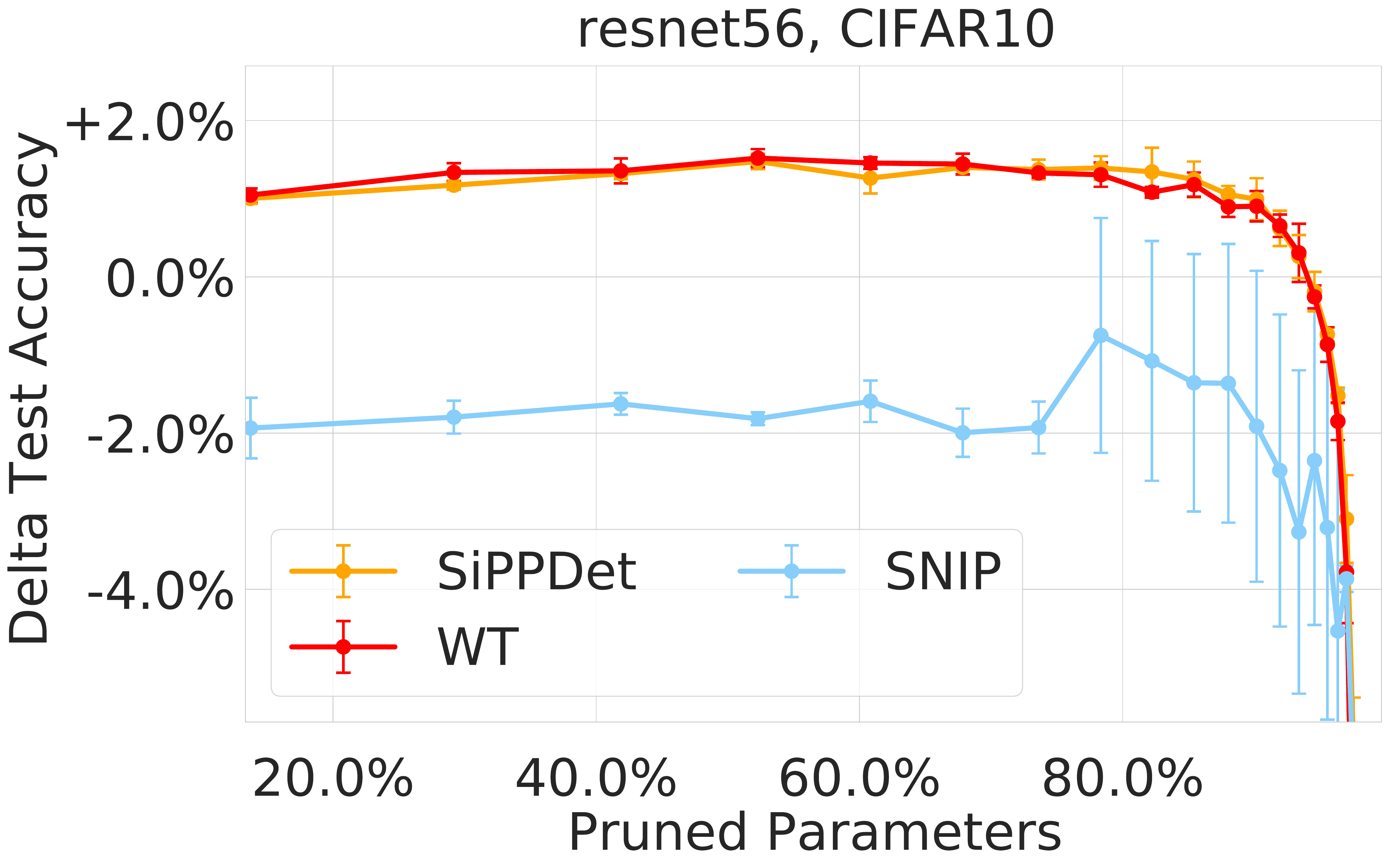}
    % \subcaption{Resnet56}
\end{minipage}%
\begin{minipage}[t]{0.33\textwidth}
    \includegraphics[width=\textwidth]{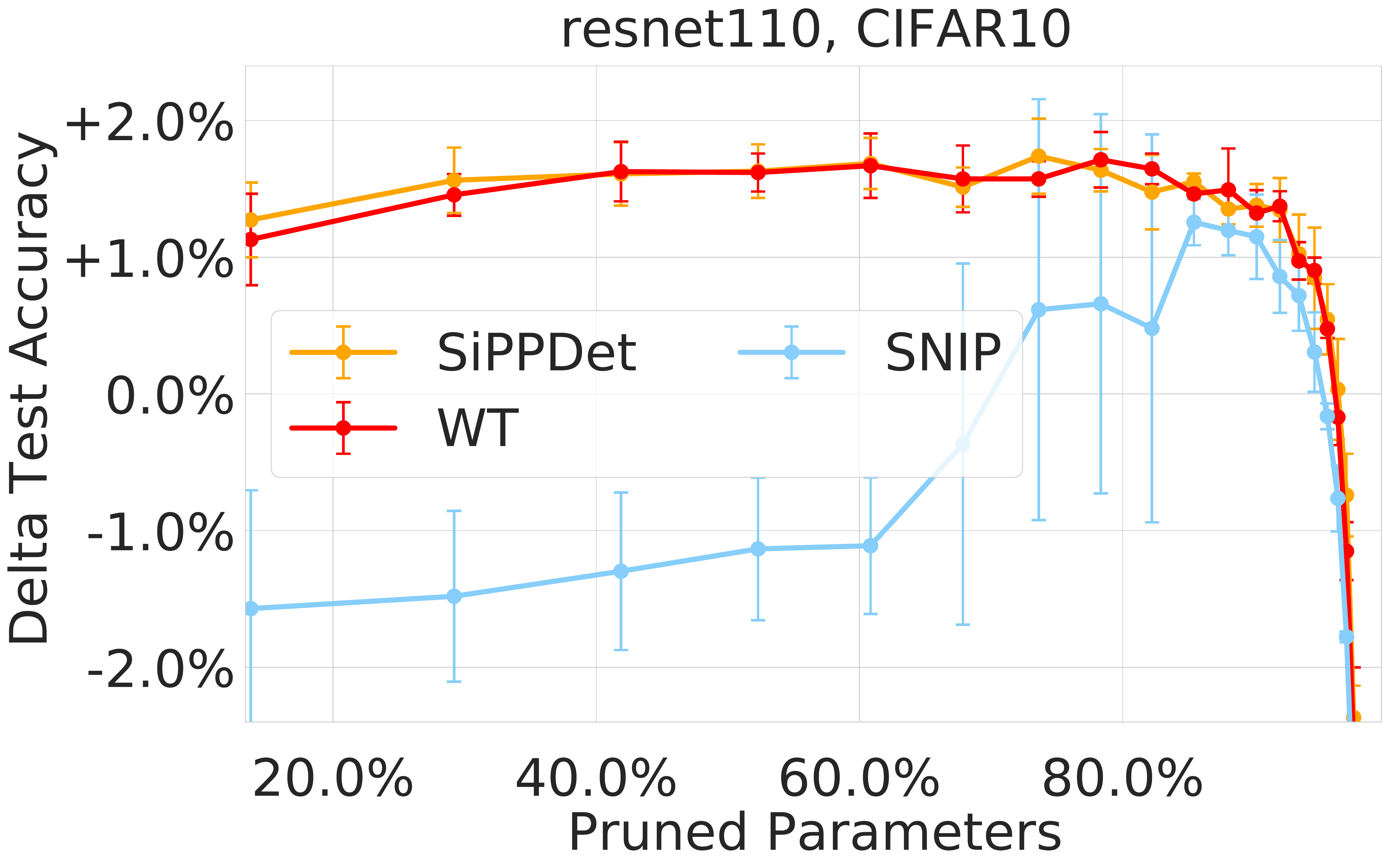}
    % \subcaption{Resnet110}
\end{minipage}%
\vspace{2ex}
\begin{minipage}[t]{0.33\textwidth}
    \includegraphics[width=\textwidth]{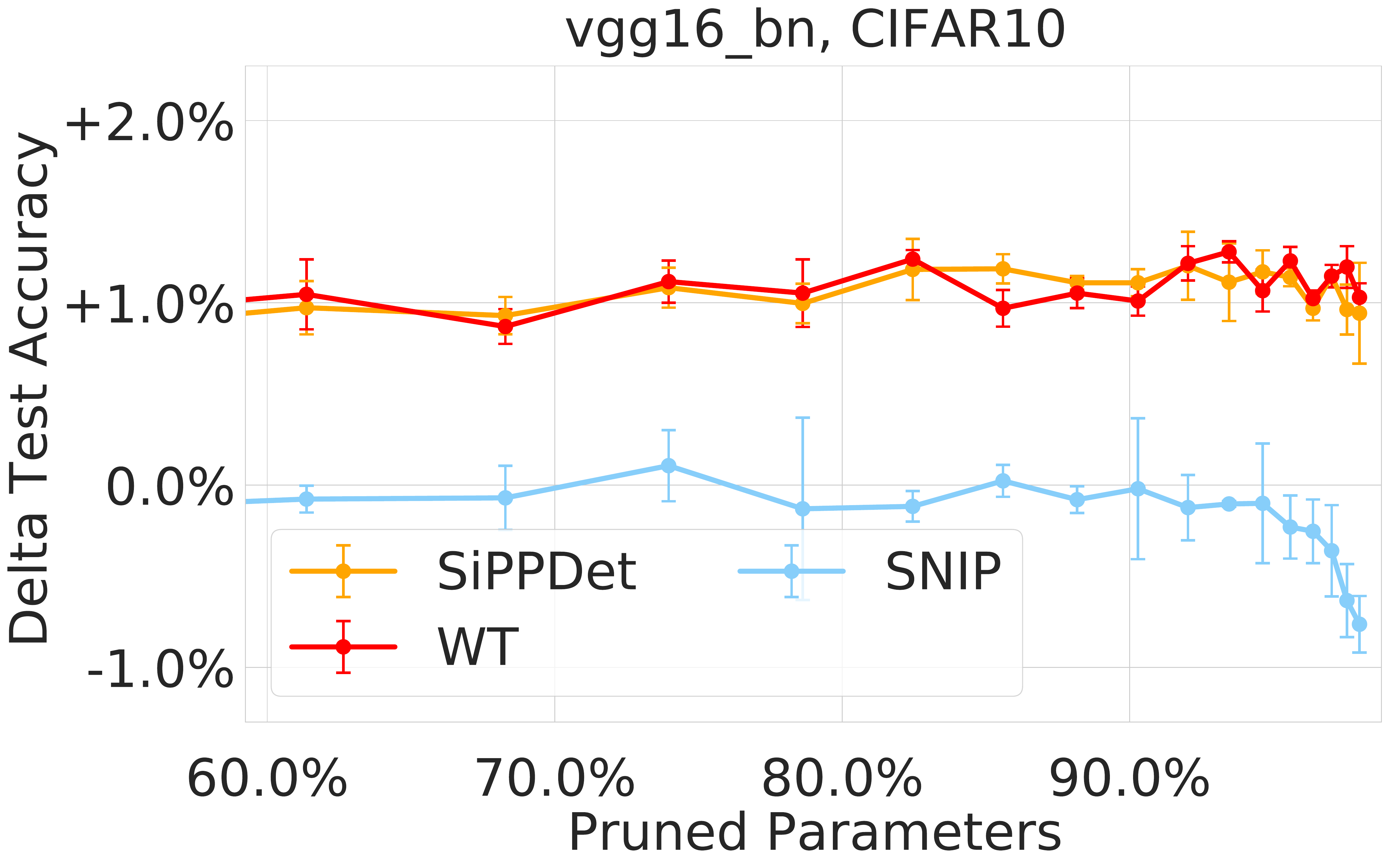}
    % \subcaption{VGG16}
\end{minipage}%
\begin{minipage}[t]{0.33\textwidth}
    \includegraphics[width=\textwidth]{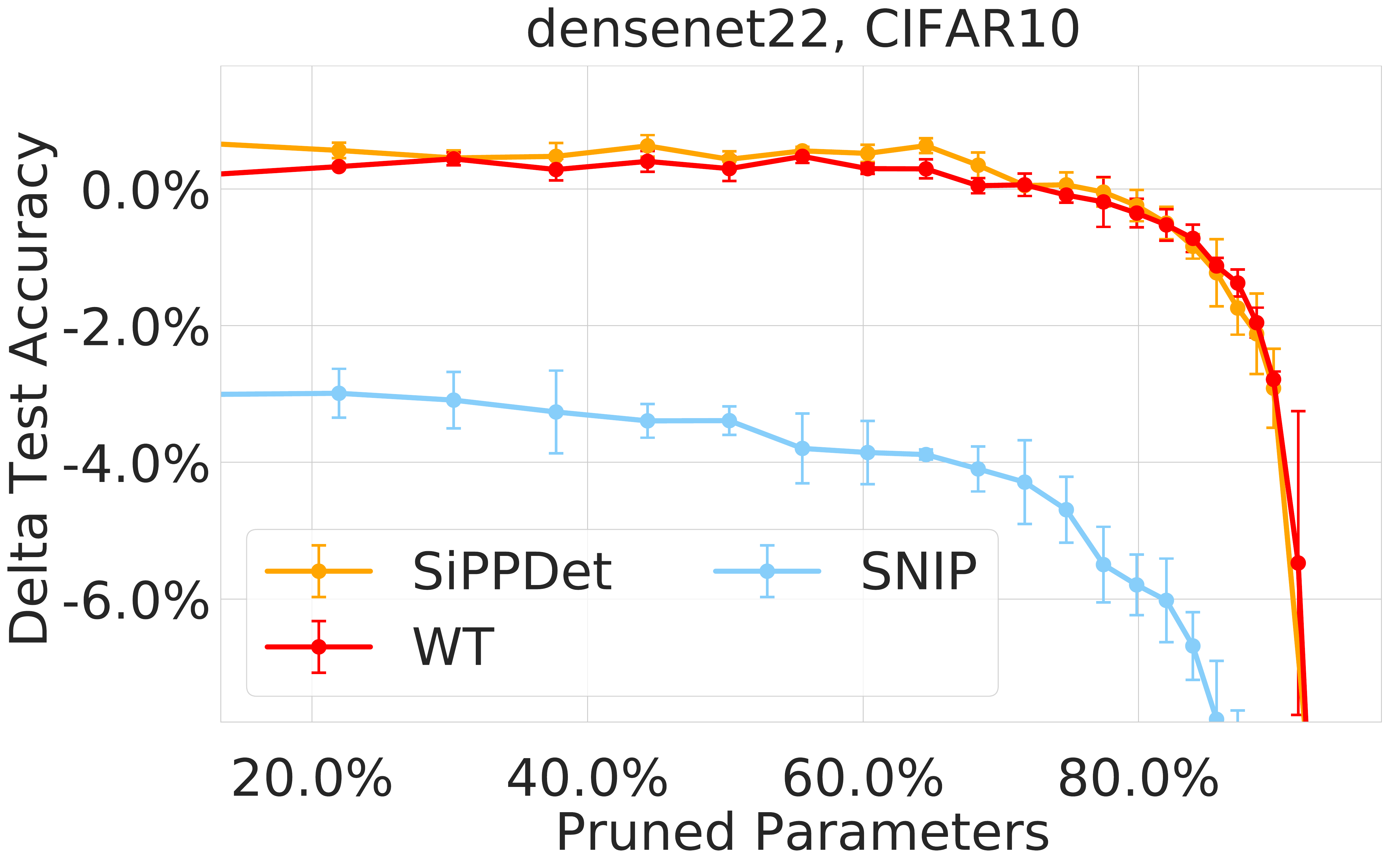}
    % \subcaption{DenseNet22}
\end{minipage}%
\begin{minipage}[t]{0.33\textwidth}
    \includegraphics[width=\textwidth]{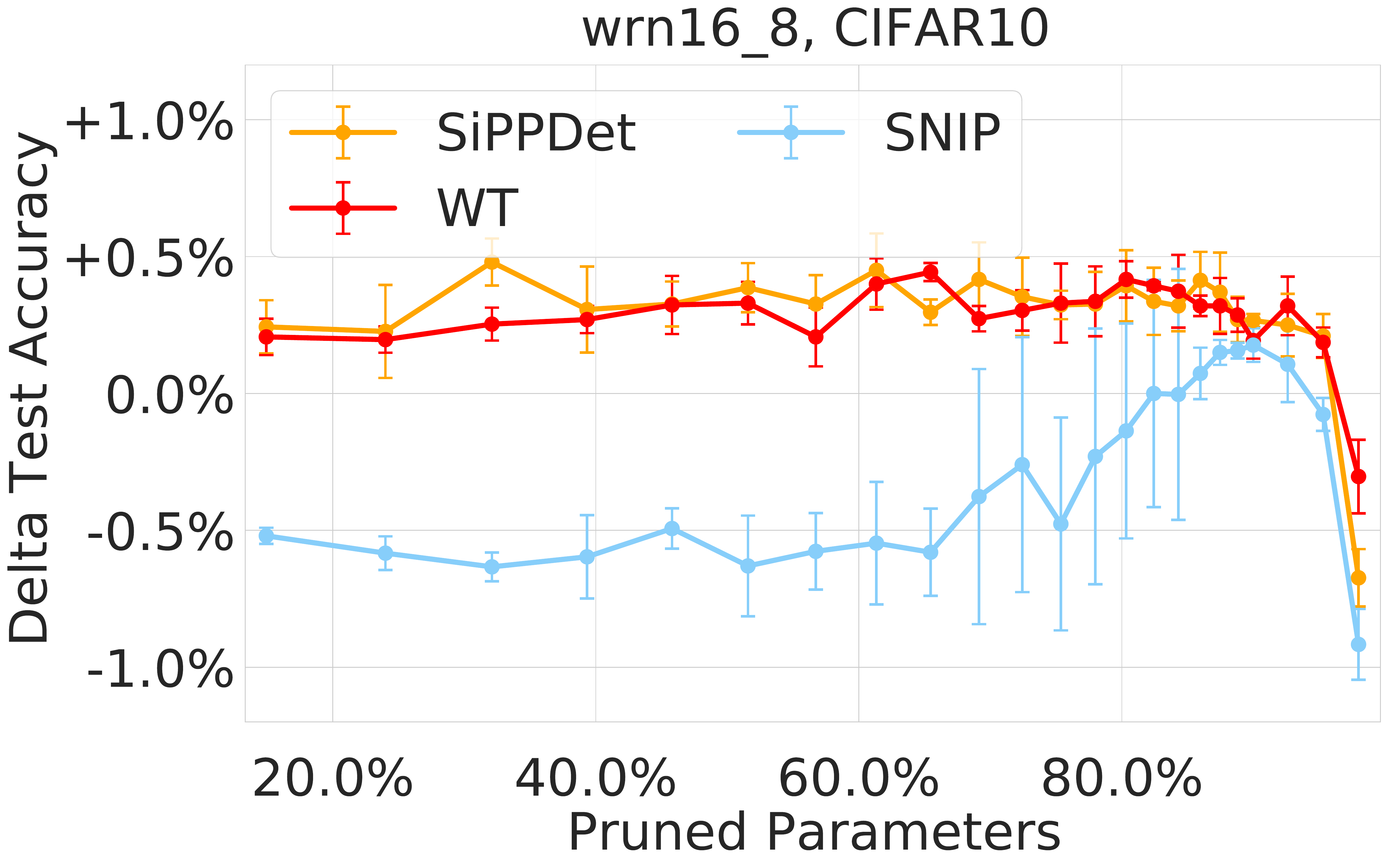}
    % \subcaption{WRN16-8}
\end{minipage}%
 
\caption{The delta in test accuracy to the uncompressed network for the generated pruned models trained on CIFAR10 for various target prune ratios. The networks were pruned using the \textbf{iterative prune+retrain} pipeline.}
\label{fig:cifar_prune_results}
%\vspace{-3ex}
\end{figure}

\subsection{Random-init + prune + train}

\paragraph{Methodology.}
On the other "extreme" of possible pruning pipeline, we consider the following scenario  as described in~\cite{lee2018snip}:
\begin{enumerate}
    \item 
    randomly initialize the network;
    \item
    prune the network to the desired prune ratio;
    \item
    train the network using the regular hyperparameters.
\end{enumerate}
While (due to the limited amount of training) this pipeline does not achieve as high prune ratios as the above scenario, it is simple and requires much less training epochs overall. It also serves as a useful experimental platform to understand if pruning methods are able to unearth important connections inherent in the network.

\paragraph{Results.}
In Figure~\ref{fig:cifar_prune_rand_results} the prune results for various CIFAR10 networks are shown. We note that for low prune ratios all pruning methods perform uniformly well, which most likely can be attributed to the overall overparameterization of the tested networks. For higher prune ratios, we observe vastly different performance. Specifically, \textsc{WT}'s performance drops to 10\% test accuracy (uniformly at random for CIFAR10) for prune ratios beyond 90\%. We suspect that weights do not contain sufficient information about the importance of the connection before training and thus \textsc{WT} fails. On the other hand, \textsc{Snip} performs consistently well due to the consideration of data and the gradients of weights. We note that \sipphybrid specifically, which adaptively mixes \sipp and \sipprand according to the theoretical bounds, performs well across all tested networks and achieves the same prune performance as \textsc{Snip}. For deeper networks (ResNet20 and ResNet56) in particular, we observe all \sipp variations performing well or even outperforming \textsc{Snip}.

\begin{figure}[t!]
\centering
\begin{minipage}[t]{0.33\textwidth}
    \includegraphics[width=\textwidth]{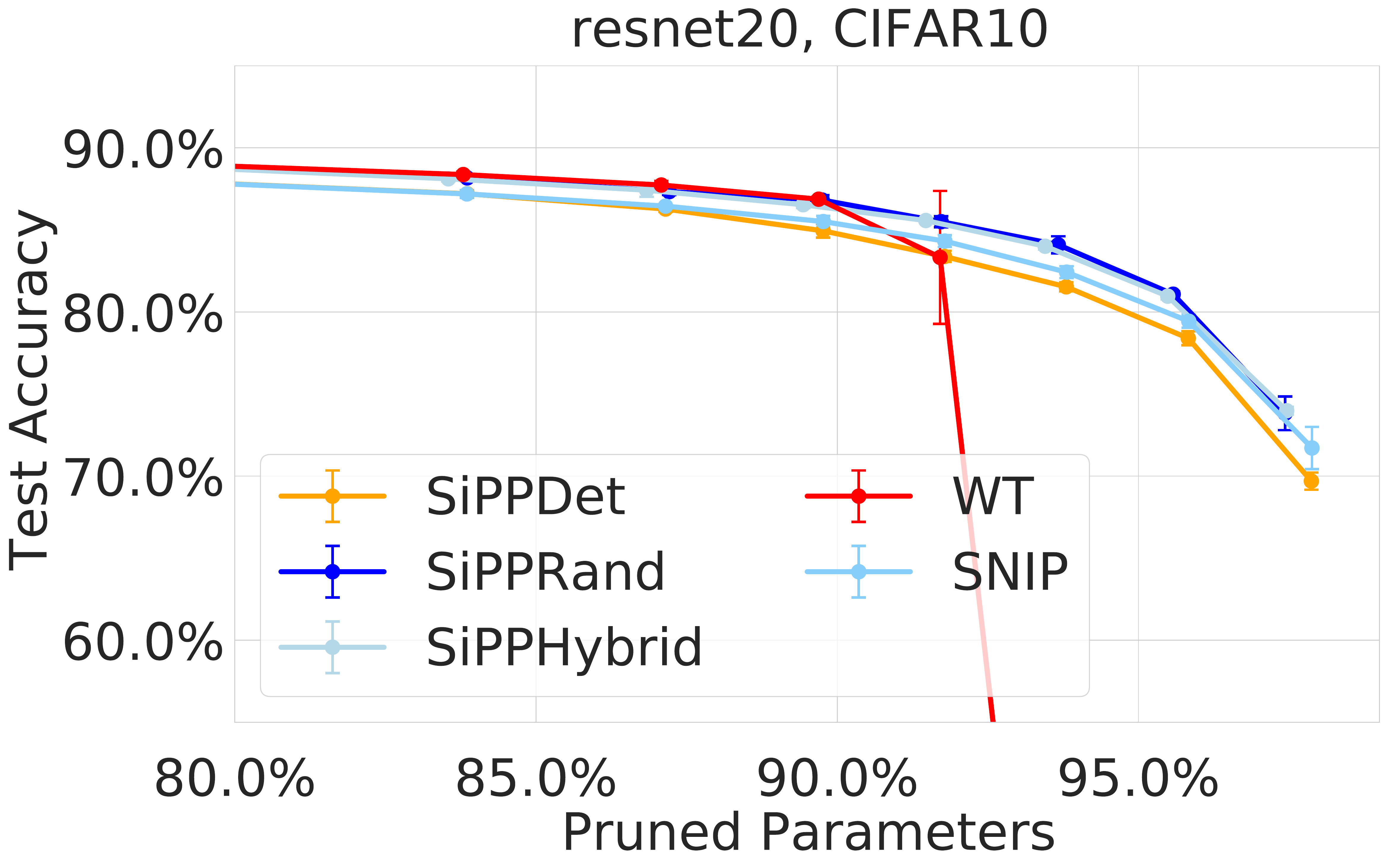}
    % \subcaption{Resnet20}
\end{minipage}%
\begin{minipage}[t]{0.33\textwidth}
    \includegraphics[width=\textwidth]{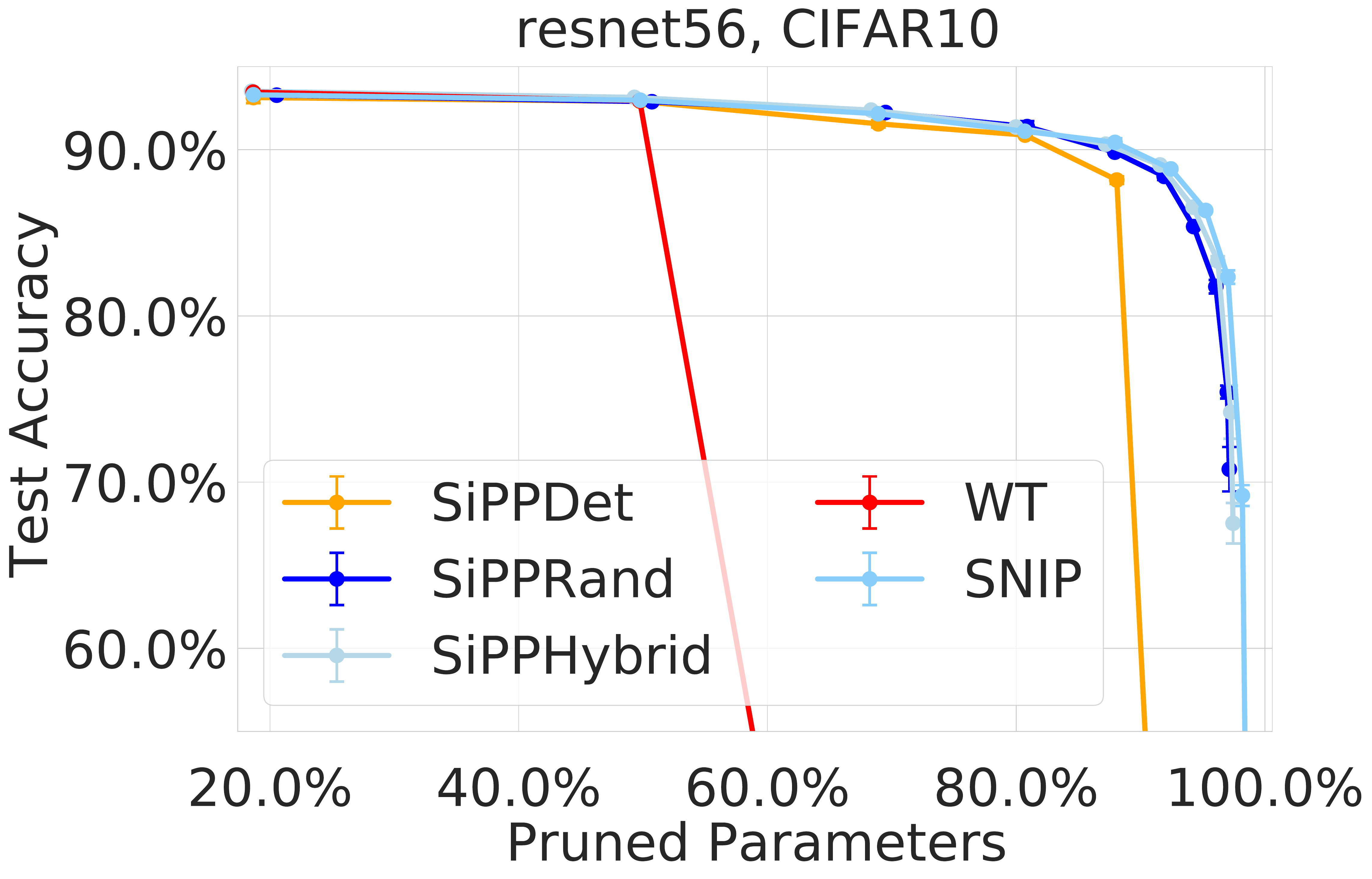}
    % \subcaption{Resnet56}
\end{minipage}%
\begin{minipage}[t]{0.33\textwidth}
    \includegraphics[width=\textwidth]{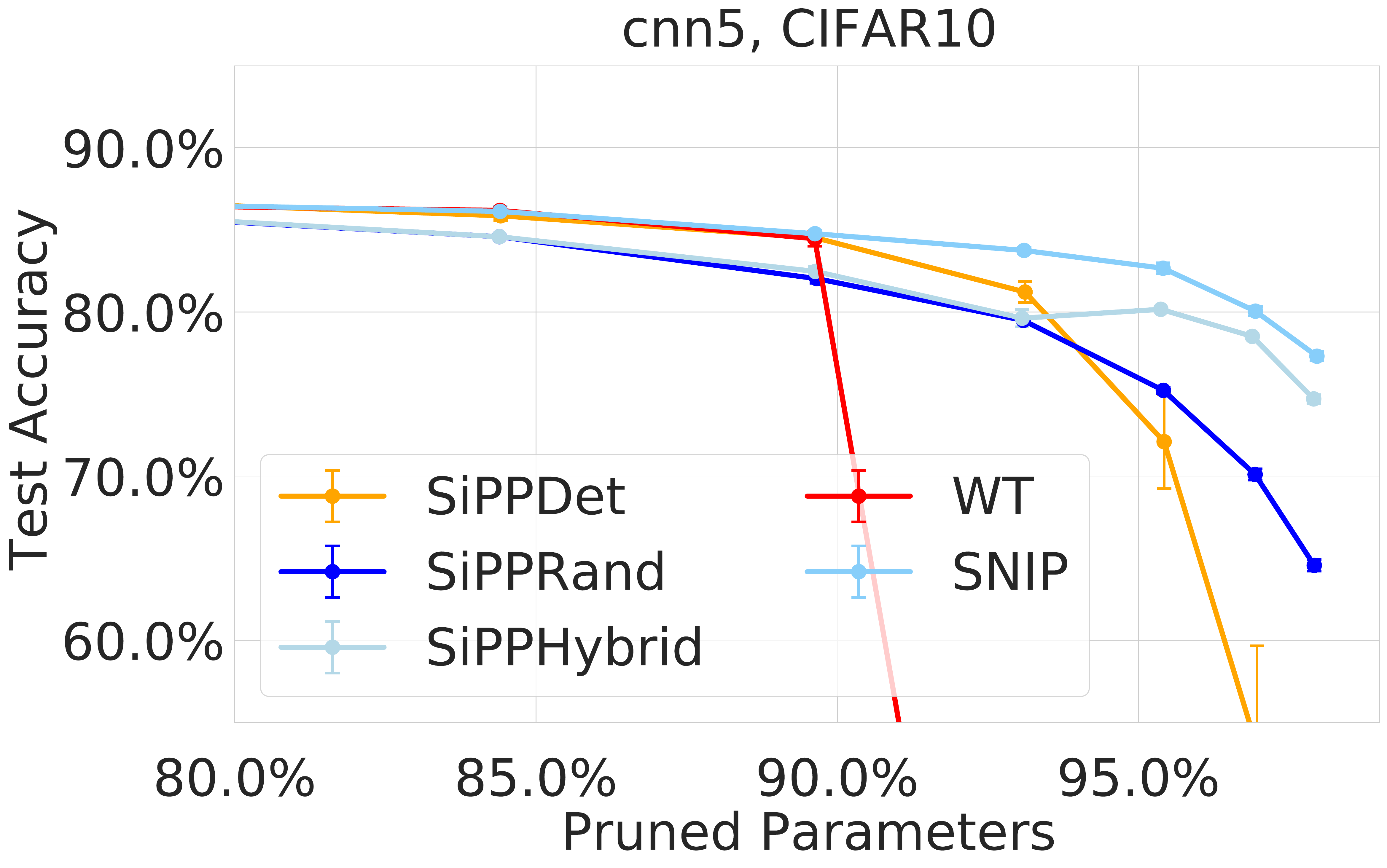}
    % \subcaption{Resnet110}
\end{minipage}
\caption{The test accuracy for the generated pruned models trained on CIFAR10 for various target prune ratios using the \textbf{random-init + prune + train} pipeline.}
\label{fig:cifar_prune_rand_results}
%\vspace{-3ex}
\end{figure}

%\vspace{-1ex}
\subsection{Discussion}
%\vspace{-1ex}
For our experiments we have embedded \sipp into two maximally diverse pruning pipelines in terms of the amount of (re-)training epochs, which constitutes the majority of computational cost in pruning. By doing so, we highlight the versatility and robustness of \sipp in performing well across many different tasks. While traditional pruning methods, such as \textsc{WT} and \textsc{Snip} (for further comparisons, see the  supplementary material), perform inconsistently when used in the context of alternative pruning pipelines we observe that \sipp serves as a consistent plug-and-play solution to the core pruning method of a pruning pipeline. 
Among the \sipp variants, we see that \sippdet tends to perform particularly well for small prune ratios (such as in the case of iterative prune+retrain) while \sipprand performs the best for extreme prune ratio (such as in the case of random-init + prune + train). \sipphybrid usually finds a close-to-optimal mixture of strategies and thus provides the most versatility among the \sipp variants, which comes at the cost of increased implementation effort. 
%\vspace{-1ex}
\section{Conclusion} 
\label{sec:conclusion}
%\vspace{-1ex}
In this work, we presented a simultaneously provably and practical family of network pruning methods, \textsc{SiPP}, that is grounded in a data-informed measure of sensitivity. Our analysis establishes provable guarantees that quantify the trade-off between the desired model sparsity and resulting accuracy of the pruned model establishing novel analytical compression bounds for a large class of neural networks.
\textsc{SiPP}'s versatility in providing strong prune results across a variety of tasks suggests that our method inherently considers the crucial pathways through the network, and does not merely operate by considering the properties, e.g., values, of the network parameters alone. 
We envision that \textsc{SiPP} can spur further research into network pruning by providing a robust core pruning method that can be reliably integrated into any pruning pipelines with close-to-optimal prune performance.

\subsubsection*{Acknowledgments}
This research was supported in part by the U.S. National Science Foundation (NSF) under Awards 1723943 and 1526815, Office of Naval Research (ONR) Grant N00014-18-1-2830, Microsoft, and JP Morgan Chase.

\bibliographystyle{plain}
\bibliography{misc/refs}

\clearpage
\newpage

\setcounter{table}{0}
\renewcommand{\thetable}{S\arabic{table}}
\setcounter{figure}{0}
\renewcommand{\thefigure}{S\arabic{figure}}
\setcounter{section}{0}
\renewcommand{\thesection}{S\arabic{section}}

\appendix
\section{Overview of Supplementary Material}
\label{sec:supp-overview}

In the following, we provide a quick overview of the material discussed in the supplementary material. We start out by providing a more complete problem definition including introducing additional notation that is required for our analysis (Section~\ref{sec:supp-problem-definition}). Subsequently, we introduce \sipp in full length including the generalization to all weights (Section~\ref{sec:supp-method}). We then provide a detailed analysis of \sipp and its variants to back up the informal claims of the main paper (Section~\ref{sec:supp-analysis}). Finally, we provide details and hyperparameters for our experimental setup and additional experimental results (Section~\ref{sec:supp-results}).
\section{Notation and Problem Definition}
\label{sec:supp-problem-definition}

The set of parameters $\theta$ of a neural network with $L$ layers is a tuple of multi-dimensional weight tensors corresponding to each layer, i.e., $\theta = (\WW^1, \ldots, \WW^L)$. The set of parameters $\theta$ defines the mapping $f_\theta: \XX \to \YY$ from the input space $\XX$ to the output space $\YY$. 
We consider the setting where we have access to independent and identically distributed (i.i.d.) samples $(x, y)$ from a joint distribution defined on $\XX \times \YY$ from which we can gather a training, test, and validation data set. To this end, we let $\DD$ denote the marginal distribution over the input space $\XX$. 

\subsection{Network Notation}

\paragraph{Layers.}
For a given input $x \sim \DD$, we denote the pre-activation and activation of layer $\ell$ by $Z^\ell(x)$ and $A^\ell(x)$, respectively. Note that 
$$
f_\theta(x) = A^L(x), \quad\quad A^0(x) = x, \quad\quad \text{and} \quad\quad A^\ell(x) = \phi^\ell(Z^\ell(x)),
$$
where $\phi^\ell(\cdot)$ denotes the activation function for layer $\ell$.
We consider any multi-dimensional layer that can be described by a linear map with \emph{parameter sharing}, e.g. fully-connected layers, convolutional layers, or LSTM cells. Specifically, for a layer $\ell$ the pre-activation $Z^\ell(x)$ of layer $\ell$ is described by the linear mapping of the activation $A^{\ell - 1}(x)$ with $\WW^\ell$, i.e., 
$$
Z^\ell(x) = \WW^\ell * A^{\ell - 1}(x),
$$
where $*$ denotes the operator of the linear map, e.g., the convolutional operator. 

\paragraph{Parameter groups.}
We denote by $c^\ell$ the number of parameter groups within a layer $\ell$ that do not interact with each other, e.g., individual filters in convolutional layers. Then, let 
$$
Z_i^\ell(x) = \WW^\ell_i * A^{\ell - 1}(x), \quad i \in [c^\ell],
$$ 
denote the $i^\text{th}$ pre-activation channel of layer $\ell$ produced by parameter group $\WW^\ell_i$. Then the entire pre-activations $Z^\ell(x)$ of a layer $\ell$ is constructed by appropriately concatenating the individual pre-activations from individual parameter groups, i.e., 
$$
Z^\ell(x) = \left[Z_1^\ell(x), \ldots, Z_{c^\ell}^\ell(x)\right].
$$
Moreover, we let $\eta^\ell \in \mathbb{N}$ denote the number of scalar values of $Z^\ell(\cdot)$ and let $\eta = \etaDef$. 
Finally, let $\rho$ denote the maximum number of parameters within a parameter group, i.e., $\rho = \max_{i, \ell} \size{\WW_i^\ell}$.

\paragraph{Patches.}
Within a parameter group, parameters may be used multiple times, c.f. \emph{parameter sharing}, in order to produce the output $Z_i^\ell(x) = \WW^\ell_i * A^{\ell - 1}(x)$. For example, in case of a convolutional layer the filter $\WW^\ell_i$ gets ``slid'' across the input $A^{\ell-1}(x)$ of the layer in order to produce one output pixel after another. Hereby, the filter acts on a distinct \emph{patch} of the layer input $A^{\ell-1}(x)$ in order to produce a specific output pixel $z(x) \in Z^{\ell}_i(x)$, where with slight abuse of notation $z(x)$ denotes a scalar entry of $Z^\ell_i(x)$. To precisely specify the associated operation that produces the output $z(x)$ we define by $\AA_i^\ell$ the \emph{set of patches} of the layer input $A^{\ell-1}(x)$ that are required to produce the output $Z^\ell_i(x)$. Specifically, let $a(\cdot) \in \AA_i^\ell$ denote some patch of $\AA_i^\ell$. Then, $a(\cdot) \subseteq A^{\ell-1}(\cdot)$ is defined such that a dot product between the parameter group $\WW^\ell_i$ and the patch $a(\cdot)$ produces the associated output scalar $z(x)$, i.e.,
\begin{equation*}
    z(x) = \dotp{\WW^\ell_i}{a(x)} = \sum_{k \in \II_i^\ell} w_k a_k(x),
\end{equation*}
where $\II_i^\ell$ denotes the index set of weights for the parameter group $\WW^\ell_i$ and $w_k, a_k(x)$ denote a scalar entry of the parameter group and patch for some input $x \sim \DD$, respectively. Note that $\eta^\ell = \sum_{i \in [c^\ell]} \abs{\AA_i^\ell}$.

The notation of patch maps $\AA$ lets us conveniently abstract away some of the implementation details of the linear map $*$ without restricting ourselves to a particular type of linear map $*$. 
For example in the context of convolutional layers, the actual linear map $*$ can significantly vary depending on the parameter settings such as stride length, padding, and so forth. 
It also enables us to consider other layers, such as recurrent layers, at the same time. In this case, $\AA$ can be generated by considering each recursive input to the layer as a separate patch.

In the case of two-dimensional convolutions (i.e. for images), we note that our notion of patch maps corresponds to the \textsc{Unfold} operation in PyTorch~\cite{pytorch2020unfold}, which we find to be a helpful reference to further contextualize the concept of patch maps.

\subsection{Problem Definition}
We now proceed to formally state the problem definition that motivates the use of \sipp and subsequent analysis. To this end, let the size of the parameter tuple $\param$, $\size{\param}$, to be the number of all non-zero entries in the weight tensors $\WW^1,\ldots,\WW^L$.

\begin{problem}
\label{prob:network-compression}
For given $\epsilon, \delta \in (0,1)$, our overarching goal is to use a pruning algorithm to generate a sparse reparameterization $\paramHat$ of $\theta$ such that $\size{\paramHat} \ll \size{\param}$ and for $\Point \sim \DD$ the $\ell_2$-norm of the reference network output $\f(\Input)$ can be approximated by $\fHat(\Input)$ up to $1 \pm \eps$
multiplicative error with probability greater than $1- \delta$, i.e., 
$$
\Pr \nolimits_{\paramHat, \Point} \left(\norm{\fHat (\Input) - \f(\Input)} \leq \eps \norm{\f(\Input)}\right) \ge 1 - \delta,
$$
where $\norm{\cdot} = \norm{\cdot}_2$ and $\Pr \nolimits_{\paramHat, \Point}$ considers the randomness over both the pruning algorithm and the network's input.
\end{problem}

\section{Method}
\label{sec:supp-method}
In this section, we provide additional details for \sipp as introduced in the main part of the paper.

\subsection{Overview}
Algorithm~\ref{alg:supp-sipp} provides an extended over view of \sipp. Moreover, in Algorithm~\ref{alg:sparsify-weights} we present \textsc{Sparsify}, which is the sub-routine to adaptively prune weights from a parameter group according to either \sippdet, \sipprand, or \sipphybrid. 

\begin{algorithm}[t!]
\small
\caption{\sipp$(\param, \BB, \delta)$}
\label{alg:supp-sipp}
\textbf{Input:}
$\param = \paramDef$: weights of the uncompressed neural network;
$\BB \in \NN$: sampling budget; 
$\delta \in (0, 1)$: failure probability;
\\
\textbf{Output:} 
$\paramHat = \paramHatDef$ : sparse weights 

\begin{spacing}{1.3}
\begin{algorithmic}[1]
\small

\STATE $\SS \gets \text{Uniform sample (without replacement) of } \SizeOfS[8 \, \eta]$ \text{points from validation set} \label{lin:set-s}

\STATE $\{ m_i^{\ell}\}_{i, \ell} \gets \textsc{OptAlloc}(\theta, \BB, \SS) \; \forall i \in [c^\ell],\; \forall \ell \in [L]$ \COMMENT{optimally allocate budget $\BB$ applying Lemma~\ref{lem:sipprand} to evaluate the resulting relative error guarantees} 

\FOR{$\ell \in \compressiblelayers$}

    \STATE $\WWHat^\ell \gets \textbf{0}$; \COMMENT{Initialize a null tensor}
    
    \FOR{$i \in [c^\ell]$}
    
        \STATE $\{s_j \}_j \gets \textsc{EmpiricalSensitivity}(\theta, \SS, i, \ell) \; \forall w_j \in \WW_i^\ell$ \COMMENT{Compute parameter importance for each weight $w_j$ in the parameter group according to Definitions~\ref{def:empirical-sens} and~\ref{def:generalized-param-importance}} 
    
        \STATE $\WWHat_i^\ell \gets \textsc{Sparsify}(\WW_i^\ell, m_i^\ell, \{s_j\}_j)$ \COMMENT{prune weights according to \sippdet, \sipprand, or \sipphybrid such that only $m_i^\ell$ weights remain in the parameter group}
	
	\ENDFOR

\ENDFOR

\STATE \textbf{return} $\paramHat = \paramHatDef$;

\end{algorithmic}
\end{spacing}
\end{algorithm}

\begin{algorithm}[t!]
\small
\caption{\textsc{Sparsify}$(\WW_i^\ell, m_i^\ell, \{s_j\}_j)$}
\label{alg:sparsify-weights}
\textbf{Input:} 
$\WW_i^\ell$: parameter group to be pruned;
$m_i^\ell$: assigned budget;
$\{s_j\}_j$: sensitivities associated with weights in parameter group
\\
\textbf{Output:} 
$\hat \WW_i^\ell$: sparse parameter group

\begin{spacing}{1.3}
\begin{algorithmic}[1]
\small

\STATE $S \gets \sum_{j \in \II_i^\ell} s_j$ \COMMENT{Compute sum of sensitivities}

\STATE  $\tilde S \gets \STildeDef{16 \eta}$

\STATE $\{q_j\}_j \gets \frac{s_j}{S}$ \COMMENT{Compute sample probabilities for \sipprand}

\STATE $N \gets N(m_i^\ell, \{q_j\}_j)$ \COMMENT{Get expected number of required samples to obtain $m_i^\ell$ unique samples} \label{lin:unique-samples}

\STATE $\II_{det} \gets $ subset of indices from $\II_i^\ell$ corresponding to the largest $m_i^\ell$ sensitivities (c.f. Lemma~\ref{lem:sippdet})

\STATE $\eps_{rand} \gets \frac{\tilde S  + \sqrt{\tilde S  (\tilde S  + 6 N)}}{N}$ \COMMENT{c.f. Lemma~\ref{lem:sipprand}} \label{lin:eps-rand}

\STATE $\eps_{det} \gets C \sum_{j \in (\II \setminus \II_{det})} s_j$ \COMMENT{c.f. Lemma~\ref{lem:sippdet}} \label{lin:eps-det}

\IF{($\eps_{rand} > \eps_{det}$ \textbf{or} always $\sippdet)$ \textbf{and} not always \sipprand} \label{lin:choose-strategy}
    \STATE $\hat \WW_i^\ell \gets $ prune weights from $\WW_i^\ell$, i.e set to $0$, that are not in $\II_{det}$ and keep the rest \label{lin:sippdet}
\ELSE 
    \STATE $\{n_j\}_j \sim \textsc{Multinomial}(\{q_j\}_j, N)$ \COMMENT{Sample $N$ times and return the counts} \label{lin:sipprand-start}
    
    \STATE $\hat \WW_i^\ell \gets $ prune weights such that $\hat w_j = \frac{n_j}{N q_j} w_j$ for each weight $\hat w_j$, $j \in \II_i^\ell$ in the parameter group
    \label{lin:sipprand-end}
\ENDIF
\STATE \textbf{return} $\hat \WW_i^\ell$

\end{algorithmic}
\end{spacing}
\end{algorithm}

\subsection{Details regarding \textsc{OptAlloc}}
As mentioned in Section~4, \textsc{OptAlloc} proceeds by minimizing the sum of relative error guarantees associated with each parameter group for a given overall weight budget $\BB$, i.e.,

$$
\begin{aligned}
    \textstyle
    &\min_{m_i^\ell \in \NN \; \forall i \in [c^\ell], \; \forall \ell \in [L]} 
    && \textstyle
    \sum_{\ell \in [L], \; i \in [c^\ell]} \epsilon_i^\ell(m_i^\ell)
    & \text{s. t.} 
    && \textstyle 
    \sum_{\ell \in [L], \; i \in [c^\ell]} m_i^\ell \leq \BB.
\end{aligned}
$$
Hereby, $m_i^\ell$ and $\eps_i^\ell(m_i^\ell)$ denote the desired number of weights and the associated theoretical error in parameter group $\WW_i^\ell$ after pruning, respectively. We evaluate the theoretical error according to Lemmas~\ref{lem:sippdet} and Lemma~\ref{lem:sipprand} when pruning with \sippdet and \sipphybrid or \sipprand, respectively. Note that in order to evaluate Lemma~\ref{lem:sipprand} we have to first convert $m_i^\ell$ to the expected number of required samples in order to obtain $m_i^\ell$ unique samples, which is also shown in Line~\ref{lin:unique-samples} of Algorithm~\ref{alg:sparsify-weights}.

\subsection{Details regarding \textsc{EmpiricalSensitivity}}
We note that the empirical sensitivity (ES) $s_j$ of a weight $w_j$ in the parameter group $\WW_i^\ell$ is given by Definition~\ref{def:empirical-sens}, where we define ES as the maximum of the relative parameter importance $g_j(x)$ over a set $\SS$ of i.i.d.\ data points. To account for both negative weights and activations, we utilize the generalized parameter importance as defined in Definition~\ref{def:generalized-param-importance} to compute $g_j(x)$ for a particular input $x$. To ensure that ES holds with probability at least $1-\delta$ for all patches and parameters simultaneously we have to appropriately choose the size of $\SS$, c.f. Line~\ref{lin:set-s} of Algorithm~\ref{alg:supp-sipp} and Section~\ref{sec:supp-analysis-generalization}.

\subsection{Details regarding \textsc{Sparsify}}
In Algorithm~\ref{alg:sparsify-weights} we present the pruning strategy for both \sippdet and \sipprand as shown in Line~\ref{lin:sippdet} and Lines~\ref{lin:sipprand-start},~\ref{lin:sipprand-end}, respectively. 
Recall that \sipphybrid adaptively chooses between both strategies according to the associated error guarantees, which get computed in Lines~\ref{lin:eps-det} and~\ref{lin:eps-rand} for \sippdet and \sipprand, respectively. We then choose the better strategy accordingly, see Line~\ref{lin:choose-strategy}. 
We can also choose to always prune using \sipprand or \sippdet as indicated in Line~\ref{lin:choose-strategy}.

\subsection{Simple \sipp}
We can greatly simplify our pruning algorithm if we prune all parameter groups using \sippdet. To see this consider the solution to \textsc{OptAlloc} when evaluating the relative error according to Lemma~\ref{lem:sippdet}. Since the relative error for a particular parameter group in this case is the sum over sensitivities that were not included and the objective of \textsc{OptAlloc} is to minimize the sum over all relative errors the optimal solution is to \emph{globally} keep the weights with largest sensitivity. In other words, pruning with \sippdet only results in global thresholding of weights according to their sensitivity. The resulting procedure is shown in Algorithm~\ref{alg:sipp-simple}.
Note that this procedure is very reminiscent of simple, global weight thresholding~\cite{Han15, renda2020comparing} but using sensitivity instead of the magnitude of the weights as prune criterion. In contrast to weight thresholding, however, \sippsimple still exhibits the same theoretical error  guarantees as \sipp.

\begin{algorithm}[htb]
\caption{\sippsimple$(\param, \BB, \delta)$}
\label{alg:sipp-simple}
\textbf{Input:}
$\param = \paramDef$: weights of the uncompressed neural network;
$\BB \in \NN$: sampling budget; 
$\delta \in (0, 1)$: failure probability;
\\
\textbf{Output:} 
$\paramHat = \paramHatDef$ : sparse weights 

\begin{spacing}{1.5}
\begin{algorithmic}[1]

\STATE $\SS \gets \text{Uniform sample of } \SizeOfS[16 \, \eta]$ \text{points from validation set}

\STATE Compute sensitivity for all weights in the network using $\SS$

\STATE Prune weights globally by keeping the $\BB$ weights with largest sensitivity

\STATE Return $\paramHat = \paramHatDef$

\end{algorithmic}
\end{spacing}
\end{algorithm}

\section{Analysis}
\label{sec:supp-analysis}
In this section, we establish the theoretical guarantees of \sipp as presented in Algorithm~\ref{alg:supp-sipp} and state our main compression theorem.

\subsection{Outline}
\label{sec:supp-analysis-outline}

We begin by considering the sparsification of an arbitrary output patch in an arbitrary parameter group and layer assuming that both the input to the layer and the weights are non-negative. To this end, we first establish the empirical sensitivity (ES) inequality that quantifies the contribution of an individual scalar weight to an output patch (Section~\ref{sec:supp-analysis-es} and informal Lemma~5.1). We then establish the relative error guarantees for each variation of \sipp for an arbitrary output patch (Section~\ref{sec:supp-analysis-error} and informal Lemmas~5.2,~5.3).
Next, we formally generalize the approximation scheme to arbitrary weights and input activations (Section~\ref{sec:supp-analysis-generalization}). Finally, we provide our formal network compression bounds by composing together the error guarantees from individual layers and parameter groups. (Section~\ref{sec:supp-analysis-composition}).

\subsection{Empirical Sensitivity}
\label{sec:supp-analysis-es}

Recall that an arbitrary parameter group indexed by $i \in [c^\ell]$ in an arbitrary layer $\ell \in [L]$, is denoted by $\WW_i^\ell$ and $\II$ denotes its parameter index set. Moreover, let $w_j$ denote some scalar entry of $\WW_i^\ell$ for some $j \in \II$. Also as before, $A^{\ell-1}(x)$ denotes the input activation to layer $\ell$ and $Z_i^\ell(x) = \WW_i^\ell * A^{\ell-1}(x)$ denotes the output pre-activation of parameter group $\WW_i^\ell$. Finally, recall that 
$a(\cdot) \in \AA_i^\ell$ denotes some patch of $\AA_i^\ell$ and that the patch $a(\cdot)$ produces the associated output scalar $z(x)$, i.e.,
$
    z(x) = \dotp{\WW^\ell_i}{a(x)} = \sum_{k \in \II_i^\ell} w_k a_k(x),
$
We now proceed with the formal definition of relative parameter importance and empirically sensitivity, which is defined as the maximum relative parameter importance over multiple data points.

\begin{definition}[Relative parameter importance]
For a scalar parameter $w_j$, $j \in \II_i^\ell$, of parameter group $W_i^\ell$ in layer $\ell$, its relative importance $g_j(x)$ is given by
$$
g_j(x) = \gDef{\Input},
$$
where $\AA_i^\ell$ denotes the set of patches for parameter group $W_i^\ell$.
\end{definition}

\begin{definition}[Empirical sensitivity]
\label{def:empirical-sens}
Let $\SS$ be a set of i.i.d.\ samples from the validation data set. Then, the empirical sensitivity $s_j(x)$ of a scalar parameter $w_j$, $j \in \II_i^\ell$, of parameter group $W_i^\ell$ in layer $\ell$ is given by
$$
s_j(x) = \max_{x \in \SS} g_j(x).
$$
\end{definition}
We note that, for ease of notation, we do not explicitly enumerate ES over $i$ and $\ell$ for parameter groups and layers, respectively.

To ensure that a \emph{small} batch of points $\SS$ suffices for an accurate approximation of parameter importance, we impose the following mild regularity assumption on the Cumulative Distribution Function (CDF) of $\g{\Point}$ similar to the assumption of~\cite{blg2018}.

\begin{assumption}[Regularity assumption]
\label{asm:cdf}
There exist universal constants $C, K > 0 $ such that for all $j \in \II_i^\ell$, the CDF of the random variable $\g{\Input} \in [0, 1]$ for $\Point \sim \DD$, denoted by $\cdf{\cdot}$, satisfies
$$
\cdf{\nicefrac{1}{C}} \leq \exp\left(-\nicefrac{1}{K}\right).
$$
\end{assumption}

Traditional distributions such as the Gaussian, Uniform, and Exponential, among others, supported on the interval $[0,1]$ satisfy Assumption~\ref{asm:cdf} with sufficiently small values of $K$ and $K'$. In other words, Assumption~\ref{asm:cdf} ensures that there are no outliers of $g_j(x)$ with non-negligible probability that are \emph{not} within a constant multiplicative factor of most other values of $g_j(x)$. Capturing outliers that are within a constant multiplicative factor, on the other hand, can be captured by considering an appropriate scaling factor of ES in the ES inequality (see Lemma~\ref{lem:es-inequality} below). However, we cannot capture these non-negligible outliers (unless we significantly increase the cardinality of $\SS$) when they are not within a constant multiplicative factor.

We now proceed to state the ES inequality as informally stated in Lemma~5.1 in the main body of the paper. We note that, intuitively, the ES inequality enables us to quantify ,i.e. upper-bound, the contribution $w_j a_j(x)$ coming from an individual weight w.h.p.\ in terms of the output patch $z(x)$ and the sensitivity $s_j$ of the weight.
\begin{lemma}[ES inequality]
\label{lem:es-inequality}
For $\delta \in (0, 1)$, the ES $s_j$ of the scalar parameter $w_j$, $j \in \II_i^\ell$, of parameter group $W_i^\ell$ computed with a set $\SS$ of i.i.d.\ data points, $\abs{\SS}= \SizeOfS[]$, satisfies
$$
\Pr_{x} \left( w_j a_j(x) \leq C s_j z(x) \right) \geq 1 - \delta \quad \forall j \in \II_i^\ell,
$$
for some input $x \sim \DD$ and some fixed input patch $a(\cdot) \in \AA_i^\ell$, where $C, K$ are the universal constants of Assumption~\ref{asm:cdf} and $z(x) = \sum_{k \in \II_i^\ell} w_k a_k(x)$.
\end{lemma}
\begin{proof}
We consider a fixed weight $w_j$ from the parameter group and a fixed input patch $a(\cdot)$. Note that \begin{equation}
\label{eq:rel-importance-inequality}
    \frac{w_j a_j(x)}{z(x)} \leq g_j(x)
\end{equation}
by definition of $g_j(x)$ since the relative parameter importance is the maximum over patches for a specific input $x$. We now consider the probability that $s_j(\SS) = \max_{x' \in \SS} g_j(x')$ is not an upper bound for $g_j(x)$ when appropriately scaled for random draws over $\SS$, where we explicitly denote the dependency of $s_j(\SS)$ on $\SS$ for the purpose of this proof. By showing this occurs with low probability we can then conclude that $s_j$ is indeed an upper bound for $g_j(x)$ most of the time. Specifically,
\begin{align*}
    \Pr_{\SS} \left( C s_j(\SS) \leq g_j(x) \right)  
    &= \Pr_{\SS} \left( s_j(\SS) \leq \nicefrac{g_j(x)}{C} \right) \\
    &\leq \Pr_{\SS} \left( s_j(\SS) \leq \nicefrac{1}{C} \right) 
    & \text{since $g_j(x) \leq 1$} \\
    &= \Pr_{\SS} \left( \max_{x' \in \SS} g_j(x') \leq \nicefrac{1}{C} \right) 
    & \text{by definition of $s_j(\SS)$} \\
    &= \left(\Pr_{x'} \left( g_j(x') \leq \nicefrac{1}{C} \right) \right)^{\abs{\SS}}
    & \text{since $\SS$ is $\abs{\SS}$ i.i.d. draws from $\DD$} \\
    &= F_j\left( \nicefrac{1}{C} \right)^{\abs{\SS}}
    & \text{since $F_j(\cdot)$ is the CDF of $g_j(x')$} \\
    & \leq \exp \left(-\nicefrac{\abs{\SS}}{K}\right) 
    & \text{by Assumption~\ref{asm:cdf}} \\
    & = \frac{\delta}{\rho} 
    & \text{since $\abs{\SS} = \SizeOfS[]$ by definition.}
\end{align*}
Thus, we can conclude that for a fixed weight $w_j$ and some input $x \sim \DD$ its relative contribution $g_j(x)$ is upper bound by its sensitivity $s_j$. Moreover, the inequality also holds for any weight $w_j$ by the union bound, i.e.,
\begin{align*}
    \Pr_\SS \left( \exists j \in \II_i^\ell | C s_j(\SS) \leq g_j(x) \right)
    & \leq \abs{\II_i^\ell} \Pr_{\SS} \left( C s_j(\SS) \leq g_j(x) \right)  
    & \text{by the union bound} \\
    & \leq \abs{\II_i^\ell} \frac{\delta}{\rho} 
    & \text{by the analysis above} \\
    & \leq \rho \frac{\delta}{\rho} 
    & \text{since $\rho = \max_{i, \ell} \abs{\II_i^\ell}$} \\
    & = \delta
\end{align*}
We thus have with probability at least $1 - \delta$ over the construction of $s_j$ that 
$$
C s_j \geq g_j(x) \quad \forall j \in \II_i^\ell
$$
and by~\eqref{eq:rel-importance-inequality} that 
$$
\frac{w_j a_j(x)}{z(x)} \leq g_j(x) \leq C s_j,
$$
which concludes the proof since the above inequality holds for any $x \sim \DD$.
\end{proof}

\subsection{Error Guarantees for positive weights and activations}
\label{sec:supp-analysis-error}
Equipped with Lemma~\ref{lem:es-inequality} we now proceed to establish the relative error guarantees for the three variants of \sipp. As before, we consider a fixed output patch for a fixed parameter group and we assume that both input activations and the weights are non-negative.

\subsubsection{Error Guarantee for \sippdet}
\label{sec:supp-analysis-sippdet}
Recall that \sippdet prunes weights from a parameter group by keeping only the weights with largest sensitivity. Let the index set of weights kept be denoted by $\II_{det}$. Below we state the formal error guarantee for a fixed output patch of a parameter group when we only keep the weights indexed by $\II_{det}$. Note that the below error guarantee holds for any index set $\II_{det}$ that we decide to keep. Naturally, however, it makes sense to keep the weights with largest sensitivity as this minimizes the associated relative error.

\begin{lemma}[\sippdet error bound]
\label{lem:sippdet}
For $\delta \in (0, 1)$, pruning parameter group $\WW_i^\ell$ by keeping only the weights indexed by $\II_{det} \subseteq \II_i^\ell$ generates a pruned parameter group $\hat\WW_i^\ell$ such that for a fixed input patch $a(\cdot) \in \AA_i^\ell$ and $x \sim \DD$
\begin{align*}
\Pr \left( \abs{\hat z_{det}(x) - z(x)} \geq \epsilon_{det} z(x) \right)  \leq \delta \quad \text{with} \quad
\epsilon_{det} = C \sum_{j \in \II \setminus \II_{det}} s_j \in (0, 1),
\end{align*}
where $z(x) = \dotp{\WW_i^\ell}{a(x)} = \sum_{j \in \II_i^\ell} w_j a_j(x)$ and $\hat z_{det}(x) = \dotp{\hat\WW_i^\ell}{a(x)} = \sum_{j \in \II_{det}} w_j a_j(x)$ denote the unpruned and approximate output patch, respectively, associated with the input patch $a(\cdot)$. The sensitivities $\{s_j\}_{j \in \II_i^\ell}$ are hereby computed over a set $\SS$ of $\SizeOfS[]$ i.i.d.\ data points drawn from $\DD$.
\end{lemma}

\begin{proof}
We proceed by considering the absolute difference $\abs{z(x) - \hat z_{det}(x)}$ and note that 
\begin{align*}
\abs{z(x) - \hat z_{det}(x)}
&= \abs{\dotp{\WW_i^\ell}{a(x)} - \dotp{\hat\WW_i^\ell}{a(x)}}\\
&= \abs{\sum_{j \in \II_i^\ell} w_j a_j(x) - \sum_{j \in \II_{det}} w_j a_j(x)} \\
&= \sum_{j \in \II_i^\ell \setminus \II_{det}} w_j a_j(x)
\end{align*}
Invoking Lemma~\ref{lem:es-inequality} we know that with probability at least $1 - \delta$ each individual weight term in the above sum is upper bound by its sensitivity, i.e.,
$$
w_j a_j(x) \leq C s_j z(x) \quad \forall{j \in \II_i^\ell}.
$$
We now bound the error in terms of sensitivity as 
\begin{align*}
\abs{z(x) - \hat z_{det}(x)}
&= \sum_{j \in \II_i^\ell \setminus \II_{det}} w_j a_j(x) \\
& \leq \sum_{j \in \II_i^\ell \setminus \II_{det}} C s_j z(x) & \text{using the above inequality} \\
&= \eps_{det} z(x) & \text{by definition of $\eps_{det}$}.
\end{align*}
We conclude by mentioning that above error bound holds with probability at least $1 - \delta$ since the associated ES inequalities hold with probability at least $1 - \delta$.
\end{proof}

\subsubsection{Error Guarantee for \sipprand}
\label{sec:supp-analysis-sipprand}
As before, we consider a fixed parameter group $\WW_i^\ell$, which has been assigned a budget of $m_i^\ell$ unique weights to be kept. Recall that \sipprand is a sampling procedure that proceeds as follows:
\begin{enumerate}
    \item 
    Assign probabilities $q_j = \nicefrac{s_j}{\sum_{k \in \II_i^\ell} s_k}$ for all $j \in \II_i^\ell$.
    \item
    Compute the expected number of samples, $N$, to obtain $m_i^\ell$ unique weights from the sampling procedure.
    \item
    Sample weights $N$ times with replacement from $W_i^\ell$ according to $q_j$.
    \item
    Reweigh the weights, $\hat w_j$, to obtain the approximate weights such that $\hat w_j = \frac{n_j}{N q_j} w_j$, where $n_j$ denotes the number of times $w_j$ was sampled.
\end{enumerate}
We note that if a weight has not been sampled, i.e. $n_j = 0$, we can drop it since the resulting weight is $0$.
We now consider the resulting error bound when sampling $N$ times with replacement.

\begin{lemma}[\sipprand error bound]
\label{lem:sipprand}
For $\delta \in (0, 1)$, pruning parameter group $\WW_i^\ell$ by sampling weights $N = N(m_i^\ell)$ times with replacement, such that weight $w_j$ is sampled with probability $q_j = \nicefrac{s_j}{\sum_{k \in \II_i^\ell} s_k}$, generates a pruned parameter group $\hat\WW_i^\ell$ such that for a fixed input patch $a(\cdot) \in \AA_i^\ell$ and $x \sim \DD$

\begin{align*}
\Pr \left( \abs{\hat z_{rand}(x) - z(x)} \geq \epsilon_{rand} z(x) \right)  \leq \delta \,\, \text{and} \,\,
\epsilon_{rand} = \left(\epsilonNeuron\right) \in (0,1),
\end{align*}
where $\hat z_{rand}(x) = \dotp{\hat\WW_i^\ell}{a(x)}$ and $z(x) = \dotp{\WW_i^\ell}{a(x)}$ are with respect to patch map $a(\cdot)$ as before, 
$\tilde S = \STildeDef{4}$, and $S = \sum_{j \in \II_i^\ell} s_j$. The sensitivities $\{s_j\}_{j \in \II_i^\ell}$ are hereby computed over a set $\SS$ of $\SizeOfS[2]$ i.i.d.\ data points drawn from $\DD$.
\end{lemma}

\begin{proof}
Our proof closely follows the proof of Lemma 1 of~\cite{blg2018}. The sampling procedure of sampling $N$ with replacement is equivalent to sequentially constructing a \emph{multiset} consisting of $N$ samples from $\II_i^\ell$ where each $j \in \II_i^\ell$ is sampled with probability $q_j$. Now, let $\CC = \{c_1, \ldots, c_N\}$ be that multiset of weight indices $\II_i^\ell$ used to construct $\hat\WW_i^\ell$.
Let $a(\cdot) \in \AA_i^\ell$ be arbitrary and fixed, let $x \sim \DD$ be an i.i.d.\ sample from $\DD$, and let
$$
\hat z_{rand}(x) = \dotp{\hat\WW_i^\ell}{a(x)} = \sum_{j \in \CC} \frac{w_j}{N q_j} a_j(x)
$$
be the approximate intermediate value corresponding to the sparsified tensor $\hat\WW_i^\ell$ and let
$$
z(x) = \sum_{j \in \II_i^\ell} w_j a_j(x)
$$
as before.
Define $\mNeuron$ random variables $\T[c_1], \ldots, \T[c_\mNeuron]$ such that for all $\edge \in \CC$
\begin{align}
\label{eqn:tplu-defn}
T_j = \frac{w_j a_j (x) }{N q_j }=  \frac{S w_j a_j(x)}{N s_j}.
\end{align}
For any $j \in \CC$, we have for the expectation of $T_j$:
$$
\E [T_j] = \sum_{k \in \II_i^\ell} \frac{w_k a_k (x) }{N q_k } q_k = \frac{z(x)}{N}.
$$

Let $T = \sum_{j \in \CC} T_j = \hat z_{rand}(x)$ denote our approximation and note that by linearity of expectation,
$$
\E[T] = \sum_{j \in \CC} \E [T_j] = z(x).
$$
Thus, $\hat z_{rand}(x) = T$ is an unbiased estimator of $z(x)$ for any $x \sim \DD$.

For the remainder of the proof we will assume that $z(x) > 0$, since otherwise, $z(x) = 0$ if and only if $T_j = 0$ for all $j \in \CC$ almost surely, in which case the lemma follows trivially. 
We now proceed with the case where $z(x) > 0$ and invoke Lemma~\ref{lem:es-inequality} (ES inequality) with $\SS$ consisting of $\SizeOfS[2]$ i.i.d.\ data points, which implies that 
\begin{equation}
\label{eq:es-inequality-sipprand}
w_j a_j(x) \leq C s_j z(x) \quad \forall{j \in \II_i^\ell} \quad\quad \text{with probability at least} \quad \quad 1 - \frac{\delta}{2}.
\end{equation}

Consequently, we can bound the variance of each $T_j, \, j \in \CC$ with probability at least $1 - \nicefrac{\delta}{2}$ as follows
\begin{align*}
\Var(T_j) 
&\leq \E[T_j^2] \\
&= \sum_{k \in \II_i^\ell} \frac{(w_k a_k(x))^2}{(N q_k)^2} q_k \\
&= \frac{S}{N^2} \sum_{k \in \II_i^\ell} \frac{w_k a_k(x)}{s_j} w_k a_k(x) \\
&\leq \frac{S}{N^2} C z(x) \sum_{k \in \II_i^\ell} w_k a_k(x) 
& \text{by the ES inequality as stated in~\eqref{eq:es-inequality-sipprand}} \\
&= \frac{S C z(x)^2}{N^2}.
\end{align*}
Since $T$ is a sum of independent random variables, we obtain 
\begin{align}
\label{eqn:varplu-bound}
\Var(T) &= N \Var(T_j) \leq \frac{SC z(x)^2}{N}
\end{align}
for the overall variance.

Now, for each $j \in \CC$ let 
$$
\tilde T_j = T_j - \E [T_j] = T_j - z(x),
$$
and let $\tilde T = \sum_{j \in \CC} \tilde T_j$. 
Note that by the definition of $T_j$ and the ES inequality~\eqref{eq:es-inequality-sipprand} we have that
$$
T_j = \frac{S w_j a_j(x)}{N s_j} \leq \frac{SC z(x)}{N}
$$
and consequently for the centered random variable $\tilde T_j$ that 
\begin{equation}
\label{eq:tildet-inequality-sipprand}
\abs{\tilde T_j} = \abs{T_j - \frac{z(x)}{N}} \leq \frac{SC z(x)}{N} =: M, 
\end{equation}
which holds with probability at least $1 - \nicefrac{\delta}{2}$ for any $x \sim \DD$. Also note that $\Var(\tilde T) = \Var (T)$.

Now conditioned on the ES inequality~\eqref{eq:es-inequality-sipprand} holding, applying Bernstein's inequality to both $\tilde T$ and $-\tilde T$ we have by symmetry and the union bound,
\begin{align*}
\Pr \left(\abs{\tilde T} \geq \eps_{rand} z(x)\right) 
&= \Pr \left(\abs{T - z(x)} \geq \eps_{rand} z(x) \right) \\
&\leq 2 \exp \left(- \frac{\eps_{rand}^2 z(x)^2}{2 \Var (T) + \frac{2 \eps_{rand} z(x) M}{3}} \right)
& \text{by Bernstein's inequality} \\
& \leq 2 \exp \left(- \frac{\eps_{rand}^2 z(x)^2}{\frac{2 S C z(x)^2}{N} + \frac{2 S C \eps_{rand} z(x)^2 }{3N}} \right)
& \text{by~\eqref{eqn:varplu-bound} and~\eqref{eq:tildet-inequality-sipprand}} \\
&= 2 \exp \left( - \frac{3 \eps_{rand}^2 N}{S C (6 + 2 \eps_{rand})} \right) \\
&\leq \frac{\delta}{2} 
& \text{by our choice of $\eps_{rand}$}
\end{align*}
Note that the (undesired) event $\abs{T - z(x)} \geq \eps_{rand} z(x) = \abs{\hat z_{rand}(x) - z(x)} \geq \eps_{rand} z(x)$ occurs with probability at most $\nicefrac{\delta}{2}$, which was conditioned on the ES inequality holding, which occurs with probability at least $1 - \nicefrac{\delta}{2}$. Thus by the union bound, 
the overall failure probability is at most $\delta$, which concludes the proof.
\end{proof}

\subsubsection{Error Guarantee for \sipphybrid}
\label{sec:supp-analysis-sipphybrid}
We note that the error guarantee for \sipphybrid follow straightforward from the error guarantees for \sippdet and \sipprand as stated in Lemma~\ref{lem:sippdet} and~\ref{lem:sipprand}, respectively, since \sipphybrid chooses the strategy among those two for which the associated error guarantee is lower. We can therefore state the error guarantee as follows.
\begin{lemma}[\sipphybrid error bound]
\label{lem:sipphybrid}
In the context of Lemmas~\ref{lem:sippdet} and~\ref{lem:sipprand}, for $\delta \in (0,1)$ \sipphybrid generates a pruned parameter group $\hat \WW_i^\ell$ such that for a fixed input patch $a(\cdot) \in \AA_i^\ell$, output patch $z(x)$, and $x \sim D$
$$
\Pr \left( \abs{\hat z_{hybrid}(x) - z(x)} \geq \epsilon_{hybrid} z(x) \right)  \leq \delta \quad \text{with} \quad
\epsilon_{hybrid} = \min \left\{ \eps_{det}, \eps_{rand} \right\} \in (0,1),
$$
where $z_{hybrid}(x)$ is the associated approximate output patch.
\end{lemma}

\subsection{Generalization to all weights and activations}
\label{sec:supp-analysis-generalization}
In this section, we generalize our analysis from the previous section to include all weights and activations. We also adapt the resulting error guarantees to simultaneously hold for all patches of all parameter groups within a layer instead of a fixed patch. 

We handle the general case by splitting both the input activations and the weights into their respective positive and negative parts representing the four quadrants, i.e.,
\begin{align*}
    z^{++}(x) &= \dotp{\WW_i^{\ell,+}}{a(x)^+} &
    z^{+-}(x) &= \dotp{\WW_i^{\ell,+}}{a(x)^-} \\
    z^{-+}(x) &= \dotp{\WW_i^{\ell,-}}{a(x)^+} &
    z^{--}(x) &= \dotp{\WW_i^{\ell,+}}{a(x)^-},
\end{align*}
where 
\begin{align*}
\WW_i^\ell &= \WW_i^{\ell,+} - \WW_i^{\ell,-}, \quad \WW_i^{\ell,+},\, \WW_i^{\ell,-} \geq 0, \\
\vspace{3px}
a(x) &= a(x)^+ - a(x)^-, \quad a(x)^+,\, a(x)^- \geq 0.
\end{align*}

First, consider negative activations 
for a non-negative parameter group. Specifically, when computing sensitivities over some set $\SS$ we split the input activations into their respective positive and negative part, and take an additional maximum over both parts. Henceforth the ES inequality~\ref{lem:es-inequality} can be applied to the positive and negative part of $a(x)$ at the same time. 
Similarly, we can split the parameter group into its positive and negative part when computing sensitivity such that the ES inequality~\ref{lem:es-inequality} holds for both parts of the parameter group as well.

More formally, the generalized relative parameter importance $g_j(x)$ for some parameter $w_j$ of parameter group $\WW_i^\ell$ can be defined as follows.

\begin{definition}[Generalized relative parameter importance]
\label{def:generalized-param-importance}
For a scalar parameter $w_j = w_j^+ - w_j^-$, $w_j^+,\, w_j^- \geq 0$, $j \in \II_i^\ell$, of parameter group $W_i^\ell$ in layer $\ell$, its generalized relative importance $g_j(x)$ is given by the maximum over its quadrant-wise relative importances, i.e.,
$$
g_j(x) = \max \{ g_j^{++}(x), g_j^{+-}(x), g_j^{-+}(x), g_j^{--}(x) \},
$$
where
$$
g_j^{++}(x) = \max_{a(\cdot) \in \AA_i^\ell} \frac{w_j^+ a_j^+(x)}{\sum_{k \in \II_i^\ell} w_k^+ a_k^+(x)} \text{, and so forth,}
$$
and where  $\AA_i^{\ell}$ denotes the set of patches for parameter group $W_i^\ell$ and $\AA_i^{\ell} \ni a(\cdot) = a^+(\cdot) - a^-(\cdot)$, $a^+(\cdot),\, a^-(\cdot) \geq 0$.
\end{definition}

The definition of generalized ES does not change compared to Definition~\ref{def:empirical-sens} and henceforth we do not re-state it explicitly. We proceed by re-deriving the ES inequality for the generalized parameter importance and any patch of the parameter group.

\begin{lemma}[Generalized ES inequality]
\label{lem:general-es-inequality}
For $\delta \in (0, 1)$, the ES $s_j$ of the scalar parameter $w_j$, $j \in \II_i^\ell$, of parameter group $W_i^\ell$ computed with a set $\SS$ of i.i.d.\ data points, $\abs{\SS}= \SizeOfS[]$, satisfies for each quadrant
$$
\Pr_{x} \left( w_j^+ a_j^+(x) \leq C s_j z^{++}(x) \right) \geq 1 - \delta \quad \forall j \in \II_i^\ell, \text{ and so forth,}
$$
for some input $x \sim \DD$ and some fixed input patch $a(\cdot) \in \AA_i^\ell$, where $C, K$ are the universal constants of Assumption~\ref{asm:cdf} and $z^{++}(x) = \sum_{k \in \II_i^\ell} w_k^+ a_k^+(x)$, and so forth, denotes the quadrant-wise output patch.
\end{lemma}
\begin{proof}
The proof follows the steps of Lemma~\ref{lem:es-inequality} with the exception of Equation~\eqref{eq:rel-importance-inequality}. To adapt it to the general case note that 
$$
\frac{w_j^+ a_j^+(x)}{z^{++}(x)} \leq g_j^{++}(x) \leq g_j(x), 
$$
and so forth, for each quadrant.
\end{proof}

Consequently, we can re-derive Lemmas~\ref{lem:sippdet}-\ref{lem:sipphybrid} such that they hold for each quadrant of a fixed patch. The derivations are analogues to the derivations in Section~\ref{sec:supp-analysis-error}. Finally, we adapt our guarantees to hold quadrant-wise for all patches of all parameter groups and layers simultaneously. 
We note that we can achieve this by appropriately adjusting the failure probability for Lemmas~\ref{lem:sippdet}-\ref{lem:sipphybrid} such that, by the union bound, the overall failure probability is bounded $\delta$. Specifically, we can invoke Lemmas~\ref{lem:sippdet}-\ref{lem:sipphybrid} with $\delta' = \nicefrac{\delta}{4\eta}$ such that 
\begin{align*}
\tilde S = \STildeDef{16 \eta} \quad \quad \text{and} \quad \quad \abs{\SS} = \SizeOfS[8 \eta],
\end{align*}
where $\eta$ denotes the number of total patches across all layers and parameter groups. The rest of the Lemmas remains unchanged. Therefore, we have that for all quadrant-wise patches our error guarantees hold. We utilize our patch-wise bounds as outlined in Section~\ref{sec:supp-method} to optimally allocate our budget across layers to minimize the relative error within each quadrant of the parameter groups and prune each parameter group according to the budget and the desired variant of \sipp.

\subsection{Network compression bounds}
\label{sec:supp-analysis-composition}

Up to this point we have established patch-wise and quadrant-wise error guarantees for the network, which suffices to prune the network according to Algorithm~\ref{alg:supp-sipp}. However, we can also leverage our theoretical guarantees to establish network-wide compression bounds of the form
$$
\Pr \nolimits_{\paramHat, \Point} \left(\norm{\fHat (\Input) - \f(\Input)} \leq \eps \norm{\f(\Input)}\right) \ge 1 - \delta,
$$
for given $\eps, \delta \in (0,1)$ as described in Problem~\ref{prob:network-compression}.

We will restrict ourselves to analyzing the general case for \sippdet but we note that each step can be applied analogously for \sipprand and \sipphybrid. 
We begin by generalizing Lemma~\ref{lem:sippdet} to establish norm-based bounds for each quadrant of the pre-activation. To this end, let 
$$
Z^{\ell++}(x) = \WW^{\ell+} * A^{\ell-1,+}(x) \text{,} \quad \hat Z^{\ell++}(x) = \hat \WW^{\ell+} * A^{\ell-1,+}(x), \quad \text{and so forth}
$$ 
denote the unpruned and approximate pre-activation quadrants, respectively. Moreover, let $S_i^\ell$ denote the sum of ES for parameter group $W_i^\ell$ as before and let $S_i^\ell(N_i^\ell)$ denote the sum over the $N_i^\ell$ largest ES for parameter group $W_i^\ell$.

\begin{corollary}
\label{cor:sippdet-norm}
For $\delta \in (0, 1)$, pruning layer $\ell$ according to \sippdet generates a pruned weight tensor $\hat\WW^\ell$ such that for a fixed quadrant and $x \sim \DD$
\begin{align*}
\Pr \left( \norm{\hat Z^{\ell++}(x) - \hat Z^{\ell++}(x)} \geq \epsilon^\ell \norm{ Z^{\ell++}(x)} \right)  \leq \delta \quad \text{with} \quad
\epsilon^\ell = \max_{i \in c^\ell} C \left(S_i^\ell - S_i^\ell(N_i^\ell) \right),
\end{align*}
where $N_i^\ell$ denotes the number of samples allocated to parameter group $W_i^\ell$.
The ESs are hereby computed over a set $\SS$ of $\SizeOfS[\eta^\ell]$ i.i.d.\ data points drawn from $\DD$.
\end{corollary}

\begin{proof}
Let $Z_i^{\ell++}(x)$ denote the pre-activation quadrant associated with parameter group $W_i^\ell$.
Invoking Lemma~\ref{lem:sippdet} with a set $\SS$ of $\SizeOfS[\eta^\ell]$ i.i.d.\ data points drawn from $\DD$ and $N_i^\ell$ samples for the respective parameter group implies that any associated patch, i.e. entry, of $Z_i^{\ell++}(x)$ is approximated with relative error at most $\eps_i^\ell = C \left(S_i^\ell - S_i^\ell(N_i^\ell) \right)$ with probability at least $ 1- \nicefrac{\delta}{\eta^\ell}$. Consequently, $\norm{Z^{\ell++}_i(x)}$ is also preserved with relative error $\eps_i^\ell$. Thus we have w.h.p.\ that
\begin{align*}
\norm{Z^{\ell++}(x) - \hat Z^{\ell++}(x)}^2 
&= \sum_{i \in [c^\ell]} \norm{Z^{\ell++}_i(x) - \hat Z^{\ell++}_i(x)}^2 \\
&\leq \sum_{i \in [c^\ell]} (\eps_i^\ell)^2 \norm{Z^{\ell++}_i(x)}^2 \\
&\leq (\eps^\ell)^2 \sum_{i \in [c^\ell]}  \norm{Z^{\ell++}_i(x)}^2 & \text{by definition of $\eps^\ell$} \\
& = (\eps^\ell)^2 \norm{Z^{\ell++}(x)}^2
\end{align*}
Taking a union bound over all $\eta^\ell$ patches in the pre-activation $Z^\ell(x)$ concludes the proof.
\end{proof}

We note that Corollary~\ref{cor:sippdet-norm} is stated for $Z^{\ell++}(x)$ but naturally extends to the other quadrants as well.

As a next step, we establish guarantees to approximate $Z^{\ell}(x)$ by leveraging the guarantees for each quadrant. To this end, note that 
$
Z^{\ell}(x) = Z^{\ell++}(x) - Z^{\ell+-}(x) - Z^{\ell-+}(x) + Z^{\ell--}(x).
$
Further, let $\Delta^\ell$ denote the ``sign complexity`` of approximating the overall pre-activation, which is defined as
\begin{definition}[Sign complexity]
For layer $\ell$, its sign complexity $\Delta^\ell$ is given by
$$
\Delta^\ell = \max_{x \in \SS} \frac{\norm{Z^{\ell++}(x)} + \norm{Z^{\ell+-}(x)} + \norm{Z^{\ell-+}(x)} + \norm{Z^{\ell--}(x)}}{\norm{Z^{\ell}(x)}}
$$
where $\SS$ denotes a set of i.i.d.\ data points drawn from $\DD$.
\end{definition}
Intuitively, $\Delta^\ell$ captures the additional complexity of approximating the layer when considering the actual signs of the quadrants as opposed to treating them separately.
We can now state the error guarantees for \sipp in context of Corollary~\ref{cor:sippdet-norm} for the overall pre-activation.

\begin{lemma}[Layer error bound]
\label{lem:sippdet-layer}
For given $\delta \in (0,1)$ and sample budget $N_i^\ell$ for each parameter group, invoking \sippdet to prune  $\WW^\ell$ generates a pruned weight tensor $\hat\WW^\ell$ such that for $x \sim \DD$
\begin{align*}
\Pr \left( \norm{\hat Z^{\ell}(x) - \hat Z^{\ell}(x)} \geq \epsilon^\ell \Delta^\ell \norm{Z^{\ell}(x)} \right)  \leq \delta \quad \text{with} \quad
\epsilon^\ell = \max_{i \in c^\ell} C \left(S_i^\ell - S_i^\ell(N_i^\ell) \right),
\end{align*}
where $N_i^\ell$ denotes the number of samples allocated to parameter group $W_i^\ell$.
The ESs are hereby computed over a set $\SS$ of $\SizeOfS[5\eta^\ell]$ i.i.d.\ data points drawn from $\DD$.
\end{lemma}

\begin{proof}
Consider invoking Corollary~\ref{cor:sippdet-norm} with a set $\SS$ of $\SizeOfS[5\eta^\ell]$ i.i.d.\ data points drawn from $\DD$. Then for each quadrant we have w.h.p.\ that 
$$
\norm{\hat Z^{\ell++}(x) - \hat Z^{\ell++}(x)} \leq \epsilon^\ell \norm{\hat Z^{\ell++}(x)} \text{, and so forth,}
$$
for an appropriate notion of high probability specified subsequently. Note that 
$$
Z^{\ell}(x) = Z^{\ell++}(x) - Z^{\ell+-}(x) - Z^{\ell-+}(x) + Z^{\ell--}(x)
$$
and so w.h.p.\ we have that  
\begin{align*}
\norm{\hat Z^{\ell}(x) - \hat Z^{\ell}(x)} 
&= \Big\| \hat Z^{\ell++}(x) - \hat Z^{\ell++}(x) - \hat Z^{\ell+-}(x) + \hat Z^{\ell+-}(x) \\
& \qquad - \hat Z^{\ell-+}(x) + \hat Z^{\ell-+}(x) + \hat Z^{\ell--}(x) - \hat Z^{\ell--}(x) \Big\| \\
&\leq \norm{\hat Z^{\ell++}(x) - \hat Z^{\ell++}(x)} + \norm{\hat Z^{\ell+-}(x) - \hat Z^{\ell+-}(x)} \\
& \qquad + \norm{\hat Z^{\ell-+}(x) - \hat Z^{\ell-+}(x)} + \norm{\hat Z^{\ell--}(x) - \hat Z^{\ell--}(x)} \\
& \leq \eps^\ell \left( \norm{\hat Z^{\ell++}(x)} + \norm{\hat Z^{\ell+-}(x)} + \norm{\hat Z^{\ell-+}(x)} + \norm{\hat Z^{\ell--}(x)}  \right) \\
&= \eps^\ell \frac{\norm{Z^{\ell++}(x)} + \norm{Z^{\ell+-}(x)} + \norm{Z^{\ell-+}(x)} + \norm{Z^{\ell--}(x)}}{\norm{Z^{\ell}(x)}} \norm{Z^{\ell}(x)} \\
& \leq \eps^\ell \Delta^\ell \norm{Z^{\ell}(x)},
\end{align*}
where the last step followed from our definition of $\Delta^\ell$. By imposing a regularity assumption on $\Delta^\ell$ similar to that of ES, we can show that $\Delta^\ell$ is an upper bound for any $x \sim \DD$ w.h.p.\ following the proof of the ES inequality (Lemma~\ref{lem:es-inequality}).

To specify the appropriate notion of high probability, we consider the individual failure cases and apply the union bound. In particular, for our choice of for the size of $\SS$, we have that for a particular quadrant the approximation fails with probability at most $\nicefrac{\delta}{5}$. Thus across all quadrants we have a overall failure probability of at most $\nicefrac{4\delta}{5}$. Finally, we consider the event that $\Delta^\ell$ does not upper bound the hardness for some input $x \sim \DD$, which occurs with probability at most $\nicefrac{\delta}{5}$ by our choice for the size of $\SS$. Henceforth, our overall failure probability is at most $\delta$, again by the union bound, which concludes the proof.
\end{proof}

We now consider the effect of pruning multiple layers at the same time and analyze the final resulting error in the output. To this end, consider the activation $\phi^\ell(\cdot)$ for which we assume the following. 

\begin{assumption}
\label{asm:activations}
For layer $\ell \in [L]$, the activation function, denoted by $\phi^\ell(\cdot)$, is Lipschitz continuous with Lipschitz constant $K^\ell$.
\end{assumption}

Without loss of generality, we will further assume that the activation function is $1$-Lipschitz, which is the case, e.g., for ReLU and Softmax, to avoid introducing additional notation. 
We now state a lemma pertaining to the error resulting from pruning multiple layers simultaneously, which will provide the basis for establishing error bounds across the entire network.

\begin{lemma}[Error propagation]
\label{lem:error-propagation}
Let $\hat A^{\ell}(x)$, $\ell \leq L$, denote the activation of layer $\ell$ when we have pruned layers $1, \ldots, \ell$ according to Lemma~\ref{lem:sippdet-layer}. Then the overall approximation in layer $\ell$ is bounded by
$$
\norm{\hat A^\ell(x) - A^\ell(x)} \leq \sum_{k = 1}^\ell \left( \prod_{k' = k+1}^\ell \norm{W^{k'}}_F \right) \eps^k \Delta^k \norm{Z^k(x)}
$$
with probability at least $1 - \delta$. The ESs are hereby computed over a set $\SS$ of $\SizeOfS[5\sum_{k \in [\ell]} \eta^k]$ i.i.d.\ data points drawn from $\DD$.

\end{lemma}

\begin{proof}
We prove the above statement by induction. For layer $\ell=1$, we have that 
\begin{align*}
\norm{\hat A^1(x) - A^1(x)} 
&= \norm{\phi^1(\hat W^1 * \hat A^0(x)) - \phi^1(\hat W^1 * A^0(x))}
\\
&\leq \norm{\hat W^1 * \hat A^0(x) - \hat W^1 * A^0(x)}
& \text{since the $\phi^1(\cdot)$ is $1$-Lipschitz} \\
&= \norm{\hat W^1 * A^0(x) - \hat W^1 * A^0(x)}
& \text{since $\hat A^0(x) = A^0(x) = x$} \\
& = \norm{\hat Z^1(x) - Z^1(x)} 
& \text{by definition of $\hat Z^1(x)$ and $Z^1(x)$} \\
&\leq \eps^1 \Delta^1 \norm{Z^1(x)} 
& \text{by Lemma~\ref{lem:sippdet-layer}},
\end{align*}
which proves that the base case holds.

We now proceed with the inductive step. Assuming the inequality is true for layer $\ell$, we have for layer $\ell+1$ that
\begin{align*}
\norm{\hat A^{\ell+1}(x) - A^{\ell+1}(x)} 
&= 
\norm{\phi^{\ell+1}(\hat W^{\ell+1} * \hat A^\ell (x)) - \phi^{\ell+1}(W^{\ell+1} * A^\ell(x))}
\\
&\leq \norm{\hat W^{\ell+1} * \hat A^\ell (x) - W^{\ell+1} * A^\ell(x)}
\quad \text{since $\phi^{\ell+1}(\cdot)$ is $1$-Lipschitz} \\
&= 
\norm{
\hat W^{\ell+1} * \hat A^\ell (x) 
- \hat W^{\ell+1} * A^\ell (x) 
+ \hat W^{\ell+1} * A^\ell (x) 
- W^{\ell+1} * A^\ell(x)
} \\
&\leq 
\norm{\hat W^{\ell+1} * (\hat A^\ell (x) - A^\ell (x))}
+ \norm{(\hat W^{\ell+1} - W^{\ell+1}) * A^\ell (x)}
\end{align*}
Note that we can bound the first term by 
\begin{align*}
\norm{\hat W^{\ell+1} * (\hat A^\ell (x) - A^\ell (x))}
&\leq \norm{\hat W^{\ell+1}}_{op}  \norm{\hat A^\ell (x) - A^\ell (x)} \\
& \leq \norm{\hat W^{\ell+1}}_{F}  \norm{\hat A^\ell (x) - A^\ell (x)} \\
& \leq \norm{W^{\ell+1}}_{F}  \norm{\hat A^\ell (x) - A^\ell (x)}
&\text{since $\hat W^{\ell+1}$ is a subset of $ W^{\ell+1}$}
\end{align*}
where $\norm{\cdot}_{op}$ and $\norm{\cdot}_{F}$ denote the $\ell_2$-induced operator norm and Frobenius norm, respectively. 
The second term is bounded by Lemma~\ref{lem:sippdet-layer}, i.e.,
\begin{align*}
\norm{(\hat W^{\ell+1} - W^{\ell+1}) * A^\ell (x)}
&= \norm{\hat Z^{\ell+1}(x) - Z^{\ell+1}(x)}
\leq \eps^{\ell+1} \Delta^{\ell+1} \norm{ Z^{\ell+1}(x)}.
\end{align*}

Putting both terms back together we have that
\begin{align*}
&\norm{\hat A^{\ell+1}(x) - A^{\ell+1}(x)} \\
&\qquad \leq \norm{W^{\ell+1}}_{F}  \norm{\hat A^\ell (x) - A^\ell (x)}
+ \eps^{\ell+1} \Delta^{\ell+1} \norm{ Z^{\ell+1}(x)} \\
&\qquad \leq \norm{W^{\ell+1}}_{F} \Bigg( \sum_{k = 1}^\ell \left( \prod_{k' = k+1}^\ell \norm{W^{k'}}_F \right) \eps^k \Delta^k \norm{Z^k(x)} \Bigg) 
+ \eps^{\ell+1} \Delta^{\ell+1} \norm{ Z^{\ell+1}(x)} \\
&\qquad = \sum_{k = 1}^\ell \left( \prod_{k' = k+1}^{\ell+1} \norm{W^{k'}}_F \right) \eps^k \Delta^k \norm{Z^k(x)}
+ \eps^{\ell+1} \Delta^{\ell+1} \norm{ Z^{\ell+1}(x)} \\
&\qquad = \sum_{k = 1}^{\ell+1} \left( \prod_{k' = k+1}^{\ell+1} \norm{W^{k'}}_F \right) \eps^k \Delta^k \norm{Z^k(x)},
\end{align*}
where the second inequality followed from our induction hypothesis. Finally, we note that, by our choice for the size of $\SS$ and the union bound, the overall failure probability is bounded above by $\delta$.

\end{proof}

From the analysis the term $\prod_{k' = k+1}^\ell \norm{W^{k'}}_F$ arises, which is an upper bound for the Lipschitz constant of the network starting from layer $k+1$. Moreover, the coefficient of the propagated error is closely related to the condition number between layer $\ell$ and the network's output. To this end, consider the following upper bound on the condition number.

\begin{definition}[Layer condition number]
\label{def:layer-condition}
For layer $\ell$, the condition number from the pre-activation of layer $\ell$ to the output of the network (activation of layer $L$) is given by
$$
\kappa^\ell = \max_{x \in S} \left( \prod_{k=\ell+1}^L \norm{W^\ell}_F \right) \frac{\norm{Z^\ell(x)}}{\norm{A^L(x)}},
$$
where $\SS$ denotes a set of i.i.d.\ data points drawn from $\DD$.
\end{definition}
To see that $\kappa^\ell$ is indeed an upper bound on the condition number we note that the condition number is defined as the maximum relative change in the output over the maximum relative change in the input, i.e., 
$$
\max_{x, x' \in \SS} 
\frac{
\frac{\norm{A^L(x) - A^L(x')}}{\norm{A^L(x)}}
}
{
\frac{\norm{Z^\ell(x) - Z^\ell(x')}}{\norm{Z^\ell(x)}}
} = \max_{x, x' \in \SS}  \frac{\norm{A^L(x) - A^L(x')}}{\norm{Z^\ell(x) - Z^\ell(x')}} \frac{\norm{Z^\ell(x)}}{\norm{A^L(x)}}.
$$
The first term can be upper bounded as
\begin{align*}
\frac{\norm{A^L(x) - A^L(x')}}{\norm{Z^\ell(x) - Z^\ell(x')}} 
& \leq \frac{\norm{Z^L(x) - Z^L(x')}}{\norm{Z^\ell(x) - Z^\ell(x')}} \\
& = \frac{\norm{W^L * (A^{L-1}(x) - A^{L-1}(x'))}}{\norm{Z^\ell(x) - Z^\ell(x')}} \\
& \leq \norm{W^L}_F \frac{\norm{(A^{L-1}(x) - A^{L-1}(x'))}}{\norm{Z^\ell(x) - Z^\ell(x')}} \\
&\leq \ldots \\
&\leq \prod_{k=\ell+1}^L \norm{W^k}_F \frac{\norm{A^\ell(x) - A^\ell(x')}}{\norm{Z^\ell(x) - Z^\ell(x')}} \\
&\leq \prod_{k=\ell+1}^L \norm{W^k}_F,
\end{align*}
which plugged back in above yields the definition of the layer condition number $\kappa^\ell$.

Equipped with Lemma~\ref{lem:error-propagation} and Definition~\ref{def:layer-condition} we are now ready to state our main compression bound over the entire network. 

\begin{theorem}[Network compression bound]
\label{thm:main-compression}
For given $\delta \in (0, 1)$, a set of parameters $\param = \paramDef$, and a sample budget $\BB$  \sipp (Algorithm~\ref{alg:supp-sipp}) generates a set of compressed parameters $\paramHat = \paramHatDef$ such that $\size{\hat \theta} \leq \BB$, $\size{W_i^\ell} \leq N_i^\ell,\, \forall i \in [c^\ell],\, \ell \in [L]$,
$$
\Pr_{\paramHat, \Point} \left(\norm{\fHat (\Input) - \f(\Input)} \leq \eps \norm{\f(\Input)}\right) \ge 1 - \delta
\quad \text{and} \quad
\eps = C \sum_{\ell = 1}^L \kappa^\ell \Delta^\ell \max_{i \in [c^\ell]} \left(S_i^\ell - S_i^\ell(N_i^\ell) \right),
$$
where $S_i^\ell$ is the sum of sensitivities for parameter group $W_i^\ell$ computed over a set $\SS$ of $\SizeOfS[6 \eta]$ i.i.d.\ data points.
\end{theorem}

\begin{proof}
Invoking Lemma~\ref{lem:error-propagation} for $\ell=L$ implies with high probability that
\begin{align*}
\norm{\hat A^L(x) - A^L(x)} 
&\leq \sum_{\ell = 1}^L \left( \prod_{k = \ell+1}^L \norm{W^{k}}_F \right) \eps^\ell \Delta^\ell \norm{Z^\ell(x)} \\
&= \sum_{\ell = 1}^L  \left( \prod_{k = \ell+1}^L \norm{W^{k}}_F \right) \frac{\norm{Z^\ell(x)}}{\norm{A^L(x)}} \Delta^\ell \eps^\ell \norm{A^L(x)} \\
&\leq \sum_{\ell = 1}^L \kappa^\ell \Delta^\ell \eps^\ell  \norm{A^L(x)} \\
&= \eps \norm{A^L(x)},
\end{align*}
where the last inequality followed from our definition of the layer condition number $\kappa^\ell$. Moreover, following the analysis of Lemma~\ref{lem:es-inequality} we can establish that $\kappa^\ell$ is an upper bound for any $x \sim \DD$ with high probability.
Finally, we note that the overall failure probability is bounded by $\delta$ by our choice for the size of $\SS$ and by a union bound over the failure probabilities of Lemma~\ref{lem:error-propagation} and of $\kappa^\ell$ not being an upper bound for some $x \sim \DD$.
\end{proof}
\section{Experimental details}
\label{sec:supp-results}

\subsection{Setup and Hyperparameters}

All hyperparameters for training, retraining, and pruning are outlined in Table~\ref{tab:cifar_hyperparameters}.
For training CIFAR10 networks we used the training hyperparameters outlined in the respective original papers, i.e., as described by \cite{he2016deep}, \cite{simonyan2014very}, \cite{huang2017densely}, and \cite{zagoruyko2016wide} for ResNets, VGGs, DenseNets, and WideResNets, respectively. For retraining, we did not change the hyperparameters and repurposed the training hyperparameters. We added a warmup period in the beginning where we linearly scale up the learning rate from 0 to the nominal learning rate.
Iterative pruning is conducted by repeatedly removing the same ratio of parameters (denoted by $\alpha$ in Table~\ref{tab:cifar_hyperparameters}). The prune parameter $\delta$ describes the failure probability of \sipp. We note no other additional hyperparameters are required to run \sipp.

For ImageNet, we show experimental results for a ResNet18 and a ResNet101. As in the case of the CIFAR10 networks, we re-purpose the same training hyperparameters as indicated in the original paper. We also use the same hyperparameters for retraining. The hyperparameters are summarized in Table~\ref{tab:imagenet_hyperparameters}.

\begin{table}[htb]
    %\small
    \centering
    \begin{tabular}{cr|cccc}
        & & VGG16 & Resnet20/56/110 & DenseNet22 & WRN-16-8 \\ \hline
        \multirow{9}{*}{Train} 
        & test error    &  7.19          & 8.6/7.19/6.43 
                        &  10.10         & 4.81 \\
        & loss          & cross-entropy  & cross-entropy
                        & cross-entropy  & cross-entropy \\
        & optimizer     & SGD            & SGD 
                        & SGD            & SGD  \\
        & epochs        & 300            & 182
                        & 300            & 200  \\
        & warm-up       & 5              & 5    
                        & 5              & 5    \\
        & batch size    & 256            & 128
                        & 64             & 128   \\
        & LR            & 0.05           & 0.1
                        & 0.1            & 0.1  \\
        & LR decay      & 0.5@\{30, \ldots\} & 0.1@\{91, 136\}
                        & 0.1@\{150, 225\}   & 0.2@\{60, \ldots\}\\
        & momentum      & 0.9            & 0.9
                        & 0.9            & 0.9  \\
        & Nesterov      & No             & No
                        & Yes            & Yes  \\
        & weight decay  & 5.0e-4         & 1.0e-4 
                        & 1.0e-4         & 5.0e-4   \\ \hline
        \multirow{2}{*}{Prune}
        & $\delta$      & 1.0e-16        & 1.0e-16
                        & 1.0e-16        & 1.0e-16  \\
        & $\alpha$      & 0.85           & 0.85
                        & 0.85           & 0.85
    \end{tabular}
    \vspace{2ex}
    \caption{We report the hyperparameters used during training, pruning, and retraining for various convolutional architectures on CIFAR-10. LR hereby denotes the learning rate and LR decay denotes the learning rate decay that we deploy after a certain number of epochs. During retraining we used the same hyperparameters. $\{30, \ldots\}$ denotes that the learning rate is decayed every 30 epochs.}
    \label{tab:cifar_hyperparameters}
\end{table}

\begin{table}[htb]
    \centering
    \begin{tabular}{cr|c}
        &  & ResNet18/101 \\ \hline
        \multirow{10}{*}{Train} 
        & top-1 test error    & 30.26/22.63 \\
        & top-5 test error    & 10.93/6.45 \\
        & loss          & cross-entropy \\
        & optimizer     & SGD \\
        & epochs        & 90 \\
        & warm-up       & 5 \\
        & batch size    & 256 \\
        & LR            & 0.1 \\
        & LR decay      & 0.1@\{30, 60, 80\} \\
        & momentum      & 0.9 \\
        & Nesterov      & No \\
        & weight decay  & 1.0e-4 \\  \hline
        \multirow{2}{*}{Prune}
        & $\delta$      & 1.0e-16 \\
        & $\alpha$      & 0.90
    \end{tabular}
    \caption{We report the hyperparameters used during training, pruning, and retraining for various convolutional architectures on ImageNet. LR hereby denotes the learning rate and LR decay denotes the learning rate decay that we deploy after a certain number of epochs.}
    \label{tab:imagenet_hyperparameters}
\end{table}

\subsection{Iterative prune+retrain results for CIFAR10}
In Figure~\ref{fig:cifar_prune_results_2} we show the results and comparisons when using iterative prune+retrain as outlined in Section~6. We highlight that \sipp performs en par with WT while \textsc{SNIP} performs significantly worse than WT. 

We also compare the performance of the three variations of our algorithm, see Figure~\ref{fig:cifar_prune_all}. Note that the performance for all of them is very similar, henceforth we choose \sippdet for its simplicity when comparing to other methods for this expensive iterative prune+retrain pipeline.

\subsection{Random-init+prune+train results for CIFAR10}
The results for this pipeline are shown in Figure~\ref{fig:cifar_prune_rand}. We note that \textsc{WT} performs significantly worse in this case whereas \textsc{SNIP} and \sipp clearly outperform \textsc{WT}. Overall, we find that sometimes \sipp can even outperform \textsc{SNIP}. More importantly, however, these experiments highlight the versatile nature of \sipp, i.e., it performs consistently well across multiple prune pipelines hence serving as a reliable and useful plug-and-play solution within a bigger pipeline. We conjecture that this is due to the provable nature of \sipp.

\subsection{Iterative prune+retrain results for ImageNet}
Finally, we show results for a ResNet18 and ResNet101 trained on ImageNet, see Figure~\ref{fig:imagenet_prune}. From the results, we can conclude that \sipp scales well to larger architectures and datasets, such as ImageNet, and can perform en par with existing state-of-the-art methods.

\begin{figure}[H]
  \centering
    \begin{minipage}[t]{0.45\textwidth}
    \includegraphics[width=\textwidth]{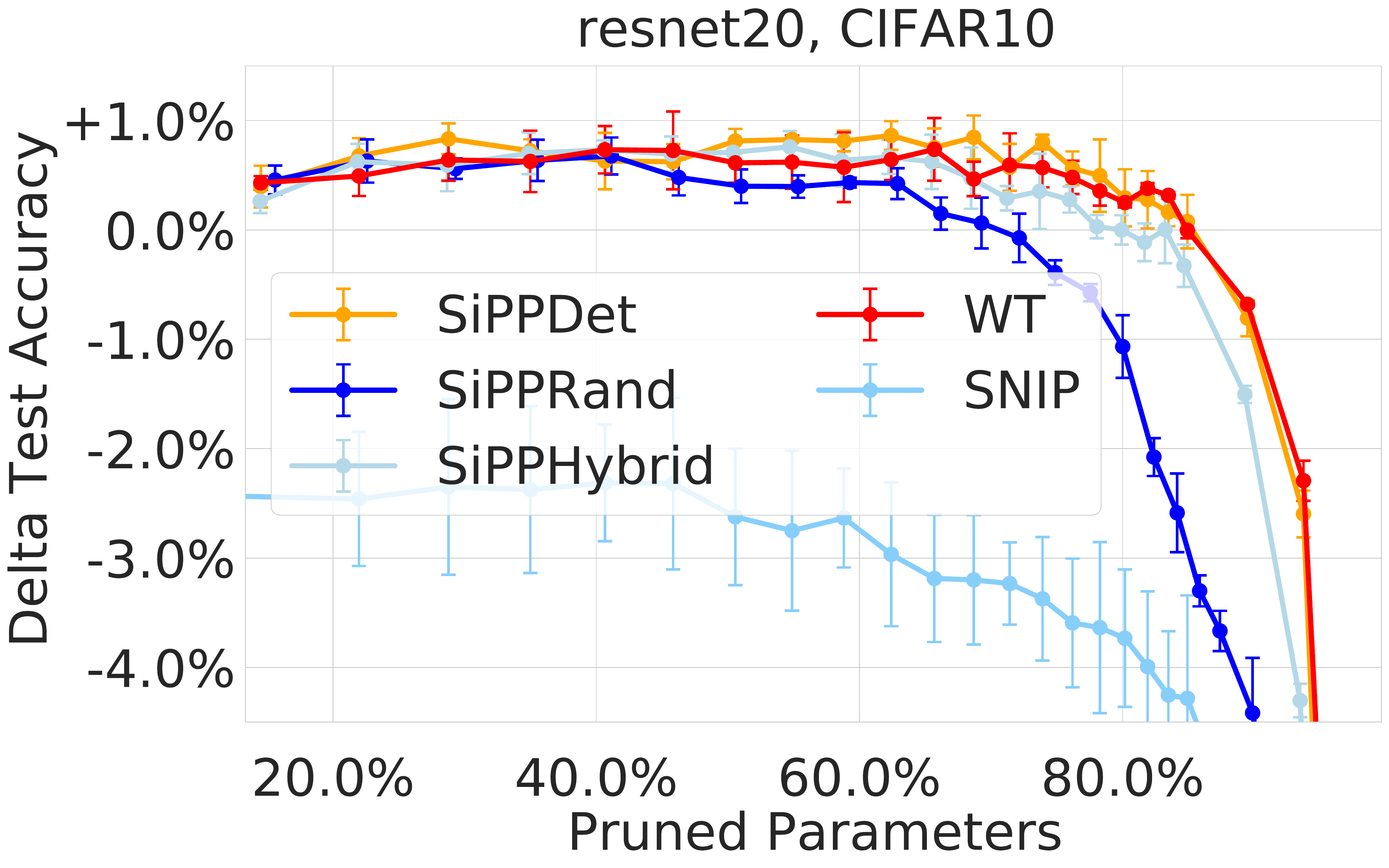}
    \subcaption{Resnet20}
  \end{minipage}%
  \hspace{1ex}
  \begin{minipage}[t]{0.45\textwidth}
    \includegraphics[width=\textwidth]{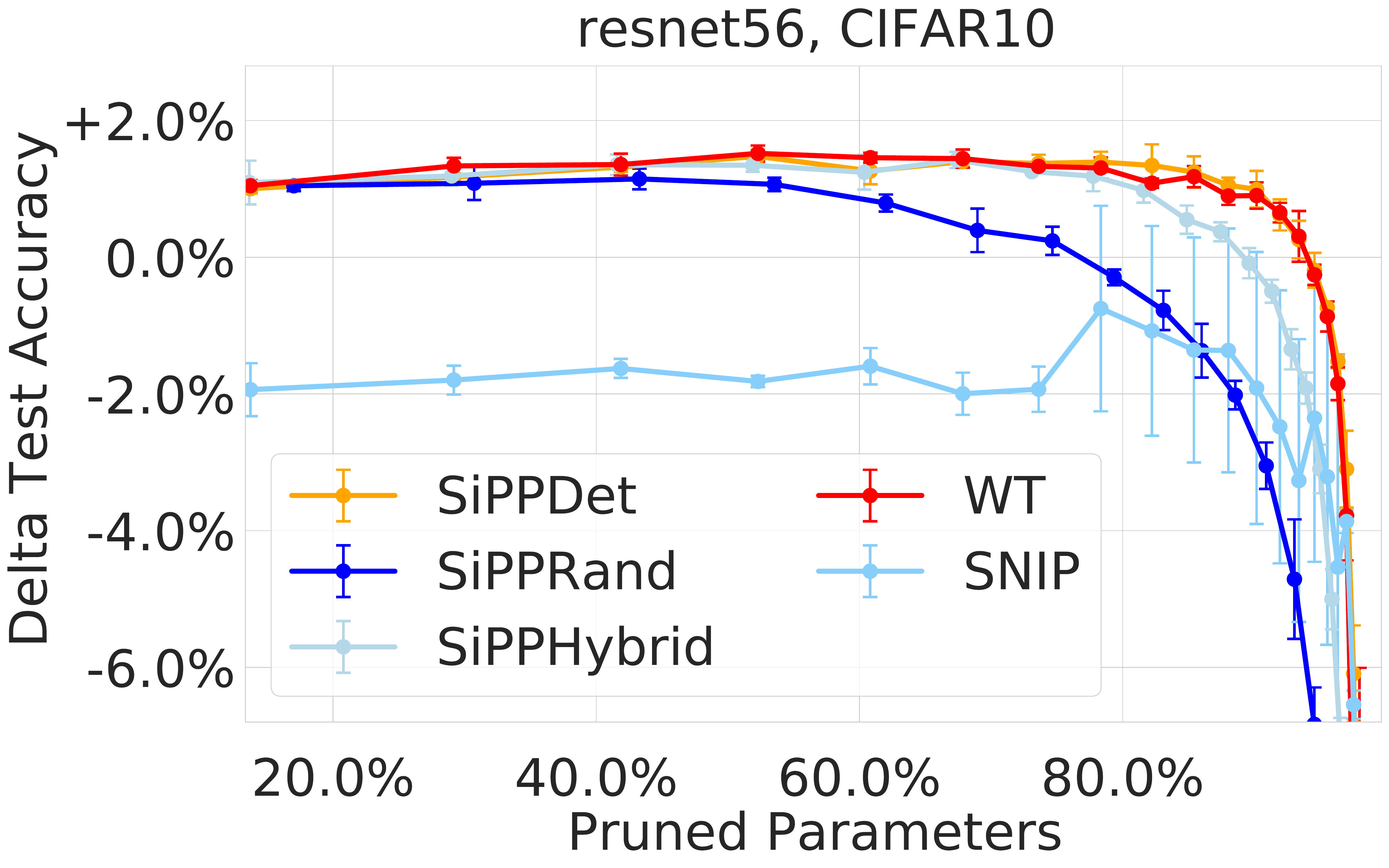}
    \subcaption{Resnet56}
  \end{minipage}%
  \caption{The delta in test accuracy to the uncompressed network for the generated pruned models trained on CIFAR10 for various target prune ratios. The networks were pruned using the \textbf{iterative prune+retrain} pipeline.}
  \label{fig:cifar_prune_all}
\end{figure}

\begin{figure}[htb]
\centering
\begin{minipage}[t]{0.48\textwidth}
    \includegraphics[width=\textwidth]{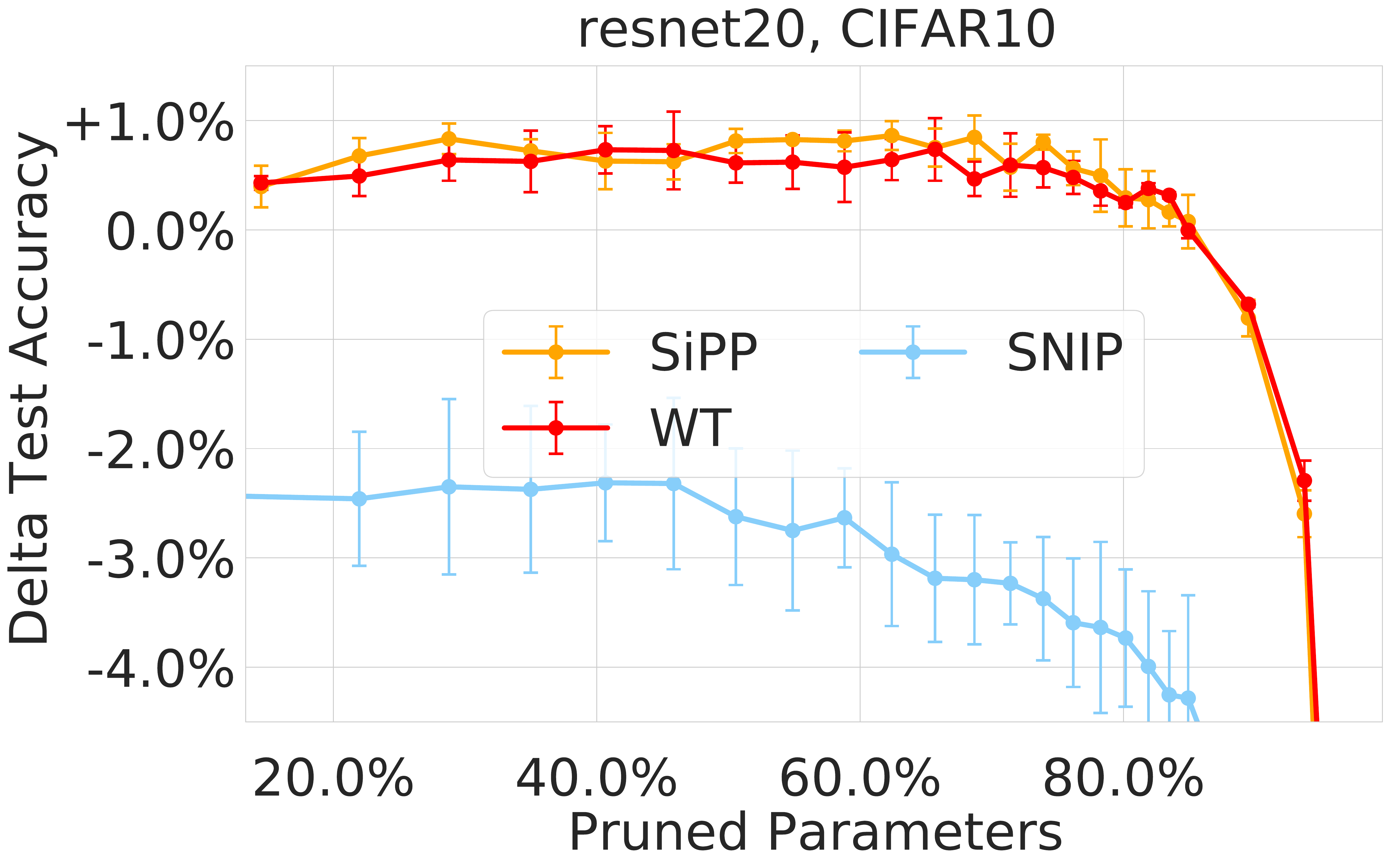}
    \subcaption{Resnet20}
\end{minipage}%
\begin{minipage}[t]{0.48\textwidth}
    \includegraphics[width=\textwidth]{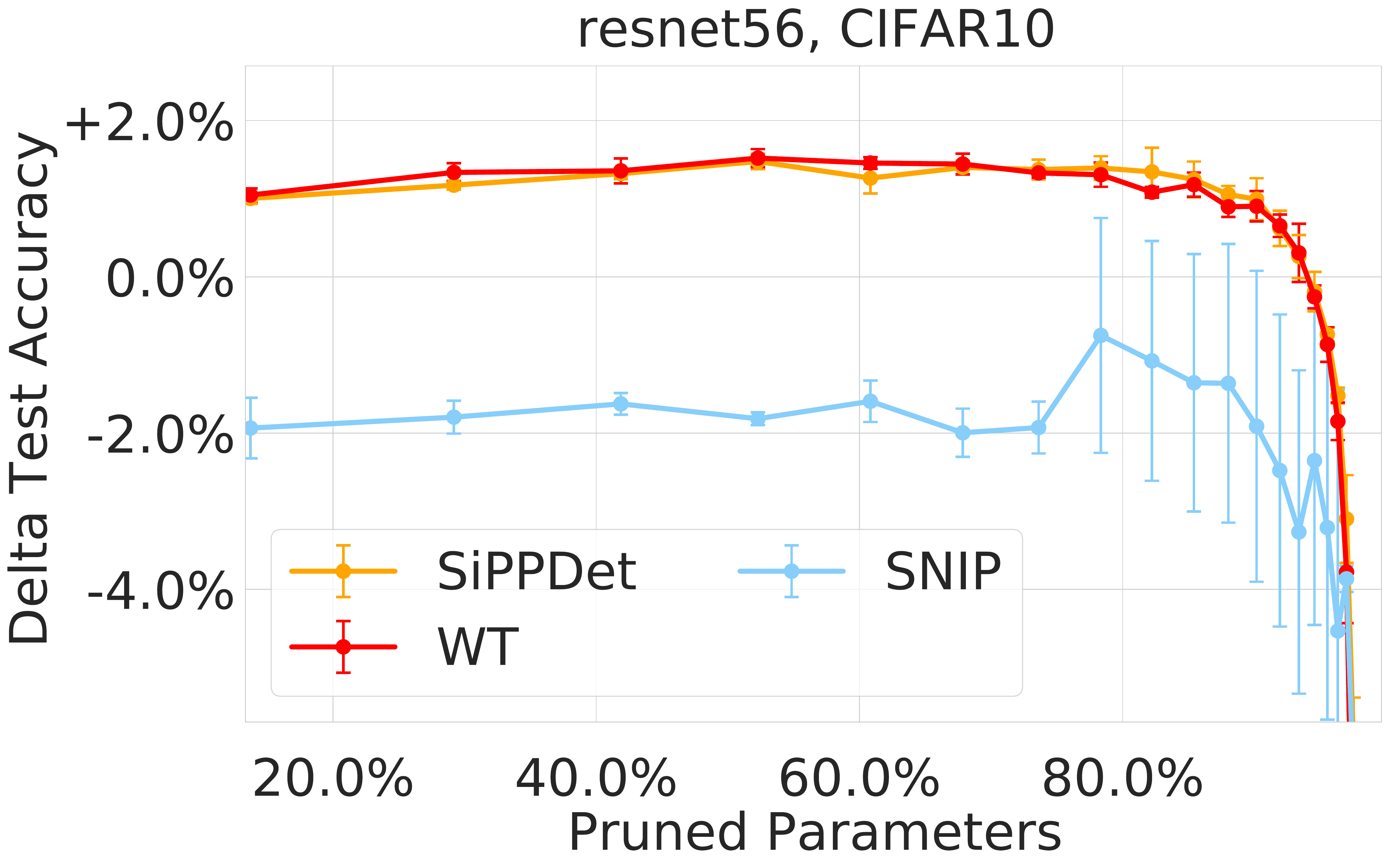}
    \subcaption{Resnet56}
\end{minipage}
\begin{minipage}[t]{0.48\textwidth}
    \includegraphics[width=\textwidth]{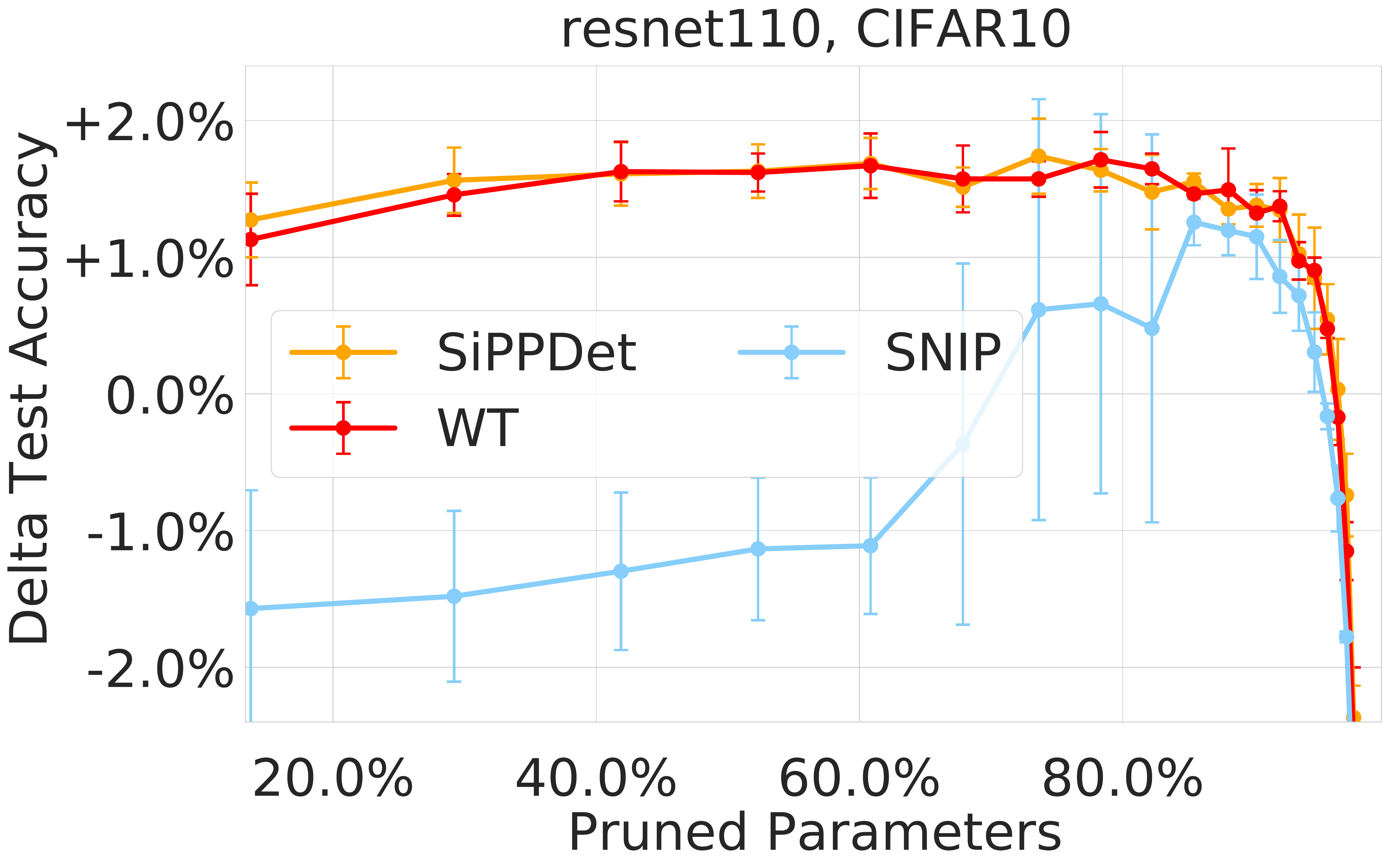}
    \subcaption{Resnet110}
\end{minipage}
\begin{minipage}[t]{0.48\textwidth}
    \includegraphics[width=\textwidth]{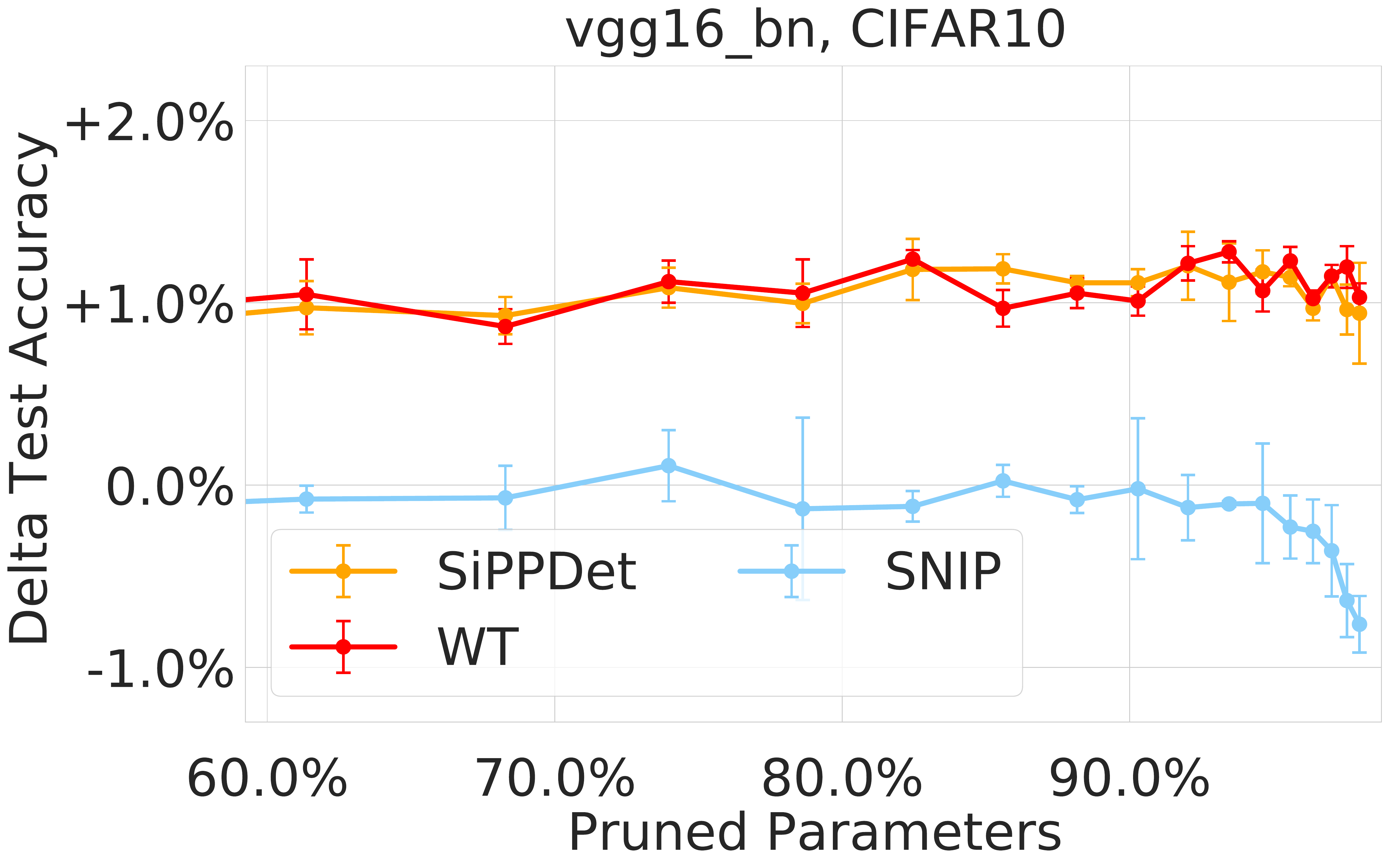}
    \subcaption{VGG16}
\end{minipage}
\begin{minipage}[t]{0.48\textwidth}
    \includegraphics[width=\textwidth]{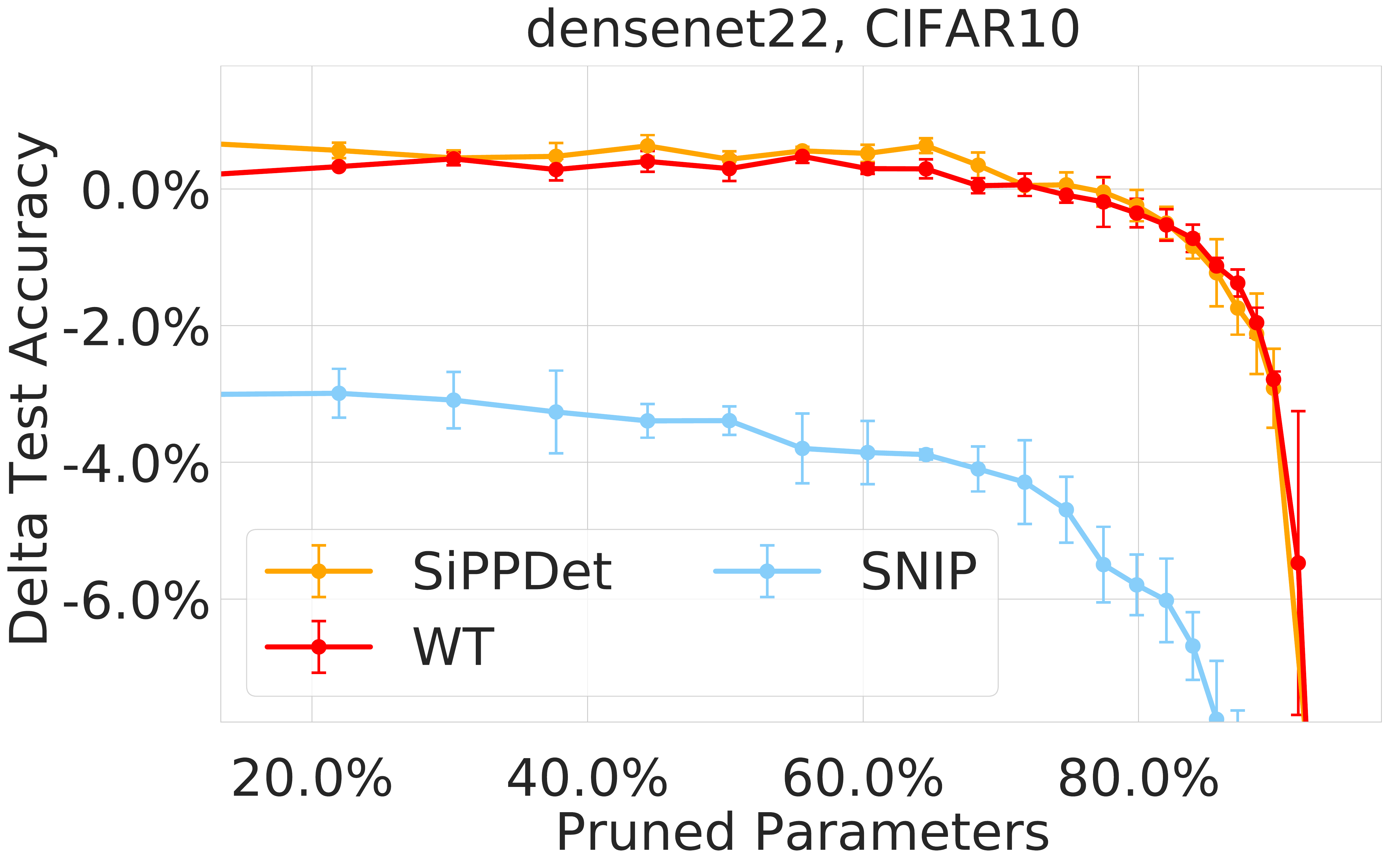}
    \subcaption{DenseNet22}
\end{minipage}%
\begin{minipage}[t]{0.48\textwidth}
    \includegraphics[width=\textwidth]{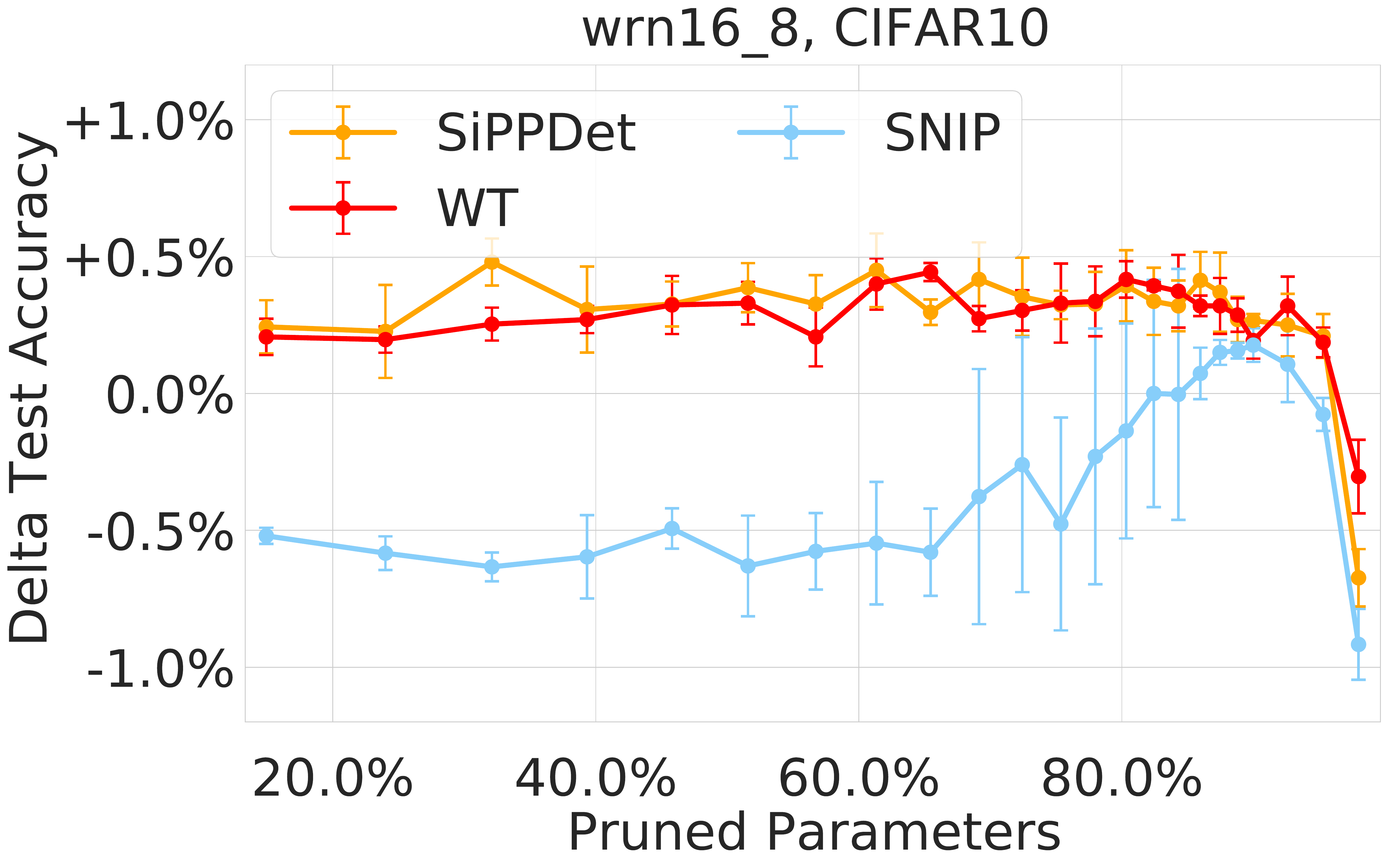}
    \subcaption{WRN16-8}
\end{minipage}%
 
\caption{The delta in test accuracy to the uncompressed network for the generated pruned models trained on CIFAR10 for various target prune ratios. The networks were pruned using the \textbf{iterative prune+retrain} pipeline.}
\label{fig:cifar_prune_results_2}
\end{figure}

\begin{figure}[htb]
\centering
\begin{minipage}[t]{0.48\textwidth}
    \includegraphics[width=\textwidth]{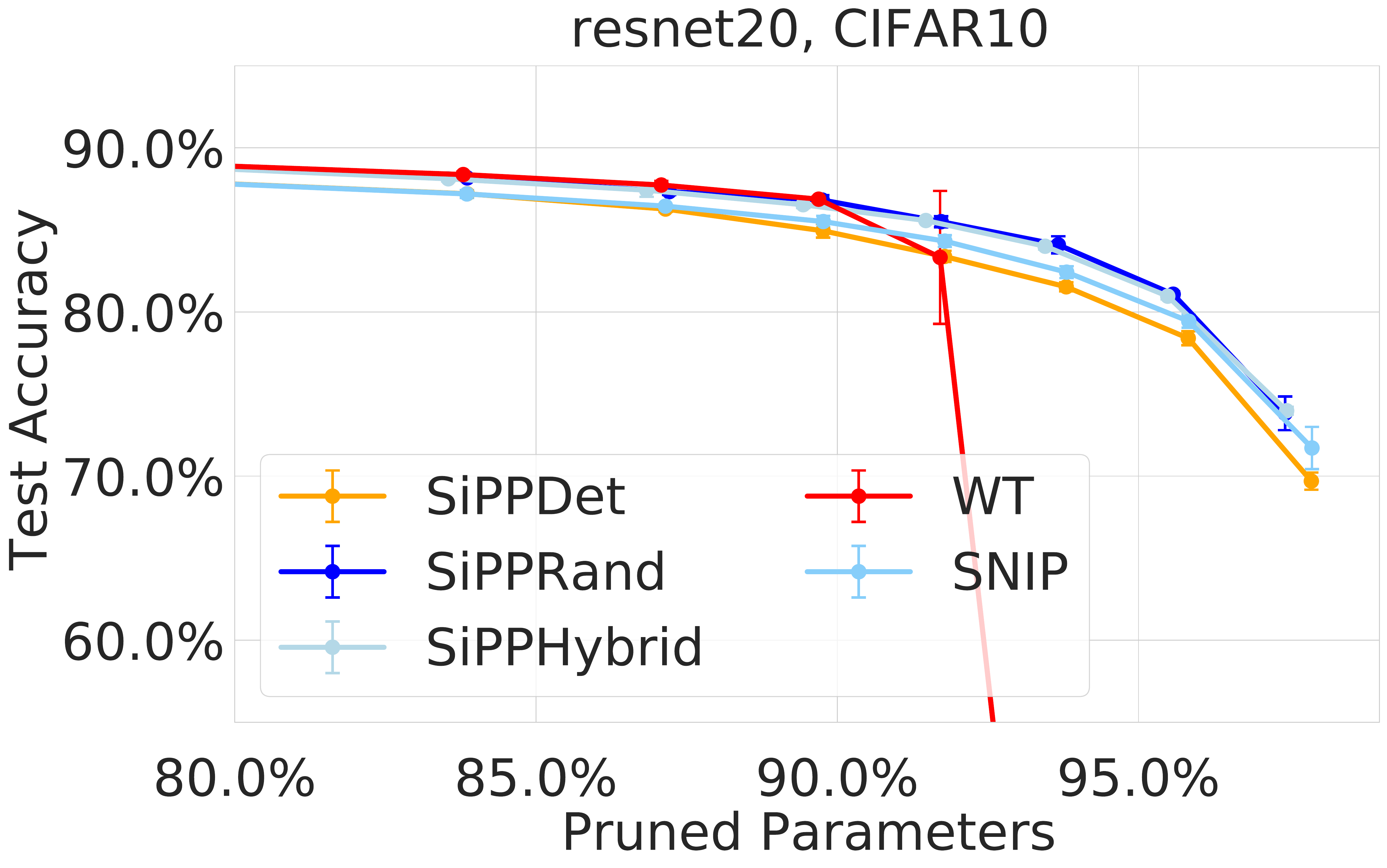}
    \subcaption{Resnet20}
\end{minipage}%
\begin{minipage}[t]{0.48\textwidth}
    \includegraphics[width=\textwidth]{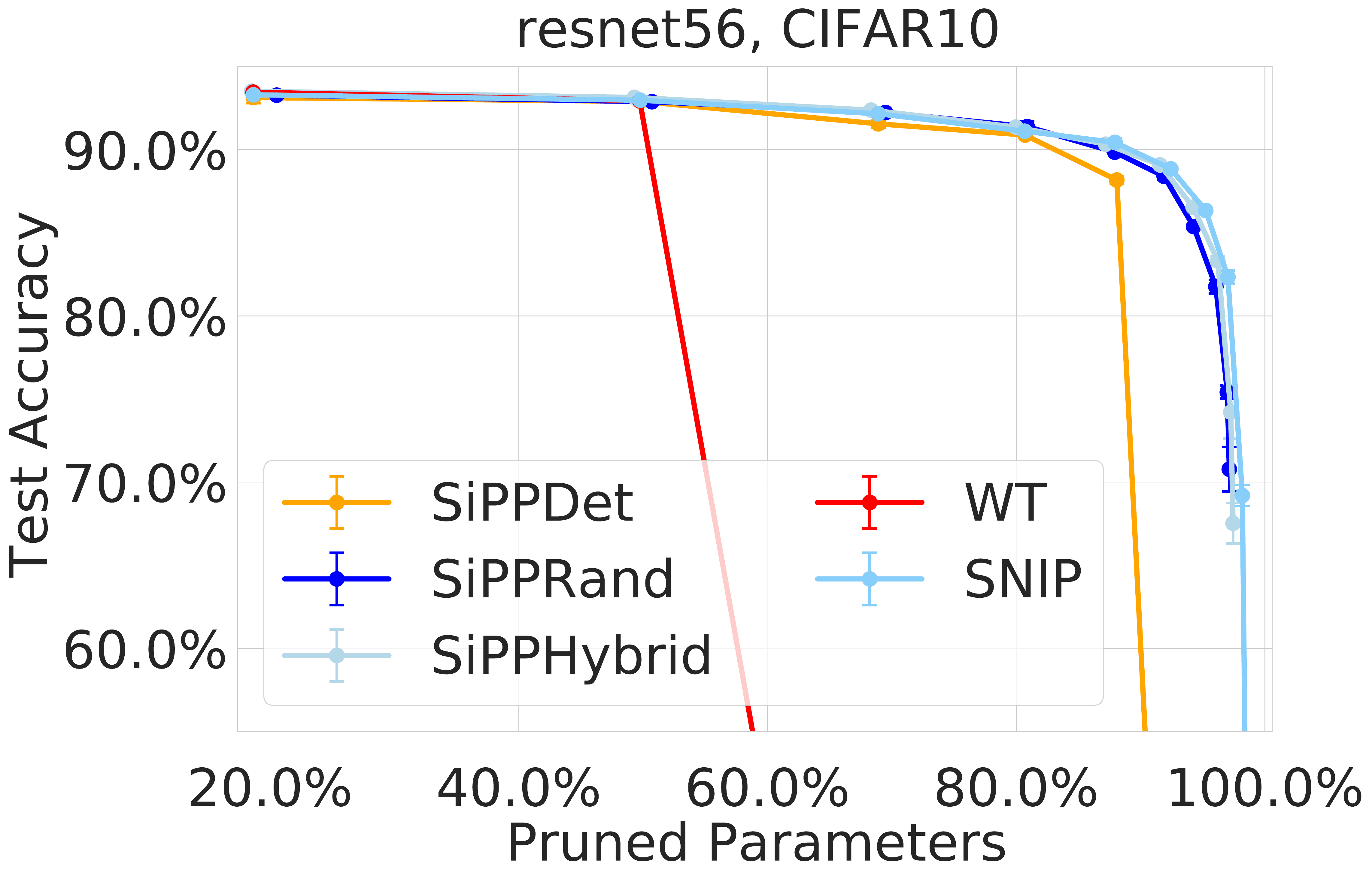}
    \subcaption{Resnet56}
\end{minipage}
\begin{minipage}[t]{0.48\textwidth}
    \includegraphics[width=\textwidth]{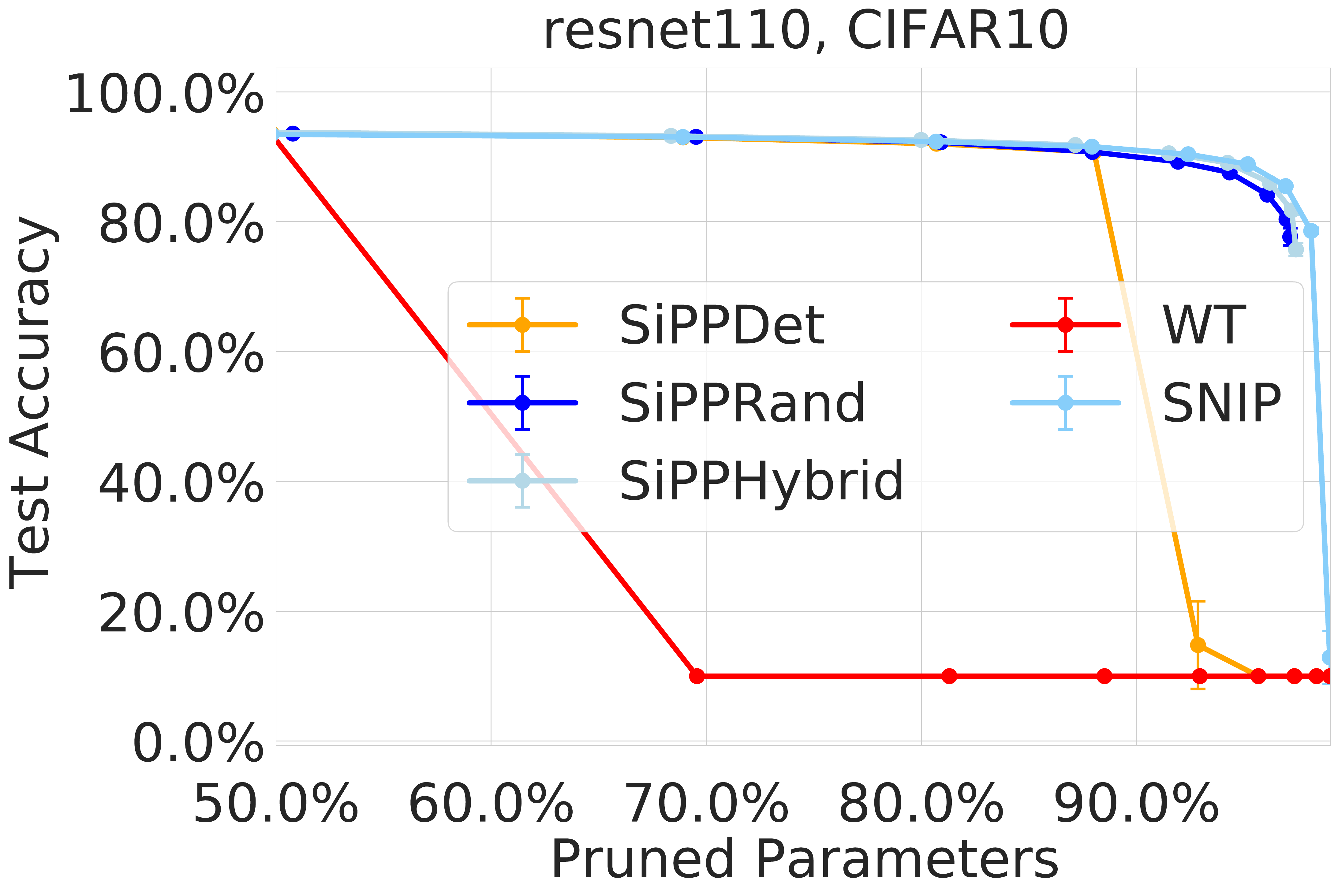}
    \subcaption{Resnet110}
\end{minipage}
\begin{minipage}[t]{0.48\textwidth}
    \includegraphics[width=\textwidth]{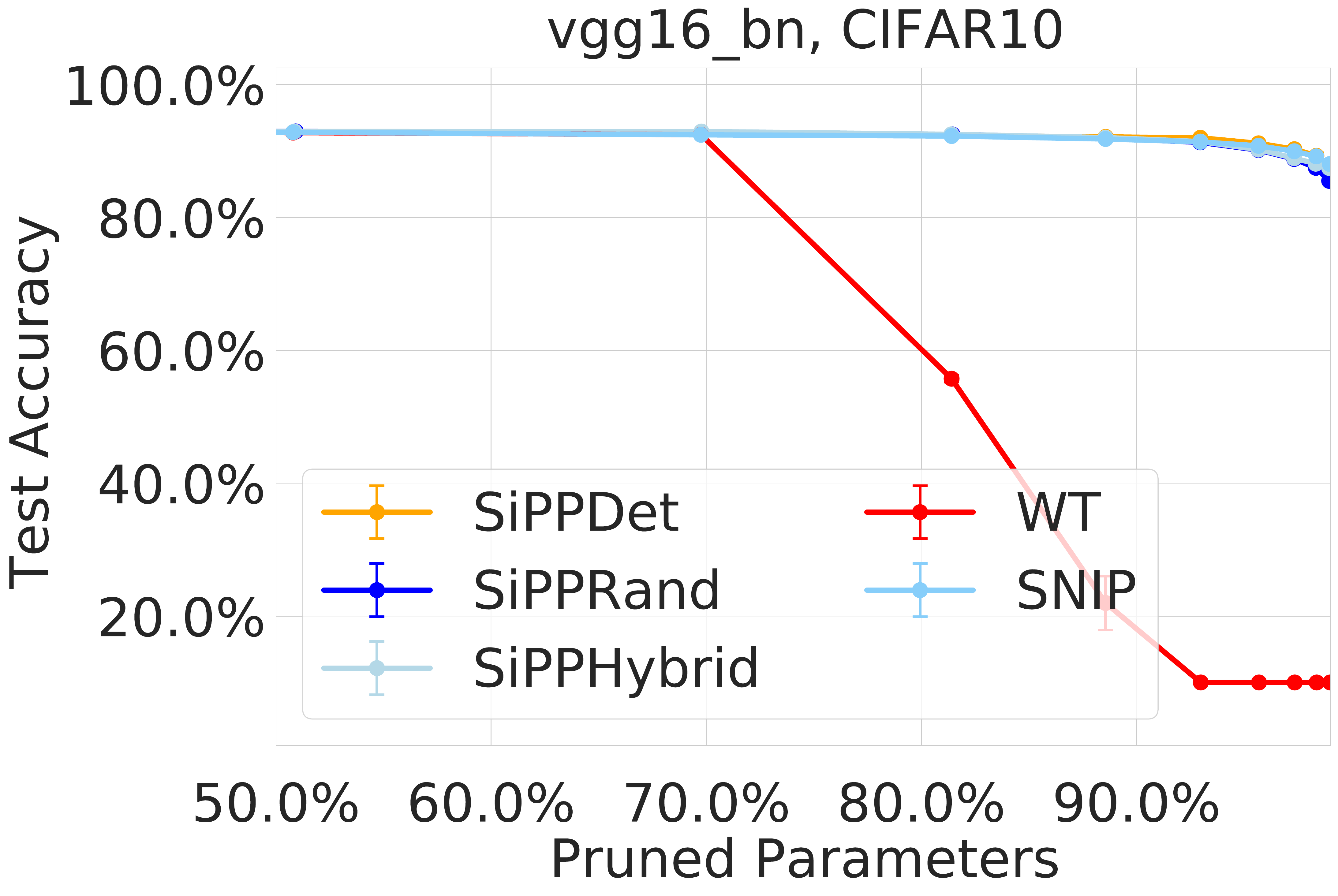}
    \subcaption{VGG16}
\end{minipage}
\begin{minipage}[t]{0.48\textwidth}
    \includegraphics[width=\textwidth]{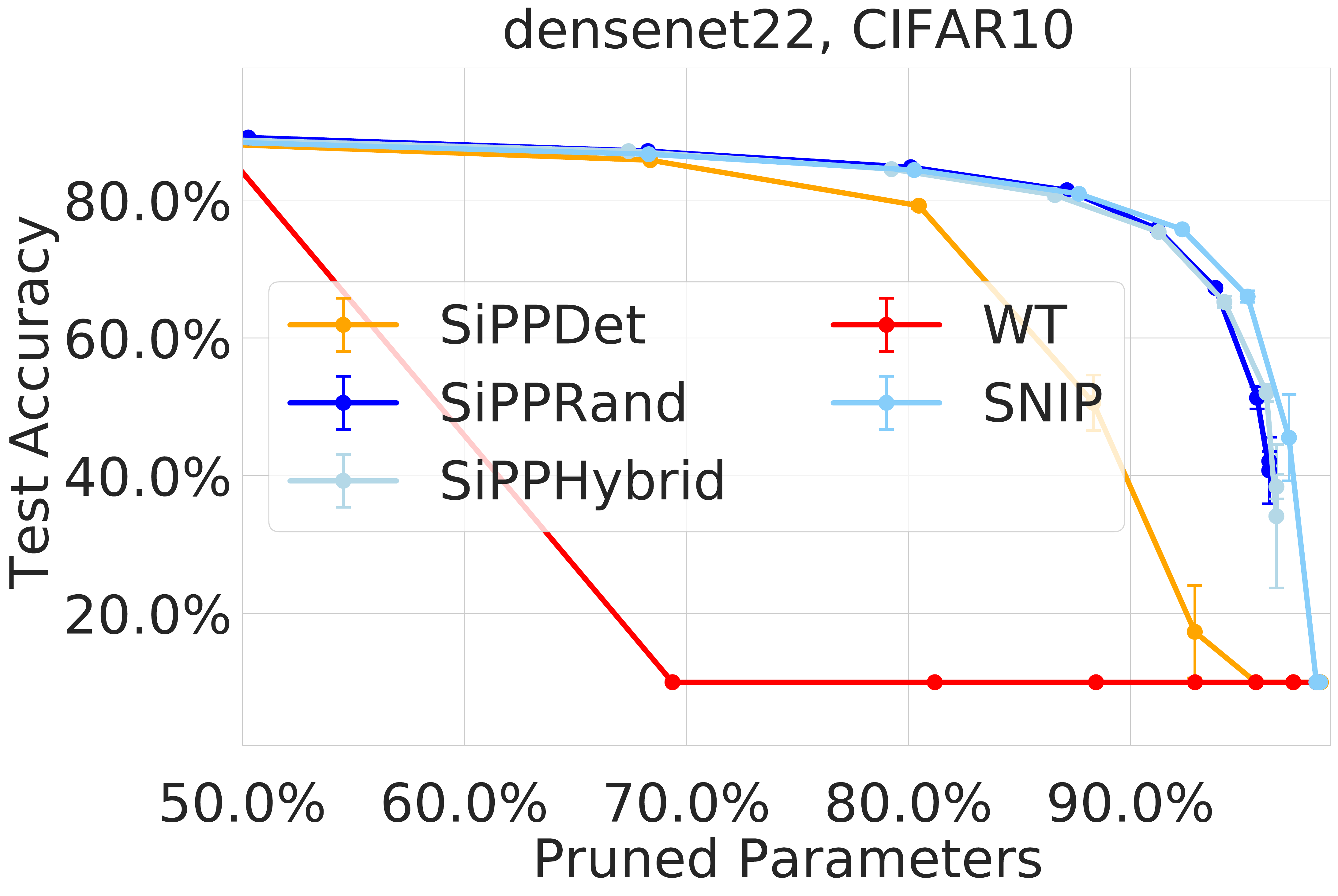}
    \subcaption{DenseNet22}
\end{minipage}%
\begin{minipage}[t]{0.48\textwidth}
    \includegraphics[width=\textwidth]{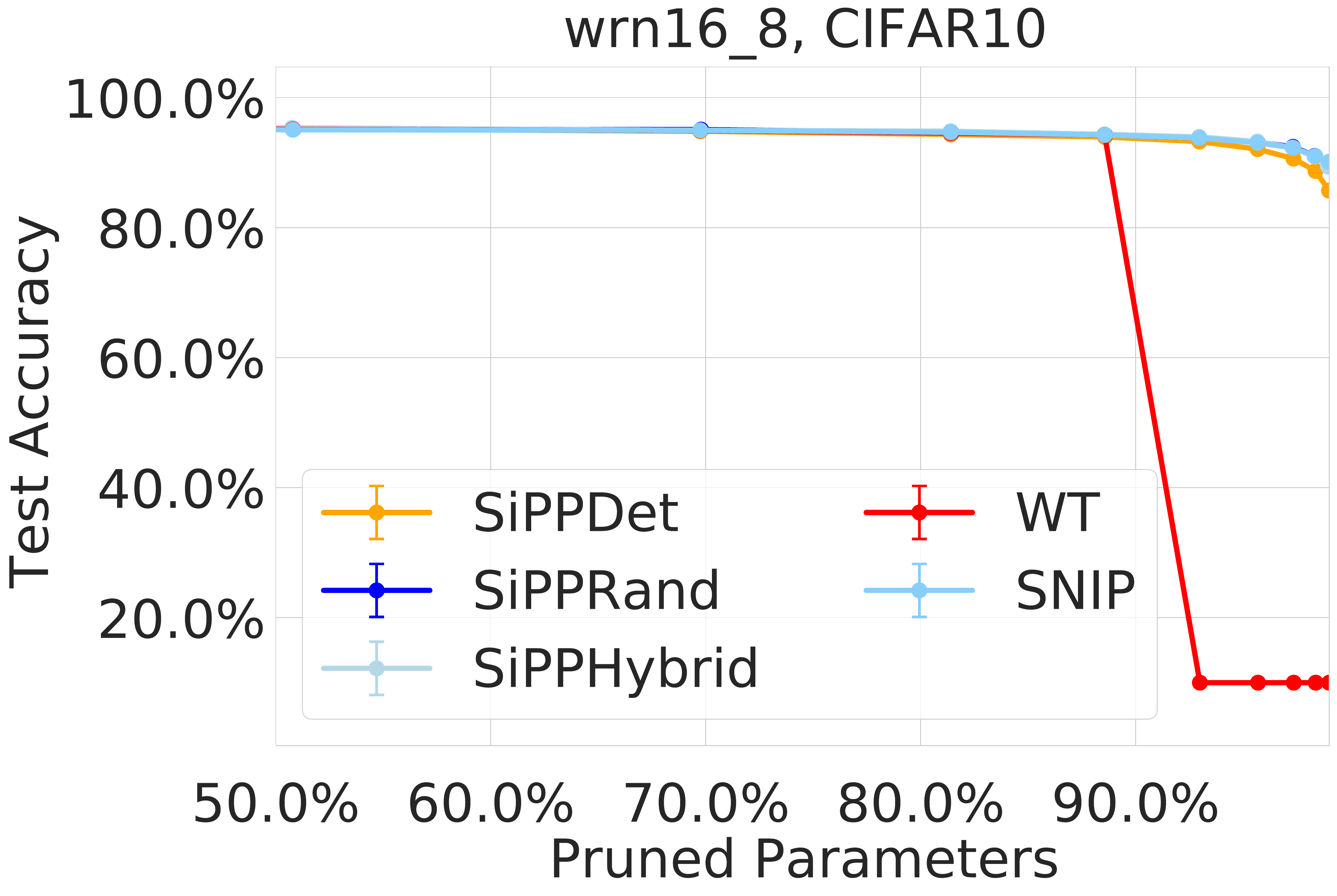}
    \subcaption{WRN16-8}
\end{minipage}%
 
\caption{The delta in test accuracy to the uncompressed network for the generated pruned models trained on CIFAR10 for various target prune ratios. The networks were pruned using the \textbf{random-init+ prune+train} pipeline.}
\label{fig:cifar_prune_rand}
\end{figure}

\begin{figure}[htb]
\centering
\begin{minipage}[t]{0.48\textwidth}
    \includegraphics[width=\textwidth]{figures/resnet18_ImageNet_e90_re90_cascade_int18_ImageNet_acc_param.pdf}
    \subcaption{ResNet18, Top 1}
\end{minipage}%
  \hspace{1ex}
\begin{minipage}[t]{0.48\textwidth}
    \includegraphics[width=\textwidth]{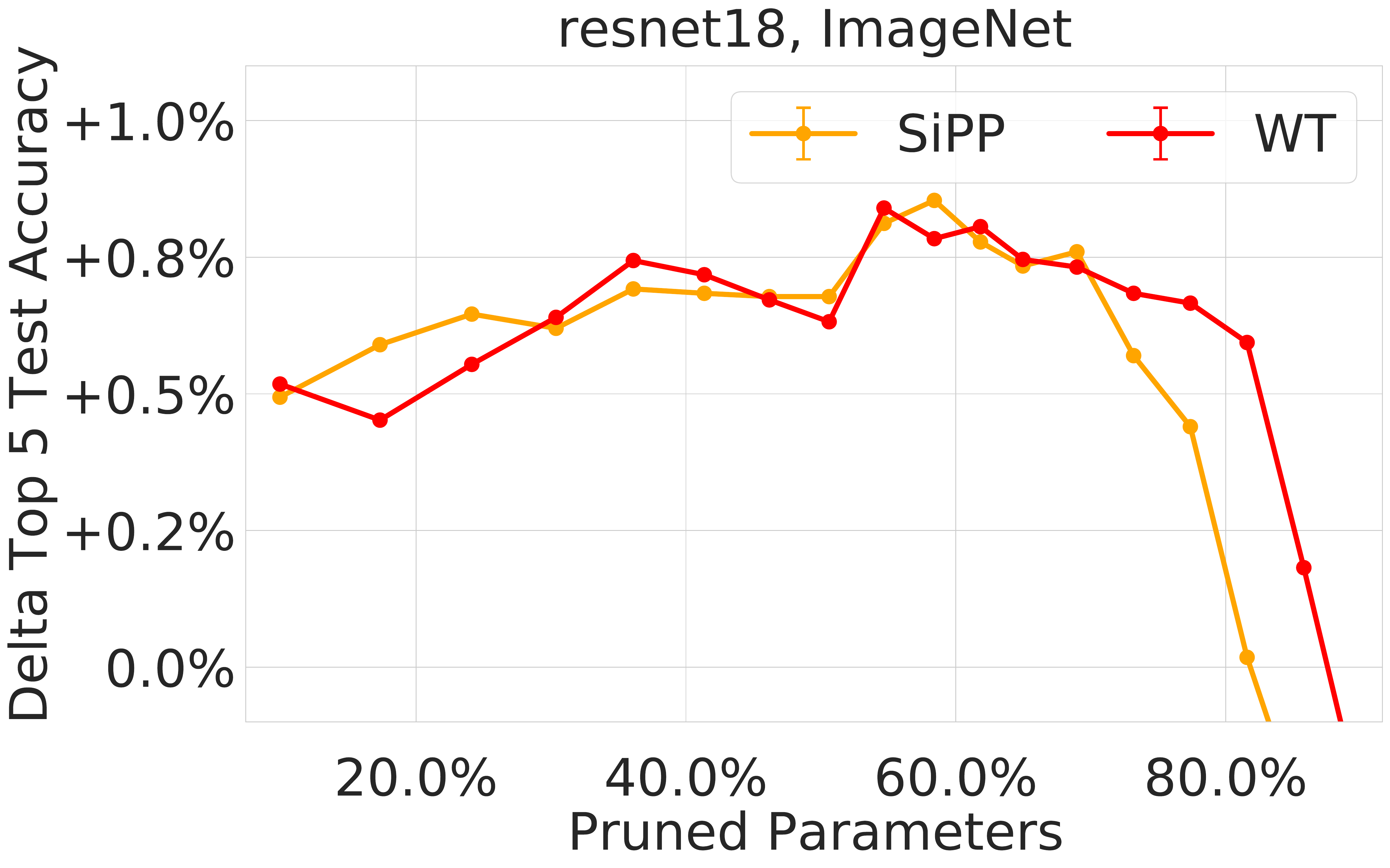}
    \subcaption{ResNet18, Top 5}
\end{minipage}
\begin{minipage}[t]{0.48\textwidth}
    \includegraphics[width=\textwidth]{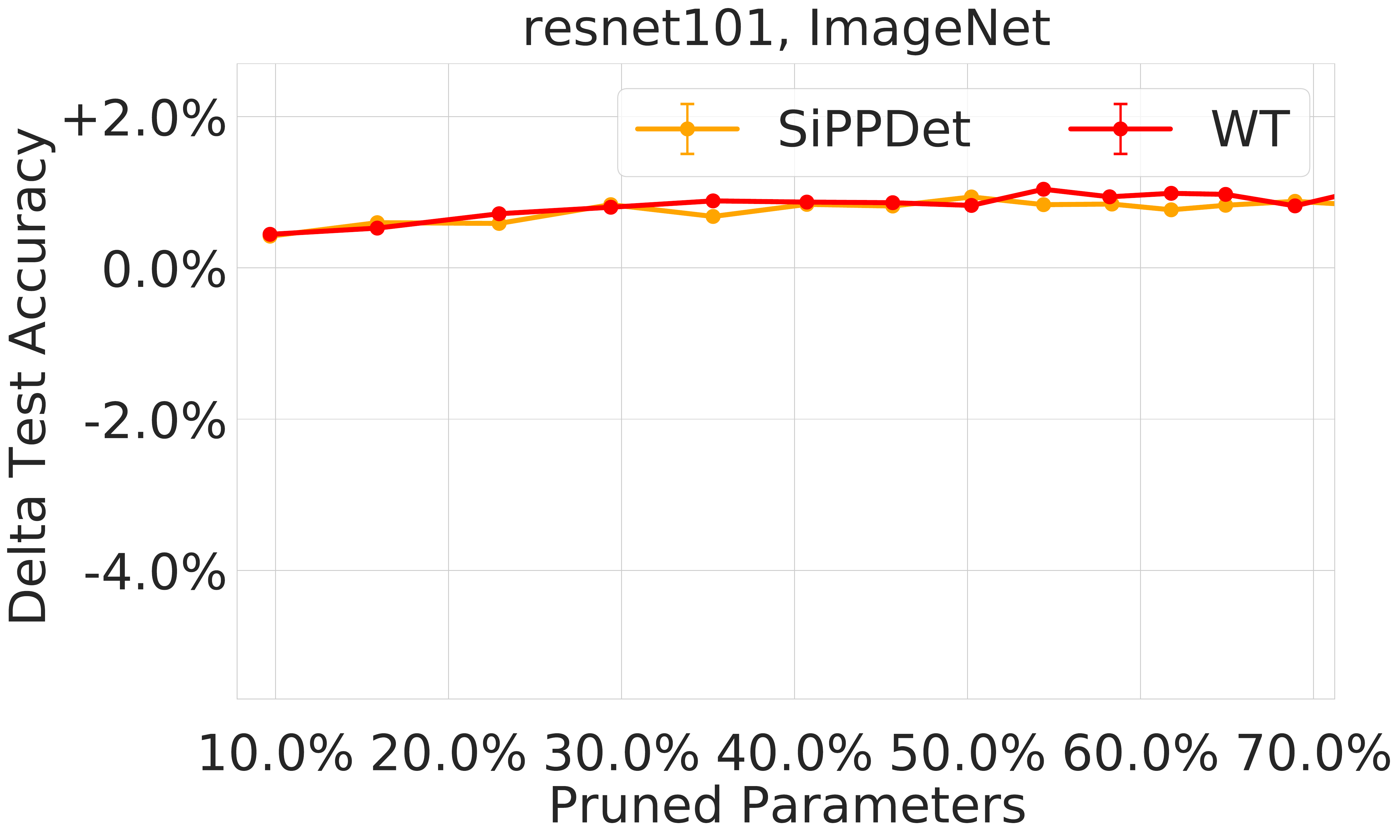}
    \subcaption{ResNet101, Top 1}
\end{minipage}%
\hspace{1ex}
\begin{minipage}[t]{0.48\textwidth}
    \includegraphics[width=\textwidth]{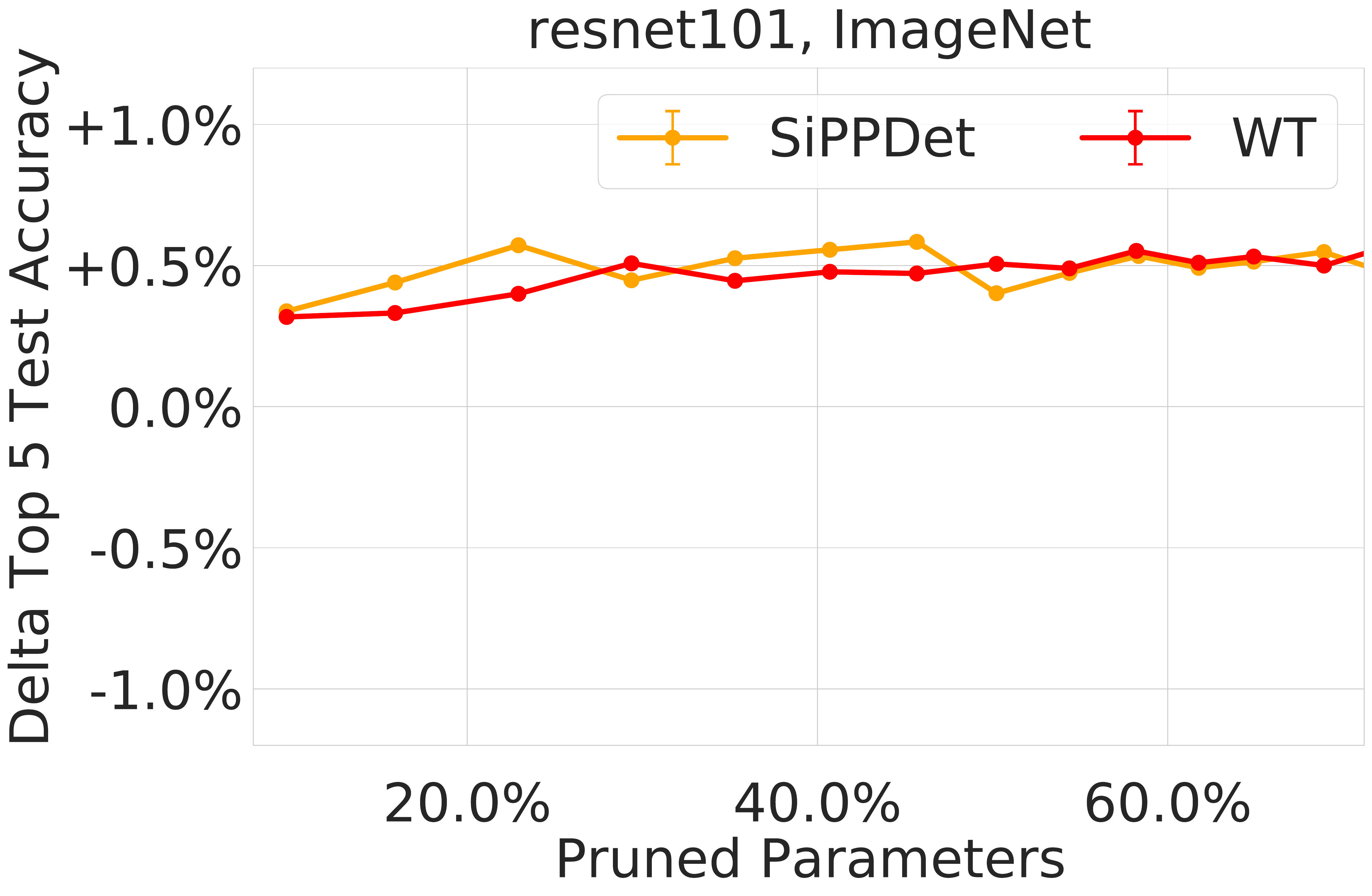}
    \subcaption{ResNet101, Top 5}
\end{minipage}%
\caption{The accuracy of the generated pruned \textbf{ResNet18} and \textbf{ResNet101} models trained on ImageNet for the evaluated pruning schemes for various target prune ratios.}
\label{fig:imagenet_prune}
\end{figure}

\end{document}